\documentclass[oneside,12pt]{IIScthesisPSnPDF}
\usepackage{etex}
\reserveinserts{28}
\usepackage[utf8]{inputenc}
\usepackage{comment}
\usepackage{empheq}
\usepackage{hyperref}
\usepackage{enumitem}
\usepackage{amssymb}
\usepackage{amsmath}
\usepackage{latexsym}
\usepackage{multicol}
\usepackage{booktabs}
\usepackage{multirow}
\usepackage{fancyhdr}
\usepackage{cleveref}
\usepackage{wrapfig}
\usepackage{xcolor}
\usepackage[ruled]{algorithm2e}

\usepackage{algorithmic}

\usepackage{cancel}
\usepackage{mathrsfs}
\usepackage{amsthm}
\usepackage{mathtools}
\usepackage{color}

\usepackage{mdwlist}
\usepackage{tikz}
\usepackage{varwidth}
\usepackage[vertfit]{breakurl}
\usepackage{datetime}
\usepackage{pdfpages}
\usepackage{tikz}
\usetikzlibrary{plotmarks}
\usepackage{pgfplots}
\usetikzlibrary{decorations.pathmorphing}  
\usetikzlibrary{decorations.shapes}  
\usetikzlibrary{fadings}
\usetikzlibrary{patterns}
\usetikzlibrary{arrows,shapes,shadows}
\usetikzlibrary{calc}
\usetikzlibrary{positioning} 
\tikzset{
    >=stealth',	
    punkt/.style={
           rectangle,
           rounded corners,
           draw=black, thick,
           text width=6.5em,
           minimum height=2em,
           text centered},
    pil/.style={
           ->,
           thick,
           shorten <=2pt,
           shorten >=2pt,}
}

\usetikzlibrary{shapes,arrows, trees}

\usepackage{setspace}
\usepackage{caption}
\usepackage{subcaption}
\usepackage{macros}
\usepackage{float}
\usepackage{morefloats}
\usetikzlibrary{calc}
\DeclareGraphicsExtensions{.png,.jpg,.jpeg,.eps}
\newtheorem{mydef}{Definition}
\newtheorem{remark}{Remark}

\newtheorem{theorem}{Theorem}

\newtheorem{lemma}{Lemma}
\newtheorem{corollary}{Corollary}

\allowdisplaybreaks	

\pdfmapline{+eufb10 EUFB10 <eufb10.pfb}   
\pgfplotsset{
every axis label/.append style={font=\large},
}

\newdateformat{monthyeardate}{ \monthname[\THEMONTH], \THEYEAR}

\newcommand{\blankpage}{
\newpage
\thispagestyle{empty}
\mbox{}
\newpage
}

\crefname{observation}{observation}{observations}
\crefname{algorithm}{algorithm}{algorithms}
\crefname{align}{equation}{equations}
\crefname{eqnarray}{equation}{equations}

\begin{document}

\title{Stochastic Approximation with Markov Noise: Analysis and applications in reinforcement learning}	
\submitdate{2015}

\submitdate{\monthyeardate\today} 
\phd
\dept{Computer Science and Automation}
\faculty{Faculty of Engineering}
\author{Prasenjit Karmakar}


\maketitle

\begin{center}
\LARGE{\underline{\textbf{Declaration of Originality}}}
\end{center}
\noindent I, \textbf{Prasenjit Karmakar}, with SR No. \textbf{04-04-00-10-12-13-1-10390} hereby declare that
the material presented in the thesis titled

\begin{center}
\textbf{Stochastic Approximation with Markov Noise: Analysis
and applications in reinforcement learning}
\end{center}

\noindent represents original work carried out by me in the \textbf{Department
of Computer Science and Automation} at \textbf{Indian Institute of
Science} during the years \textbf{Aug 2013- Jan 2018}.

\noindent With my signature, I certify that:
\begin{itemize}
	\item I have not manipulated any of the data or results.
	\item I have not committed any plagiarism of intellectual
	property.
	I have clearly indicated and referenced the contributions of
	others.
	\item I have explicitly acknowledged all collaborative research
	and discusions.
	\item I have understood that any false claim will result in severe
	disciplinary action.
	\item I have understood that the work may be screened for any form
	of academic misconduct.
\end{itemize}

\vspace{20mm}

\noindent {\footnotesize{Date:	\hfill	Student Signature}} \qquad

\vspace{20mm}

\noindent In my capacity as supervisor of the above-mentioned work, I certify
that the above statements are true to the best of my knowledge, and 
I have carried out due diligence to ensure the originality of the
report.

\vspace{20mm}

\noindent  {\footnotesize{Advisor Name: \hfill Advisor Signature}} \qquad

\blankpage

\vspace*{\fill}
\begin{center}
\large\bf \textcopyright \ Prasenjit Karmakar\\
\large\bf \monthyeardate\today\\
\large\bf All rights reserved
\end{center}
\vspace*{\fill}
\thispagestyle{empty}

\blankpage

\vspace*{\fill}
\begin{center}
DEDICATED TO \\[2em]
\Large\it My Parents and My teachers\\[2em]
\end{center}
\vspace*{\fill}
\thispagestyle{empty}


\setcounter{secnumdepth}{3}
\setcounter{tocdepth}{3}
\frontmatter 
\pagenumbering{roman}
\prefacesection{Acknowledgements}
 I take this opportunity to express my profound gratitude to the people who have helped and supported me throughout my PhD. 
 I thank my guide, Prof. Shalabh Bhatnagar, for his constant and valuable suggestions during my research. 
 Without his constant guidance and support this thesis would not have been possible.
 In all my time as a
student, I never
had one idea that didn’t have his ear. I aspire to be able to bring the kind of rigour,
commitment, and clarity to one’s work as he does.
 I will always cherish our many lengthy discussions, especially
when we were working on problems.
 Next I thank Prof. Csaba Szepesvári for several interesting and useful discussions when I started working on the first 
 problem. I thank Rajkumar Maity, who was a project associate of Stochastic Systems Lab, for helping me in running 
 the simulations related to the off-policy learning work in this thesis.
 I thank Prof. Narahari, Chairman of Electrical Sciences Division, who has always been keen to know how my thesis work is progressing.  
 I am grateful to all the professors of CSA Department and Mathematics Department for sharing their knowledge and for their encouragement.  
 I  also
thank  all  the  anonymous  reviewers  of  our  papers,  and  the  committee  of  my  comprehensive
examination.
 I, also, thank the CSA Dept office staff for their invaluable support.  
 I would like to particularly acknowledge the help by Mrs. Suguna, Mrs. Kushael and Mrs. Padmavathi who promptly helped me on many occasions. 
 Last but not the least, I express my heartfelt gratitude to the Almighty, my parents, and friends for their love and blessings.

\prefacesection{Abstract}
Stochastic approximation algorithms are sequential non-parametric methods for finding a zero or minimum of a function in the situation where only the noisy observations of the function values are available. Two time-scale stochastic approximation algorithms consist of two coupled recursions which are updated with different (one is considerably smaller than the other) step sizes which in turn facilitate convergence for such algorithms.

 We present for the first time an asymptotic convergence analysis of two time- scale stochastic approximation driven by 'controlled' Markov noise. In particular, the faster and slower recursions have non-additive controlled Markov noise components in addition to martingale difference noise. We analyze the asymptotic behavior of our framework by relating it to limiting differential inclusions in both time scales that are defined in terms of the ergodic occupation measures associated with the controlled Markov processes. 

Using a special case of our results, we present a solution to the off-policy convergence problem for temporal-difference learning with linear function approximation.

 One of the important assumption in the earlier analysis is the point-wise boundedness (also called the 'stability') of the iterates. However, finding sufficient verifiable conditions for this is very hard when the noise is Markov as well as when there are multiple timescales. We compile several aspects of the dynamics of stochastic approximation algorithms with Markov iterate-dependent noise when the iterates are not known to be stable beforehand. We achieve the same by extending the lock-in probability (i.e. the probability of convergence to a specific attractor of the limiting o.d.e. given that the iterates are in its domain of attraction after a sufficiently large number of iterations (say) $n_0$) framework to such recursions. Specifically, with the more restrictive assumption of Markov iterate-dependent noise supported on a bounded subset of the Euclidean space we give a lower bound for the lock- in probability. We use these results to prove almost sure convergence of the iterates to the specified attractor when the iterates satisfy an `asymptotic tightness' condition. This, in turn, is shown to be useful in analyzing the tracking ability of general 'adaptive' algorithms. Additionally, we show that our results can be used to derive a sample complexity estimate of such recursions, which then can be used for step-size selection.

 Finally, we obtain the first informative error bounds on function approximation for the policy evaluation algorithm proposed by Basu et al. when the aim is to find the risk-sensitive cost represented using exponential utility. We also give examples where all our bounds achieve the “actual error” whereas the earlier bound given by Basu et al. is much weaker in comparison. We show that this happens due to the absence of difference term in the earlier bound which is always present in all our bounds when the state space is large. Additionally, we discuss how all our bounds compare with each other.
\prefacesection{Publications based on this Thesis}
\begin{enumerate}

\item P. Karmakar and S. Bhatnagar (2018), ``Two Timescale Stochastic Approximation with Controlled Markov noise 
and Off-policy temporal difference learning'',{\em Mathematics of Operations Research, Vol. 43, No. 1},
{\url{https://doi.org/10.1287/moor.2017.0855}}.

\item P. Karmakar, R. Maity and S. Bhatnagar (2016), ``On a convergent \textit{off}-policy temporal difference learning algorithm in \textit{on}-line learning environment'', 
{\url{https://arxiv.org/abs/1605.06076}}.

\item P. Karmakar and S. Bhatnagar (2016),
  ``Dynamics of stochastic approximation with
Markov iterate-dependent noise with the stability
of the iterates not ensured'', {\em  to be submiited in IEEE Transactions on Automatic Control},  
{\url{https://arxiv.org/abs/1601.02217}}.

\item P. Karmakar and S. Bhatnagar (2017), 
 `` On the function approximation error for risk-sensitive
reinforcement learning'', {\em  to be submiited in IEEE Transactions on Automatic Control}, {\url{https://arxiv.org/abs/1612.07562}}.

 \end{enumerate}

\tableofcontents
\listoffigures
\mainmatter 
\setcounter{page}{1}
\chapter{Introduction}
\label{chap:introduction}

\section{History of Stochastic Approximation algorithms}
Optimization is ubiquitous in various research and application fields.
It is quite often that an optimization problem can be reduced to finding
zeros (roots) of an unknown function $f(\cdot)=E[g(\cdot,\eta)]$,
which can be observed but
the observation may be corrupted by errors (such is the case where distribution of the noise $\eta$ is unknown, and say 
i.i.d samples of $\eta$, namely $\mathbb{R}^k$-valued $\{\eta_k\}$ are available). Here $g: \mathbb{R}^d \times \mathbb{R}^k \to \mathbb{R}^d$. This is the topic of stochastic 
approximation (SA). The error source may be observation noise, but
may also come from structural inaccuracy of the observed function. For
example, suppose one wants to find zeros of $f(\cdot)$
but actually observes functions $f_k(\cdot) = g(\cdot,\eta_{k+1})$
which are different from $f(\cdot)$.

The basic recursive algorithm for finding roots of an unknown function
on the basis of noisy observations is the Robbins-Monro (RM) algorithm \cite{robbins_monro} namely, 
\begin{align}
\label{rm} 
\theta_{k+1} = \theta_k + a(k) g(\theta_k, \eta_{k+1}),  
\end{align}
where $\theta_n \in \mathbb{R}^d, n \geq 0$.
Here $g(\theta_k, \eta_{k+1})$ is a noisy observation of the function $f(\cdot)$ when the parameter value is 
$\theta_k$, $\eta_{k+1}$ is the noise at instant $k+1$ and $a(k), k\geq 0$ is a step-size sequence. 
The algorithm thus incrementally updates the parameter $\theta$ at each instant by making use of the 
previous values of the updates.

\subsection{Convergence Analysis: Probabilistic Method vs. ODE method}
Robbins and Monro \cite{robbins_monro} have proved mean-square convergence to the point 
where a regression function assumes a given value. Wolfowitz \cite{wolf1} showed
that under weaker assumptions we may still obtain convergence in probability to the root; and
Blum \cite{blum}  demonstrated that, under even weaker assumptions, there is not only convergence 
in probability but in fact also convergence with probability 1. Kiefer and Wolfowitz \cite{wolf2}
have devised a method for approximating the point where the maximum of a regression
function occurs. They proved that under suitable conditions there is convergence in probability and 
Blum \cite{blum} subsequently weakened somewhat the conditions and strengthened the conclusion to 
convergence with probability 1. Dvoretzky \cite{dvor} deals with a vastly
more general situation. The underlying idea is to think of the random element as noise
superimposed on a convergent deterministic scheme. The Robbins-Monro and Kiefer-Wolfowitz 
procedures, under conditions weaker than any previously considered, are included 
as very special cases and, despite this generality, the conclusion is stronger since
his results assert that the convergence is both in the mean-square sense as well as with probability 1.

In another method \cite[Chapter 1.2]{fu_chen}, Robbins-Siegmund Theorem \cite[Theorem 1.3.12]{duflo} 
is used to prove the almost sure convergence of the stochastic 
approximation algorithm. However, the classical probabilistic approach to
analyzing stochastic approximation algorithms requires rather restrictive conditions on the
observation noise (for example see A 1.2.4 in \cite{fu_chen}). An alternative approach called 
the ordinary differential equation (o.d.e) method was proposed by Ljung \cite{ljung} which treats the 
stochastic approximation algorithm as a noisy discretization (or 
Euler Scheme in numerical analysis parlance) for the o.d.e 
\begin{align}
\label{o.d.e}
\dot{\theta}(t) = f(\theta(t)). 
\end{align} Then a result from \cite{benaim_original} (this also appeared in \cite{benaimode} later) says that almost surely the sequence $\{\theta_n\}$ generated by (\ref{rm})
converges to a (possibly sample path dependent) compact connected internally chain transitive invariant set of (\ref{o.d.e}). 
Additionally, suppose there exists a continuously differentiable $V: \mathbb{R}^d \to [0, \infty)$ such that $\lim_{\|\theta\| \to \infty} V(\theta) = \infty$, 
$H:= \{\theta \in \mathbb{R}^d : V(\theta) =0\} \neq \phi$, and $\langle f(\theta), \nabla V(\theta) \rangle \leq 0$ with 
equality if and only if $\theta \in H$ (thus $V$ is a `Lyapunov function'). Then a standard result (\cite[Corollary 3]{borkar1}, \cite{benaimode})
says that almost surely, $\{\theta_n\}$ converges to an internally chain transitive invariant set contained in $H$. 

\subsection{Two time-scale stochastic approximation}
We first provide some motivation here. Suppose that 
an iterative algorithm calls for a particular subroutine in each iteration. Suppose also that this subroutine itself is 
another iterative algorithm. The traditional method would be to use the output of the subroutine after running it 'long 
enough' (i.e. until near-convergence) during each iterate of the outer loop. But the question is that can we get 
the same effect by running both the inner and the outer loops (i.e. the corresponding iterations) concurrently, albeit 
on different timescales i.e. using different step-size schedules, one of which governing the `slower 
recursion' goes to zero at a rate faster than the other. Then the inner 'fast' loop sees the outer 'slow' loop as quasi-static while the latter sees the 
former as nearly equilibrated. Such iterations are described as:
\begin{align}
\theta_{n+1} = \theta_n + a(n)\left[h(\theta_n, w_n) + M^{(1)}_{n+1}\right], \label{int_slow} \\ 
w_{n+1} = w_n + b(n)\left[g(\theta_n, w_n) + M^{(2)}_{n+1}\right].\label{int_fast} 
\end{align}
Here $h:\mathbb{R}^{d+k} \to \mathbb{R}^d, g:\mathbb{R}^{d+k} \to \mathbb{R}^k$ are Lipschitz and $\{M^{(1)}_n\},\{M^{(2)}_n\}$ are 
martingale difference sequences w.r.t. the increasing $\sigma$-fields
\begin{align}
\mathcal{F}_n :=\sigma(\theta_m,w_m,M^{(1)}_m, M^{(2)}_m, m \leq n), n \geq 0. \nonumber     
\end{align}
Stepsizes $\{a(n)\}, \{b(n)\}$ are positive scalars satisfying
\begin{align}
\sum_{n}a(n) = \sum_{n}b(n)= \infty, \sum_n (a(n)^2 + b(n)^2) < \infty, \frac{a(n)}{b(n)} \to 0. \nonumber 
\end{align}
The last condition implies that $a(n) \to 0$ at a faster rate than $\{b(n)\}$, implying that $(\ref{int_slow})$ moves on a slower timescale than $(\ref{int_fast})$. Examples of such
stepsizes are $b(n)= \frac{1}{n}, a(n) = \frac{1}{1+n\log n}$, or
$b(n) = \frac{1}{n^{2/3}}, a(n) = \frac{1}{n}$.

The intuition of such a  framework comes from comparing the above iterations to the singularly perturbed o.d.e. \cite{sptb}
\begin{align}
\dot{w}(t) = \frac{1}{\epsilon}g(\theta(t),w(t)), \label{e1} \\
\dot{\theta}(t) = h(\theta(t), w(t)) \label{e2},
\end{align}
in the limit $\epsilon \Downarrow 0$. Thus $w(\cdot)$ is the fast transient and $\theta(\cdot)$ the slow component. It then makes 
sense to think of $\theta(\cdot)$ as quasi-static (i.e. `almost a constant') while analyzing the behaviour of $\theta(\cdot)$. 
This suggests looking at the o.d.e.
\begin{align}
\dot{w}(t) = g(\theta,w(t)), \nonumber 
\end{align}
where $\theta$ is held fixed as a constant parameter. Suppose that the above o.d.e. has a 
globally asymptotically stable equilibrium $\lambda(\theta)$ where $\lambda: \mathbb{R}^k \to \mathbb{R}^d$. Then clearly   
for sufficiently small values of $\epsilon$, we expect $w(t)$ to closely track $\lambda(\theta(t))$ for $t >0$. 
In turn this suggests looking at the o.d.e.
\begin{align}
\dot{\theta}(t) = h(\theta(t), \lambda(\theta(t))), \nonumber 
\end{align}
which should capture the behaviour of $\theta(\cdot)$ in (\ref{e2}) to a good approximation. This technique of 
replacing fast variable $w$ with $\lambda(\theta)$ is common in the literature of 
singularly perturbed o.d.e. \cite[Section 1.2]{sptb}. However, 
here in order to relate the discrete time system to its corresponding o.d.e, Borkar assumed 
$\lambda$ to be a Lipschitz continuous 
map.

Then under \textbf{(A1)}, \textbf{(A2)} and \textbf{(A3)} from \cite[Chapter 6]{borkar1}, Borkar proved almost 
sure convergence of such iterates. Another analysis of similar two 
timescale systems can be found in \cite[Chapter 8.6]{kushner}.

\subsection{Stochastic Approximation with state-dependent Markov Noise}
\label{markov}
Stochastic Approximation algorithms with Markov Noise are algorithms where 
$\eta_k$'s are trajectories of a parameterized (by the algorithm's iterates) Markov chain. Such noise arises in many 
important applications where evolution of the noise process depends more intimately on the 
iterate (also called the \textit{state})
and there is a reasonably long-term ``memory'' in this dependence.
In the following we investigate cases where such iterate-dependent Markov noise arises:
\begin{enumerate}
    \item The following adaptive ``routing'' example 
taken from \cite{kushner2,kushner} illustrates this point in a simple way. Suppose that calls arrive at a switch randomly, but at 
discrete instants $n = 1,2, \dots$. No more than a single call can arrive at a time and 
\begin{align}
P\{\mbox{arrival at time}~~n | \mbox{data upto but not including time~~} n \} = \mu > 0 \nonumber
\end{align}

The assumptions concerning single calls and discrete time make the formulation simpler, but the analogous models 
in continuous time are treated in essentially the same manner. To have a clear sequencing of the events, 
suppose that each call in progress gets completed ``just before'' the next discrete instant, so if 
a call is completed at time 
$n^-$, then the circuit is available for use by a new call that arrives at time $n$. There are two possible 
routes for each call. The $i$-th route has $N_i$ lines and can handle $N_i$ calls simultaneously. The sets of call 
lengths and inter-arrival times are mutually independent, and $\lambda_i > 0$ is the probability 
that a call is completed at the $(n+1)$st instant, given that it is in the system at time $n$ and handled by route $i$, and 
the rest of the past data. The system is illustrated in Fig. \ref{routing}. 
\begin{figure}
 \centering
\includegraphics[width=4 in, height= 3 in]{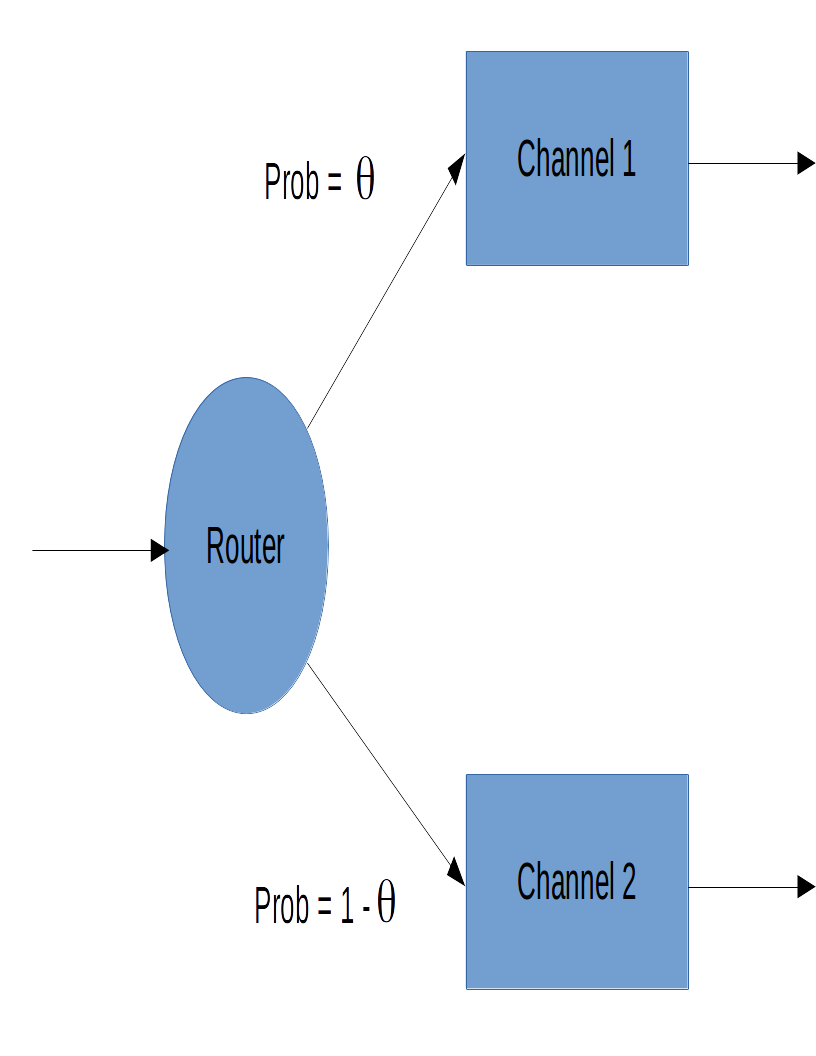}
\caption{The routing system, this figure is taken from \cite[p~38, Fig. 3.1]{kushner}}
\label{routing}   
\end{figure}
The routing law is ``random'', and is updated by a stochastic approximation procedure with 
constant step size $\epsilon$. Let $\eta^{\epsilon}_{n}=(\eta^{\epsilon}_{n,1}, \eta^{\epsilon}_{n,2})$ 
denote the occupancies of the two routes at time $n$. If a call arrives at time $n+1$ then it is sent 
to route $1$ with probability $\theta^{\epsilon}_{n}$ and to 
route 2 with probability $1- \theta^{\epsilon}_{n}$. If all 
lines of the selected route are occupied at that time, the call is redirected to the other route. If the alternative 
route is also full, the call is lost from the system. Let $J^{\epsilon}_{n,i}$ be the indicator function of the
 event that a call arrives at time $n+1$, is sent to route $i$, and is accepted there. The updating rule for 
 $\theta^{\epsilon}_{n}$ is 
\begin{align}
\theta^{\epsilon}_{n+1} &= \Pi_{[a,b]}[\theta^{\epsilon}_{n} + \epsilon Y^{\epsilon}_{n}] \nonumber \\ 
                       &= \Pi_{[a,b]}[\theta^{\epsilon}_{n} + \epsilon \left(\left(1-\theta^{\epsilon}_{n}\right)J^{\epsilon}_{n.1} - \theta^{\epsilon}_{n} J^{\epsilon}_{n.2}\right)]  \nonumber
\end{align}
where $0 < a <b <1$ are truncation levels and $\Pi_{[a,b]}$ denotes the truncation operator. 

The occupancies $\eta^{\epsilon}_{n}$ (and the random acceptances and routing choices) determine the effective noise in the system, and 
the evolution of $\eta^{\epsilon}_{n}$ depends on $\theta^{\epsilon}_{n}$ in a complicated way and with significant memory. 
The dependence is of the Markovian type in that 
\begin{align}
P\{\eta^{\epsilon}_{n+1} = \tilde{\eta} | \eta^{\epsilon}_{i}, \theta^{\epsilon}_{i}, i \leq n\}  = P\{\eta^{\epsilon}_{n+1}=\tilde{\eta} |\theta^{\epsilon}_{n},\eta^{\epsilon}_{n}\}, \forall n. \nonumber  
\end{align}

\item Here we are concerned with a controlled Markov chain and average cost per unit time problem, where the
control is parameterized. We seek the optimal parameter value. The approach is based on an estimate of the derivative of the invariant measure
with respect to the parameter. This is an example of an optimization problem over an infinite time interval. The procedure attempts to approximate
the gradient of the stationary cost with respect to the parameter.

The process is a finite-state Markov chain $\{X_n\}$ with time-invariant and
known transition probabilities $p(x, y|\theta)$ that depend continuously and differentiably on a parameter $\theta$ that takes values in some compact set. We
will suppose that $\theta$ is real-valued and confined to some interval $[a, b]$. Write $p_{\theta}(x, y|\theta)$ for the $\theta$-derivative. Let
the chain be ergodic for each $\theta$, and denote the unique invariant measure
by $\mu(\theta)$. Let $E_{\mu(\theta)}$ denote expectation under the stationary probability and
$E_x^{\theta}$ the expectation, given parameter value $\theta$ and initial condition $x$. The objective is to obtain
\begin{align}
\min_{\theta} e(\theta) = E_{\mu(\theta)} k(X_0, X_1, \theta) = \sum_{x,y} \mu(x,\theta) p(x,y|\theta) k(x,y, \theta)    
\end{align}
\end{enumerate}
where $k(x, y, \cdot)$ is continuously differentiable in $\theta$ for each value of $x$ and
$y$, and $\mu(x, \theta)$ is the stationary probability of the point $x$.

Let $L(x,y,\theta) = \frac{p_{\theta}(x,y|\theta)}{p(x,y|\theta)}$.
Then the following algorithm can be given for approximating the optimal 
value of $\theta$.

\begin{align}
\theta_{n+1}&=\theta_n  - \epsilon_n k_{\theta}(X_n, X_{n+1}, \theta_n) - \epsilon_n L(X_n, X_{n+1}, \theta_n) k(X_n, X_{n+1}, \theta_n) - \epsilon_n(k(X_n, X_{n+1}, \theta_n) - \lambda_n) C_n  \nonumber \\
\lambda_{n+1} &= \lambda_n + \epsilon_n'\left[k(X_n, X_{n+1}, \theta_n) - \lambda_n \right] \nonumber \\
C_{n+1} &= \beta C_n + L(X_n, X_{n+1}, \theta_n) \nonumber
\end{align}
where $\epsilon_n'=q \epsilon_n$ for some $q>0$ and $k_{\theta}(.,.,.)$ denotes derivative with respect to $\theta$. 

Clearly, the noise $\{X_n\}$ in the algorithm is iterate-dependent.

This scenario is sometimes referred as policy gradient algorithm in reinforcement learning.

Now, for each $\theta \in \mathbb{R}^d$ we consider a transition probability
$\Pi_{\theta}(y;dx)$ on $\mathbb{R}^k$. This transition probability defines a
controlled Markov chain on $\mathbb{R}^d$. 

Stochastic approximation
iterates in $\mathbb{R}^d$ driven by Markovian noise are given by
\begin{align}
\label{q1}
\theta_{n+1} = \theta_n + a(n)f(\theta_n,Y_{n+1}), n\geq 0,
\end{align}
where $\theta_0$ is the initial point, $\{\theta_n\}$ are the iterates,
$\{Y_n\}$ is an $\mathbb{R}^k$-
valued `Markov iterate-dependent' noise, i.e., satisfies
\begin{align}
P[Y_{n+1} \in A | \mathcal{F}_n] = \int_A\Pi_{\theta_n}(Y_n; dx) \ \mbox{a.s.,} \label{eqn4}
\end{align}
where $\mathcal{F}_n :=$ the $\sigma$-field generated by all random variables realized till time $n$, $a(n)$ is the $n$-th step-size and $f: \mathbb{R}^d \times \mathbb{R}^k \to \mathbb{R}^d$.

It is well known that under reasonable assumptions \cite{metivier,benveniste,adam,met_prior}, (\ref{q1}) is an
asymptotic pseudotrajectory (in the sense of \cite{ben_hirsch}) and that
by the results in \cite{ben_hirsch, benaim_original} its limit set can be precisely
described as an internally chain transitive set of the o.d.e.
\begin{align}
\label{ode_m}
\dot{\theta}(t)=h(\theta(t)),
\end{align}
where $h(\theta)= \int f(\theta, y)\Gamma_{\theta}(dy)$, with $\Gamma_{\theta}$ being the unique stationary distribution
of the Markov iterate-dependent process $\{Y_n\}$ for a fixed $\theta$.

In another work, \cite{borkar}, Borkar studied the stochastic approximation algorithm where the noise 
is a controlled Markov process.

Note that Markov noise arises naturally in reinforcement learning situations because of the Markov decision process in the
background. In the next subsection we briefly describe about reinforcement learning.

\subsection{Markov Decision Process, Dynamic Programming and Reinforcement learning}
\label{back_m}
This section is a summarized version of the materials present in the popular book \cite{suttonb} and Chapter 2 and 3 of 
\cite{maeith}. Some nice books on these materials are \cite{bert,bert2, bert3, csaba_b} 
\subsubsection{The Agent-Environment Interface}
Reinforcement learning (RL) problem deals with the
problem of learning from interaction to achieve a goal. The learner/decision-
maker is called the \textit{agent}. The agent  interacts with the \textit{environment}. 
These entities interact continually, the agent selecting
actions and the environment responding to those actions and presenting new situations to 
the agent. The environment also gives the agent a \textit{reward} each time the 
agent picks an action or control. The goal of the agent is to maximize a certain 
long-term reward objective.  A complete specification of an
environment defines a \textit{task}, one instance of the reinforcement learning problem.

More specifically, the agent and environment interact at each of a sequence of 
discrete time steps, $t = 0, 1, 2, 3, . . .. $. At each time step $t$, the agent receives some
representation of the environment’s state, $S_t \in S$, where $S$ is the set of possible states,
and on that basis selects an action, $A_t \in \mathcal{A}(S_t)$, where $\mathcal{A}(S_t)$ 
is the set of actions
available in state $S_t$ . One time step later, in part as a consequence of its action, the
agent receives a numerical reward , $R_{t+1} \in \mathcal{R} \subset \mathbb{R}$, and 
moves to a new state,
$S_{t+1}$ . Figure \ref{agent} shows a diagram of  the agent-environment interaction.

\begin{figure}
 \centering
\includegraphics[width=3 in]{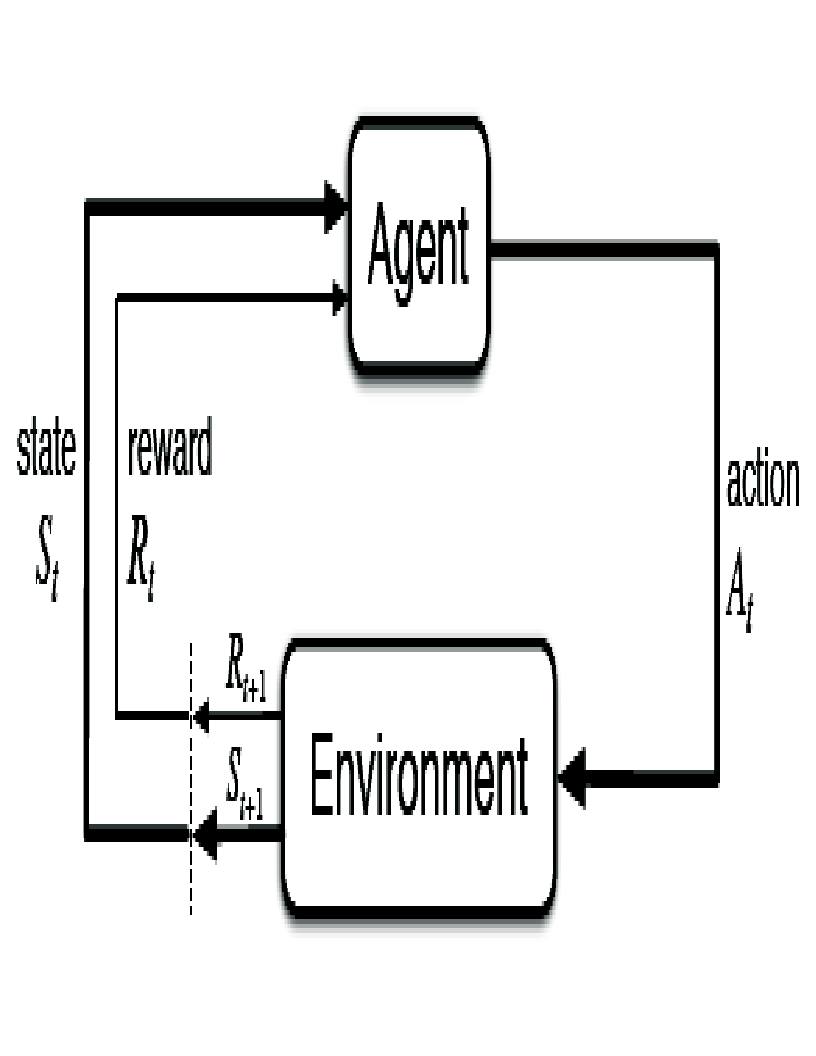}
\caption{The agent-environment interaction in reinforcement learning, this figure is taken from \cite[p~48]{suttonb}}
\label{agent}   
\end{figure}

At each time step, the agent implements a mapping from states to probabilities
of selecting each possible action. This mapping is called the agent’s \textit{policy} and is
denoted $\pi_t$ , where $\pi_t(a|s)$ is the probability that $A_t = a$ if $S_t = s$. This is 
called \textit{stochastic policy}. Policy can be deterministic also where $\pi_t: S \to A$. 
Here we assume \textit{stationary policy} where $\pi_t$ is independent of $t$.

Reinforcement
learning methods specify how the agent changes its policy as a result of its experience.
The agent’s goal, roughly speaking, is to maximize the total amount of reward it
receives over the long run.

\subsubsection{Goals and Rewards}
In reinforcement learning, the purpose or goal of the agent is formalized in terms of a
special reward signal that the environment gives to the agent. Informally, the agent's goal is to maximize the
total amount of reward it receives. This means maximizing not immediate reward,
but cumulative reward in the long run.
The use of a reward signal to formalize the idea of a goal is one of the most distinctive
features of reinforcement learning.

\subsubsection{Returns}
So far we have discussed the objective of learning informally. We have said that the
agent's goal is to maximize the cumulative reward it receives in the long run. How
might this be defined formally? If the sequence of rewards received after time step
$t$ is denoted $R_{t+1}, R_{t+2}, R_{t+3}, \dots$, then what precise aspect of this sequence do we
wish to maximize? In general, we seek to maximize the expected return, where the
return $G_t$ is defined as some specific function of the reward sequence. In the simplest 
case the return is simply the sum of the rewards:
\begin{align}
\label{fstate}
G_t:= R_{t+1} + R_{t+2} + R_{t+3} + \dots + R_T,
\end{align}
where $T$ is a final time step. This approach makes sense in applications in which there
is a natural notion of a final time step, that is, when the agent-environment interaction
breaks naturally into subsequences, which we call episodes. Each episode ends in
a special state called the \textit{terminal state}, followed by a reset to a standard starting
state or to a sample from a standard distribution of starting states. Tasks with
episodes of this kind are called \textit{episodic tasks}. In episodic tasks one sometimes need
to distinguish the set of all non-terminal states, denoted S, from the set of all states
plus the terminal state, denoted $\mathcal{S}+$.

On the other hand, in many cases the agent-environment interaction does not
break naturally into identifiable episodes, but goes on continually without limit. For
example, this would be the natural way to formulate a continual process-control task,
or an application to a robot with a long life span. The
return formulation (\ref{fstate}) is problematic for tasks that do not terminate  because the final time
step would be $T = \infty$, and the return, which is what we are trying to maximize,
could itself easily be infinite (for example, suppose the agent receives a reward of
+1 at each time step). 
The additional concept that we need is that of discounting. According to this
approach, the agent tries to select actions so that the sum of the discounted rewards
it receives over the future is maximized. In particular, it chooses $A_t$ to maximize the
expected discounted return:
\begin{align}
G_t := R_{t+1} + \gamma R_{t+2} + \gamma^2 R_{t+3} + \dots = \sum_{k=0}^{\infty} \gamma^k R_{t+k+1} \nonumber 
\end{align}
where $\gamma$ is a parameter, $0\leq \gamma < 1$, called the discount rate.

In the cases considered thus far the total expected cost is finite either because of discounting or because 
of a cost-free termination state that the system eventually enters. In many situations, however, discounting 
is inappropriate and there is no natural cost-free termination state. In such situations it is often meaningful to optimize 
the average cost per stage starting from a state $i$, which is defined by 
\begin{align}
J_{\pi}(i) = \lim_{N \to \infty} \frac{1}{N} E\{\sum_{k=0}^{N-1}g(x_k, \mu_k(x_k))|x_0 =i\}
\end{align}
where $g(i,a)$ is the cost of taking action $a$ from state $i$. $\mu_k$ maps states $x_k$ into controls $u_k = \mu_k(x_k)$.
For details , see \cite[Chapter 7.4]{bert}.

\subsubsection{The Markov Property} 
In the reinforcement learning framework, the agent makes its decisions as a function
of a signal from the environment called the environment's state. In this section we
discuss what is required of the state signal, and what kind of information we should
and should not expect it to provide. In particular, we formally define a property of
environments and their state signals that is of particular interest, called the Markov
property.

In this section, by ``the state'' we mean whatever information is available to the agent.
We assume that the state is given by some preprocessing system that is nominally
part of the environment. In other words, our main concern is not with designing the
state signal, but with deciding what action to take as a function of whatever state
signal is available.

We now formally define the Markov property for the reinforcement learning problem.
To keep the mathematics simple, we assume here that there are a finite number
of states and reward values. Consider how a general environment
might respond at time $t+ 1$ to the action taken at time $t$. In the most
general, causal case this response may depend on everything that has happened earlier.
In this case the dynamics can be defined only by specifying the complete joint
probability distribution:

\begin{align}
P\{S_{t+1} = s', R_{t+1} =r |S_0, A_0, R_1,\dots, S_{t-1}, A_{t-1}, R_t, S_t, A_t\} \nonumber
\end{align}

for all $r, s'$ , and all possible values of the past events: $S_0 , A_0 , R_1 , ..., S_{t-1} , A_{t-1} ,
R_t , S_t , A_t$. If the state signal has the Markov property, on the other hand, then the
environment’s response at $t + 1$ depends only on the state and action representations
at $t$, in which case the environment’s dynamics can be defined by specifying only
\begin{align}
\label{kern}
p(s',r|s,a):=Pr\{S_{t+1} =s', R_{t+1}=r| S_t =s, A_t=a\} 
\end{align}
for all $r,s',s$, and $a$. In other words, a state signal has the Markov property, and is a
Markov state, if and only if (\ref{kern}) holds for all $s',r$, and histories,
$S_0, A_0, R_1, \dots, S_{t-1},\linebreak A_{t-1}, R_t, S_t, A_t$. In this case, the environment and task as a
whole are also said to have the Markov property.

If an environment has the Markov property, then its one-step dynamics (\ref{kern}) enables
us to predict the next state and expected next reward given the current state and
action. One can show that, by iterating this equation, one can predict all future
states and expected rewards from knowledge only of the current state as
would be possible given the complete history up to the current time. It also follows
that Markov states provide the best possible basis for choosing actions. That is, the
best policy for choosing actions as a function of a Markov state is just as good as
the best policy for choosing actions as a function of complete histories.

\subsubsection{Markov decision process}
A reinforcement learning task that satisfies the Markov property is called a Markov
decision process, or MDP. If the state and action spaces are finite, then it is called a
finite Markov decision process (finite MDP).
A finite MDP is defined by its state and action sets and by the one-step
dynamics of the environment. 
These quantities completely specify the dynamics of a finite MDP. Most of the theory
we present in the rest of this section implicitly assumes the environment is a finite MDP.
Given the dynamics as specified by (\ref{kern}), one can compute anything else one might
want to know about the environment, such as the expected rewards for state-action
pairs,
\begin{align}
r(s,a):=E[R_{t+1}|S_t=s, A_t=a] = \sum_{r\in \mathcal{R}}r\sum_{s' \in \mathcal{S}}p(s',r|s,a), \nonumber
\end{align}
the state-transition probabilities, 
\begin{align}
p(s'|s,a):= Pr\{S_{t+1}=s'|S_t=s,A_t=a\}=\sum_{r \in \mathcal{R}}r p(s',r|s,a)  \nonumber
\end{align}
and the expected rewards for state-action-next-state triples, 
\begin{align}
r(s,a,s') := E[R_{t+1}|S_t=s,A_t=a, S_{t+1}=s']=\frac{\sum_{r \in \mathcal{R}}r p(s',r|s,a)}{p(s'|s,a)}. \nonumber
\end{align}

\subsubsection{Bellman Equation} We mentioned that the goal of RL agent is to maximize 
a certain long-run reward criterion. 
received rewards in the long-run criterion. The infinite horizon expected discounted reward is a standard reward criteria. 
The state-value function of a policy, $\pi: \mathcal{S} \to \mathcal{A}$, is defined as:
\begin{align}
V^{\pi}(s)=E\left[\sum_{t=0}^{\infty}\gamma^tR_{t+1} | S_0=s, \pi \right], \nonumber 
\end{align}
where $\gamma \in [0,1)$ is the discount rate and $E[.]$ denotes expectation over random samples,
which are generated by following policy $\pi$.

Let $P^{\pi}$ denote the state-state transition probability matrix and $V^{\pi} \in \mathbb{R}^d$ be the value 
function vector, whose $s$-th element is $V^{\pi}(s)$. $V^{\pi} \in \mathbb{R}^{|S|}$
satisfies the following Bellman equation \cite{bert3,suttonb}:
\begin{align}
V^{\pi} = R^{\pi} + \gamma P^{\pi}V^{\pi}:= T^{\pi} V^{\pi} \nonumber 
\end{align}
where $R^{\pi}$ is the vector with components $E[R_{t+1} | S_t =s]$, and $T^{\pi}$ is known as the Bellman
operator for the policy $\pi$.

Analogously, we can also define the action-values, $Q^{\pi}(s,a)$ which evaluate the value of taking
action $a$ where the initial state is  $s$ while for the other subsequent states, actions are chosen according to 
the policy $\pi$: 
\begin{align}
Q^{\pi}(s,a) = E\left[\sum_{t=0}^{\infty} \gamma^t R_{t+1} | S_0 =s, A_0=a, \pi \right]. \nonumber
\end{align}

So far we have formulated the problem of policy evaluation or prediction. However, solving
a reinforcement learning problem roughly means finding an optimal policy that achieves a large long-term
 reward or value. An optimal policy is a policy whose value function is 
better than (not necessarily strictly) the value function of other policies. The value function of an optimal
policy is called the optimal value function. The optimal state-value function, denoted $V^*$, 
is obtained as:
\begin{align}
V^*(s):= \max_{\pi} V^{\pi}(s), \forall s \in \mathcal{S}, \nonumber
\end{align}
and the optimal action-value function, denoted by $Q^*$ is obtained as
\begin{align}
Q^*(s,a):= \max_{\pi} Q^{\pi}(s,a), \forall s \in \mathcal{S}, \forall a \in \mathcal{A}. \nonumber
\end{align}

In the next section, we review some of the most popular temporal-difference (TD) learning
algorithms, also known as classical TD methods. One of the key properties of classical TD
methods is their ability to learn from every single fragment of experience without waiting
for the final outcome. Once the policy evaluation is done, the full control problem can be solved by
on-policy TD control method such as Sarsa, see \cite[Chapter 6.4]{suttonb} for details.

\subsubsection{Temporal Difference Learning}

TD learning is a key algorithm for prediction (i.e. the problem of estimating the value function corresponding to a 
given stationary policy $\pi$) and plays central role in reinforcement learning
\cite{sutton_td, suttonb}. It uses bootstrapping ideas developed in dynamic programming 
as well as Monte Carlo simulation. Classical TD methods such as TD($\lambda$),
Sarsa, and Q-learning are simple, sample-based, online, and incremental algorithms and as
such are popular in the RL community. 

TD methods use each fragment of experience to update the
value of state $S_t$, $V_t(S_t)$, at time $t$. This would allow moment-to-moment prediction. This
is different from the dynamic programming (DP) approach in the sense that the value of each
state, in DP approach, is updated by sweeping over next states. In the next paragraph, we
briefly describe  the simplest TD method.

\paragraph{Tabular TD(0) algorithm for estimating $V^{\pi}$:} The simplest TD method, known as
tabular TD(0), estimates the value of each individual state; e.g. $S_t$, according to the following update:
\begin{align}
\label{td(0)}
V_{t+1}(S_t) &= V_t(S_t) + \alpha_t \left[R_{t+1} + \gamma V_t(S_{t+1}) - V_t(S_t)\right], \\
             &= V_t(S_t) + \alpha_t S_t,
\end{align}
where $S_t = R_{t+1} + \gamma V_t(S_{t+1}) - V_t(S_t),$
is the one-step TD error, or in short TD error, and $\alpha_t$ is a deterministic positive step-size parameter, 
which is typically small, and for the purpose of convergence analysis is assumed to 
satisfy the Robbins-Monro conditions: $\sum_{t=0}^{\infty} \alpha_t = \infty, \sum_{t=0}^{\infty} \alpha^2_t < \infty$. 
Tabular TD(0) is guaranteed to converge to $V^{\pi}$ under standard conditions. 

\paragraph{Temporal-Difference learning with function approximation}
When the number of states and actions is excessive, estimating the true value function for any policy $\pi$ can be 
computationally a nightmare. In such a case, it makes sense to consider a parametrization of the value function using 
low-dimension parameters. There is a version of TD with function approximation that we describe below.

The TD(0) algorithm (\ref{td(0)}), can be combined with parametrized value function $V_{\theta}, \theta \in \mathbb{R}^d$.
The value function can be either linear or nonlinear and differentiable function (such as a 
neural network) with respect to the parameter vector, $\theta$. The resulting algorithm has the
following update rule:
\begin{align}
\theta_{t+1} =  \theta_{t} + \alpha_t \delta_t(\theta_t)\nabla V_{\theta_t}(S_t) \nonumber
\end{align}
where 
\begin{align}
\delta_t(\theta_t) = R_{t+1} + \gamma V_{\theta_t}(S_{t+1}) - V_{\theta_t}(S_{t}) \nonumber
\end{align}
and $\nabla V_{\theta}(s) \in \mathbb{R}^d$ denotes the gradient of $V_{\theta}$ with respect to $\theta$ at $s$. 
We call this algorithm linear/nonlinear TD(0). Also, the TD(0)-solution (or TD(0)-fixpoint), $\theta$, (if it exists), satisfies:
\begin{align}
E\left[\delta_t(\theta)\nabla V_{\theta}(S_t)\right]=0. \nonumber
\end{align}

For simplicity, we adopt the following notation for TD error:
\begin{align}
\delta_t \equiv \delta_t(\theta_t). \nonumber
\end{align}
Sutton’s TD($\lambda$) with linear function approximation \cite{sutton_td} is one of the simplest
forms of TD learning with function approximation and, since its development, has played a 
central role in modern reinforcement learning. To begin first, we review the simplest form
of TD($\lambda$) with linear function approximation, that is referred to as linear TD(0). In the beginning,
let’s limit ourselves to on-policy training data. This refers to the case 
where the value function to be learned is for the same policy that is used for picking actions 
in any state that is visited by the Markov chain. The case when these policies are different is more 
interesting and is referred to as the off-policy case.

The linear TD(0) algorithm, starts with an arbitrary parameter vector, $\theta_0$. Upon observing
the $t$-th transition from state $S_t$ to $S_{t+1}$ (on-policy transitions), which follows with feature-
vector observation $(\phi_t , R_{t+1} , \phi_{t+1})$, where $\phi_t := \phi(S_t)$, the learning parameter vector is
updated according to
\begin{align}
\theta_{t+1} = \theta_t + \alpha_t \delta_t \phi_t, \nonumber
\end{align}
where $\delta_t = R_{t+1} + \gamma \theta_t^T\phi_{t+1} - \theta_t^T\phi_{t}$
is the TD-error with linear function approximation. If the
discount factor $\gamma$ is zero, the problem becomes one of supervised learning and the linear TD(0) update
rule becomes the conventional least-mean-square (LMS) algorithm in supervised learning. Thus,
the key feature that distinguishes reinforcement learning from supervised learning is the 
existence of the bootstrapping term, $\theta_t^T \phi_{t+1}$, in the above update rule. This will allow the
algorithm to guess about the future outcome without waiting for it - i.e.  allows it to learn from 
single fragment of experience without waiting for the final outcome. Also, this is the key
difference between TD learning and conventional Monte-Carlo methods (which are based
on supervised learning ideas). Thus, TD methods have the ability to learn from single
transitions without waiting for the final outcome; they do this by guessing from a guess!
We will see how this key idea allows TD to learn from off-policy data.
It is well known that the linear TD(0) algorithm is convergent under on-policy training
(\cite{tsit, tadic_t}). From the theory of stochastic methods, the
convergence point of linear TD(0), is a parameter vector, say $\theta$, that satisfies
\begin{align}
\label{td_0}
0 = E\left[\delta_t(\theta) \phi_t\right] = b- A\theta
\end{align}
where 
\begin{align}
\delta_t(\theta) = R_{t+1} + \gamma \theta^T\phi_{t+1} - \theta^T\phi_{t}, \nonumber \\
A= E\left[\phi_t(\phi_t-\gamma \phi_{t+1})^T\right], b=E\left[R_{t+1}\phi_t\right] \nonumber
\end{align}
and the expectation is over all random samples. In this thesis, the parameter vector $\theta$, which satisfies 
the above equation, is called the TD(0) solution.

In general, the TD-solution refers to the fixed-point of the expected TD update.
But under off-policy training, if this fixed-point exists, it may not be stable. 
Off-policy learning refers to learning about one way of behaving, called the target policy,
from data generated by another way of selecting actions, called the behavior policy. The
target policy is often a deterministic policy that approximates the optimal policy. Conversely, 
the behavior policy is often stochastic, exploring all possible actions in each state
as part of finding the optimal policy.
In other words,
if $\theta$ does satisfy (\ref{td_0}), then the TD(0) algorithm in expectation may cause it to
move away and eventually diverge to infinity.

(\ref{td_0}) gives us the TD-solution in the parameter space. In the value function
space, the TD-solution, $\theta$, satisfies
\begin{align}
V_{\theta} = \Pi TV_{\theta}, \nonumber
\end{align}
where $V_{\theta}=\Phi \theta \in \mathbb{R}^{|\mathcal{S}|}$, $\Phi$ 
is the matrix whose rows are the $\phi(s)^T$ for any given row entry $s$, and $\Pi$ 
is the projection operator to the linear space. The projection operator, $\Pi$, 
takes any value function $v$ and projects it to the nearest value function representable by the function
approximator:
\begin{align}
\Pi v = V_{\theta^*}, \nonumber
\end{align}
with $\theta^* = \arg\min_{\theta}{\|V_{\theta}-v\|}^2_{\mu}$, where $\mu$ 
is state-visitation probability distribution vector whose $s$-th component, 
$\mu(s)$, represents the probability of visiting state $s$, and 
\begin{align}
\|v\|^2_{\mu} = \sum_s \mu(s) v^2(s). \nonumber
\end{align}

In a linear architecture, in which $V_{\theta} = \Phi \theta$, the projection operator is linear and independent
of $\theta$: 
\begin{align}
\Pi = \phi(\phi^TD\phi)^{-1}\phi^TD \nonumber
\end{align}
where $D$ is a diagonal matrix whose diagonal elements are $\mu(s)$.

\paragraph{Derivation of TD(0) with function approximation} 
In this section, we  briefly overview the derivation of TD(0) with function approximation. The
purpose of this section it to show how linear/nonlinear TD(0) has been derived to get a better
understanding of why it may diverge. Particularly, we  will show that linear/nonlinear TD(0) is
not a true gradient-descent method, which (partially) might suggest why TD methods with
function approximation are not robust for general settings.

To start, let’s consider the following mean-square-error (MSE) objective function:
\begin{align}
MSE(\theta) = E\left[(V^{\pi}(S_t) - V_{\theta}(S_t))^2\right]. \nonumber 
\end{align}
Following gradient-descent methods, the learning update can be obtained by adjusting the
modifiable parameter $\theta$ along the steepest descent direction of the MSE objective function;
that is, $\theta_{\mbox{new}} - \theta_{\mbox{old}} \propto -\frac{1}{2}\nabla \mbox{MSE}(\theta)|_{\theta = \theta_{\mbox{old}}}$,
where 
\begin{align}
-\frac{1}{2} \nabla \mbox{MSE}(\theta_{\mbox{old}}) \nonumber \\
= -\frac{1}{2} \nabla (E\left[(V^{\pi}(S_t) - V_{\theta}(S_t))^2\right])_{\theta = \theta_{\mbox{old}}} \nonumber \\
=E[(V^{\pi}(S_t) - V_{\theta_{\mbox{old}}}(S_t))\nabla V_{\theta}(S_t)|_{\theta = \theta_{\mbox{old}}} \label{samp}
\end{align}
assuming that the interchange between the gradient and the expectation above can be justified.
Nonetheless, in general,  it is not practical to compute the above update term because: 1) The target value,
$V_{\pi}(S_t)$, is not known; 2) We do not have access to model of the environment and therefore
cannot compute the expectation term.

We can get around the first problem by approximating the target value:
\begin{align}
V^{\pi}(s) \approx E\left[R_{t+1} + \gamma V_{\theta}(S_{t+1}) | S_t =s,\pi\right]. \nonumber
\end{align}
This is called the bootstrapping step. To get around the second problem we use the theory
of stochastic approximation; that is, at every time-step we conduct a direct sampling from
(\ref{samp}) and update the parameters along (stochastic) direction of
\begin{align}
(V^{\pi}(S_t) - V_{\theta}(S_t))\nabla V_{\theta}(S_t). \nonumber
\end{align}

Putting these all together, we get the linear/nonlinear TD(0) algorithm:
\begin{align}
\theta_{t+1} = \theta_t + \alpha_t \delta_t \nabla V_{\theta_t}(S_t), \nonumber 
\end{align}
where $\delta_t$ is the one-step TD error.

\paragraph{The stability problem of linear TD(0)}
In the previous section, we overviewed the derivation of one of the simplest and popular TD
methods; that is TD(0), in conjunction with function approximation. In this section, we raise
one of its most outstanding problems, that is, the stability problem.

In the next paragraph, we consider one of the well-known counterexamples, which shows
the divergence of linear TD(0) and its approximate dynamic programming counterpart (updating
by sweeping over the states instead of sampling).

\paragraph{Baird’s Off-policy Counterexample:} Consider the 7-star version of the ``star'' counterexample
\cite{baird}. The Markov decision process (MDP) is
depicted in Fig. \ref{baird_f}. The reward is zero in all transitions, thus the true value function for
any given policy is zero; for all states.
The behavior policy, in this example, chooses the solid line action with probability of 1/7
and the dotted line action with probability of 6/7. The goal is to learn the value of a target
policy that chooses the solid line more often than the probability of 1/7. In this example,
the target policy chooses the solid action with probability of 1.

The value functions are approximated linearly in the form of $V(i) = 2\theta(i) + \theta_0$, for 
$i \in \{1,2,\dots,6\}$, and $V(7) = \theta(7) + 2 \theta_0$. Here, the discount factor is $\gamma= 0.99$. 
The TD solution, in this example, is $\theta(i)=0, i \in \{1,2,\dots,7\}$, and $\theta_0 =0$. 
Both TD(0) and DP
(with incremental updates), however, will diverge on this example; that is, their learning
parameters will go to $\pm \infty$  as is illustrated in Fig. \ref{baird_g}.

\begin{figure}
 \centering
\includegraphics[width=3 in]{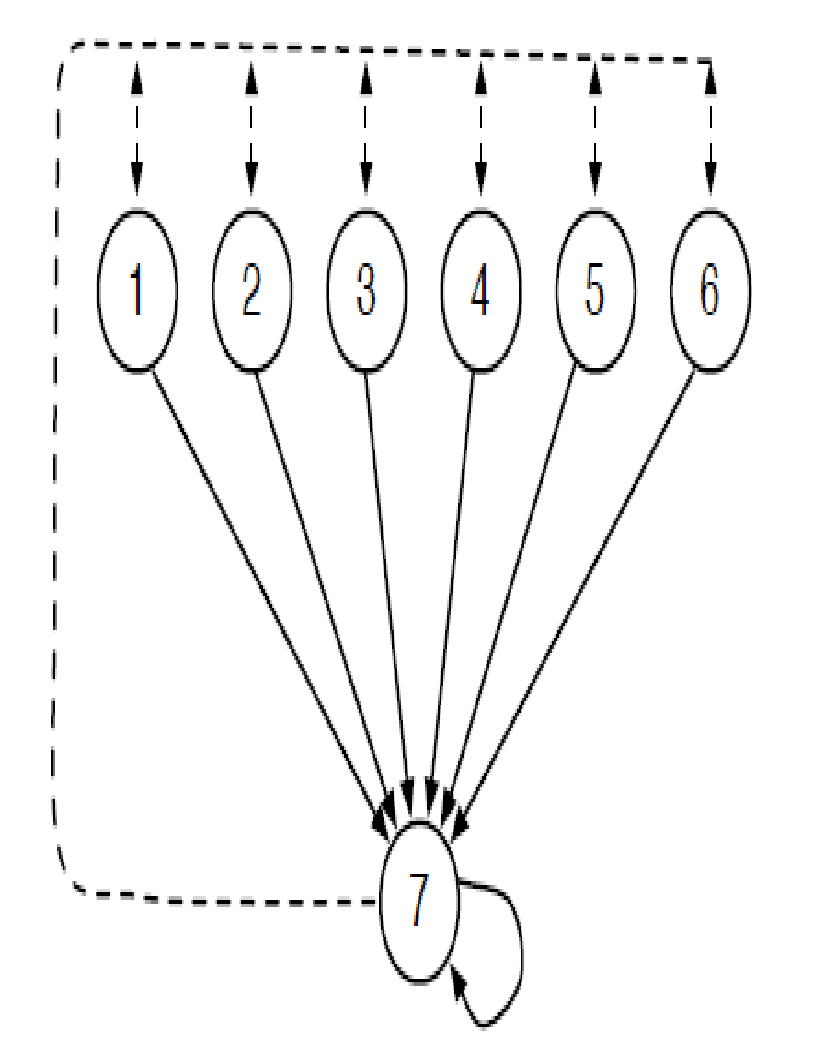}
\caption{The Baird’s 7-star MDP. Every transition in this MDP receives zero reward. Each state,
has two actions, represented by solid line and dotted line, respectively the solid line action only makes transition to
state 7, while the dotted line action uniformly makes a transition to one of the states 1-6 with probability
of 1/6. This figure is taken from \cite[p~17, Fig. 2.4]{maeith}.}
\label{baird_f}   
\end{figure}

\begin{figure}
 \centering
\includegraphics[width=3 in]{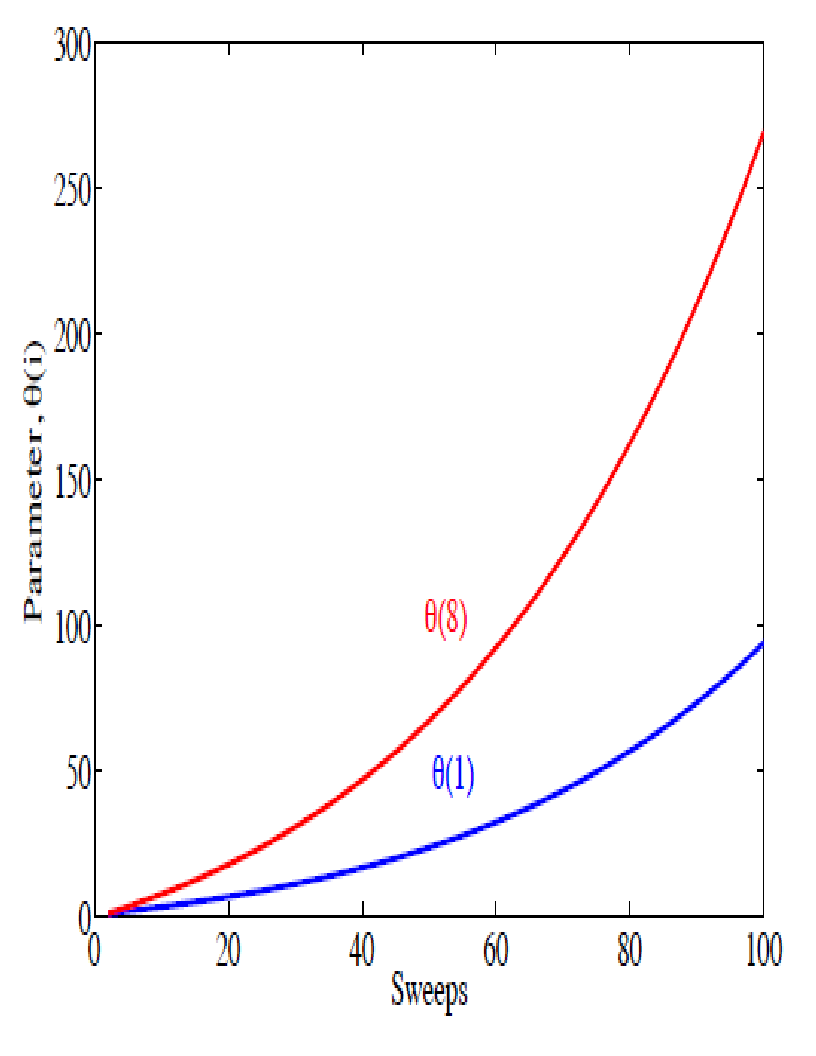}
\caption{The learning parameters in Baird’s counterexample diverge to infinity. The parameters
are updated according to the expected TD(0) update (similar to dynamic programming). 
This figure is taken  from \cite[p~18, Fig. 2.5]{maeith}.}
\label{baird_g}   
\end{figure}

Why is TD(0) with linear/nonlinear function approximation sometimes unstable? To address
this question, first we need to look at the way the algorithm is derived. Unlike supervised
learning methods, which use a mean-square-error objective function, in RL, the
objective is to estimate value functions that satisfy the Bellman equation. This would be
straightforward to do if we use tabular representation. However, for the case of function approximation,
it is not clear what is the underlying equation for the approximate value functions.
In other words, it is not straight forward to do function approximation.

TD methods with function approximation are proposed as a way of conducting this approximation. 
In the previous section, we showed how linear/nonlinear TD(0) originally has been
derived. To do this, we started with the idea of updating the learning parameters along
the gradient-descent direction of the mean-square error objective function, but with an approximation
step, $V^{\pi}(s) \approx E\left[R_{t+1} + \gamma V_{\theta} (S_{t+1}) | S_t =s\right]$. 
Thus, the resulting algorithm
would not be true stochastic gradient-descent method. As a result, TD(0) with function
approximation may diverge. 

A good way to see why linear/nonlinear TD(0) is not a true gradient-descent method is
to show that its update cannot be derived by taking the gradient of any function \cite{barnard}. To do this, let’s consider linear TD(0) with the update term, $\delta(\theta)\phi$, where 
$\delta(\theta) = R + \gamma \theta^T\phi' - \theta^T\phi$, $\phi \equiv \phi(S_t)$, $\phi' \equiv \phi(S_{t+1})$, 
where $t$ represents the time-step.

Now, let’s assume that there exists a function $J(\theta)$ whose gradient is the TD(0) update, 
$\delta(\theta)\phi$, that is, $\nabla J(\theta) = \delta(\theta)\phi$. Thus, the $j$-th element of 
$\nabla J(\theta)$ is 
\begin{align}
\frac{\partial J(\theta)}{\partial \theta_j} = \delta(\theta) \phi_j. \nonumber
\end{align}

Now, if we take another partial derivative with respect to the $i$-th component, we get
\begin{align}
\frac{\partial^2 J(\theta)}{\partial \theta_i \partial \theta_j} \neq \frac{\partial^2 J(\theta)}{\partial \theta_j \partial \theta_i} \nonumber
\end{align}

This is a contradiction, because the second derivative of a differentiable function is independent
of the order of derivatives. This derivation, however, shows that the second derivative
of function $J$ is not symmetric with respect to the order of derivative. As such, we conclude
linear TD(0) update is not a gradient of any function.

\paragraph{Objective Function for
Temporal-Difference Learning}
An objective function is some function of the modifiable parameter $\theta$ that 
we seek to minimize by updating $\theta$. In (stochastic) gradient-descent, the updates to $\theta$ are proportional to
the negative (sample) gradient of the objective function with respect to $\theta$. In standard RL,
the objective is to find a solution that satisfies the Bellman equation. However, in the case of
function approximation, it is not clear how to combine the Bellman equation with value
function approximation. In the previous section, we showed
that the TD-solution for linear TD(0) satisfies
\begin{align}
V_{\theta} = \Pi T V_{\theta}. \nonumber 
\end{align}
Thus, a nice choice for the
objective function would be to take the mean-square projected Bellman-error (MSPBE)
objective function:
\begin{align}
J(\theta) = \|V_{\theta} - \Pi T V_{\theta}\|^2_{\mu}. \nonumber 
\end{align}
GTD \cite{sutton1}, GTD2, TDC\cite{sutton} algorithms are off-policy evaluation algorithms based on this objective function. 
\subsection{Risk-sensitive reinforcement learning}
\label{risks}
The most familiar metrics in infinite horizon sequential decision problems are additive costs such as   
discounted cost and long-run average cost 
respectively. However, there is another cost criterion namely 
multiplicative cost (or risk-sensitive cost as it is better known) which has 
important connections with dynamic games and robust
control and is popular in certain applications, particularly related 
to finance where it offers the advantage of `penalizing all
moments', so to say, thus capturing the `risk' in addition to
mean return (hence the name).
In the following we summarize a few key concepts (which is based on \cite{risk}) to analyzing attitude toward risk 
taking in decision analysis practice, with particular emphasis on the use of the exponential 
utility function.

\subsubsection{Certainty Equivalence and the Idea Underlying Utility Functions}
A difficulty with decision-making under uncertainty is illustrated by the following: 
Suppose we are offered an alternative with equal chances of winning Rs. 10,000 or losing
Rs. 1,000. How much are we willing to pay for this? We certainly will not pay more
than Rs. 10,000, and we will certainly take the alternative if someone offers to give 
Rs. 1,000 in addition to the alternative. How can we settle on a number somewhere
between these two extremes?
A decision maker's attitude toward risk taking is addressed with the concept
of the certainty (or certain) equivalent, which is the certain amount that is equally
preferred to an uncertain alternative. If certainty equivalents are known for the
alternatives in a decision, then it is easy to find the most preferred alternative: It
is the one with the highest (lowest) certainty equivalence if we are considering profit
(cost).
The ``Weak Law of Large Numbers'' argues for using expected
values as certainty equivalents when the stakes in a decision under uncertainty are
small. This Law shows that under general conditions the average outcome for a large
number of independent decisions stochastically converges to the average of the expected 
values for the selected alternatives in the decisions (the term stochastically
converges means that the probability the actual value will differ from the expected
value by any specified amount gets closer to zero as the number of independent decisions 
increases). Thus, if we value alternatives at more than their expected values,
we will lose money over many decisions since we will only sell such alternatives
for more than they will return on average. Similarly, if we value alternatives at less
than their expected values, we will lose money because we will sell alternatives for
less than they will return on average.

However, additional factors enter when the stakes are high. Most of us would
be willing to pay up to the expected value of Rs. 2.50 for a lottery ticket giving us a
50:50 chance of winning Rs. 10.00 or losing Rs. 5.00. On the other hand, most of us would
not be willing to pay as much as Rs. 25,000 for a lottery ticket with a 50:50 chance
of winning Rs. 100,000 or losing Rs. 50,000 even though Rs. 25,000 is the expected value of
this lottery. This is because a few Rs. 50,000 losses would leave most of us without
the resources to continue. We cannot ``play the averages'' over a series of decisions
where the stakes are this large, and thus considerations of long-run average returns
are less relevant to our decision making.
Many conservative business people are averse to taking risks. That is, they
attempt to avoid the possibility of large losses. Such individuals have certainty
equivalents that are lower than the expected values of uncertain alternatives if we
are dealing with profits, or higher than the expected values if we are dealing with
costs. That is, these individuals are willing to sell the alternatives for less than these
alternatives will yield on average over many such decisions in order to avoid the risk
of a loss. Intuitively, we might consider incorporating this aversion toward risk into
an analysis by replacing expected value as a decision criterion by something else
which weights less desirable outcomes more heavily. Thus, we might replace the
expected value of alternative $A$, i.e.
\begin{align}
E(x|A) = \sum_{i=1}^n x_i p(x_i|A) \nonumber
\end{align}
as a decision criterion by the expected value of some utility function $u(x)$, that is 
\begin{align}
E[u(x)|a] = \sum_{i=1}^n u(x_i) p(x_i|A) \nonumber
\end{align}
where $p(x_i|A)$ is the probability of $x_i$ given that $A$ is selected.

If $x$ is total assets in hundreds of thousands of dollars, then we might have
 $u(x) = \log (x+1)$. With this utility function, higher asset positions will not receive
as much weight as with expected value and very low asset positions will receive large
negative weight. This will tend to favor alternatives that have lower risk even if they
also have lower expected values.

The certainty equivalent CE can be determined if a utility function $u(x)$ is
known using the relationship $u(\mbox{CE}) = E[u(x)|A]$ where $E[u(x)|A]$ is the expectation
of the utility for alternative $A$. As an example, consider again the decision above
which has equal chances of either winning Rs. 100,000 or losing Rs. 50,000, and suppose
that the decision maker's initial asset position is Rs. 100,000. The expected value of this
alternative in terms of total assets is $0.5 \times Rs. 200, 000 + 0.5 \times Rs. 50, 000 = Rs. 125, 000$.
Using the logarithmic utility function shown in the preceding paragraph, we can
solve for the certainty equivalent from $\log$(CE + 1) = 0.5 $\log$ (2 + 1) + 0.5 $\log$(0.5 + 1)
which gives CE = Rs. 112, 000. Thus, the alternative has a certainty equivalent which
is Rs. 13,000 less than the expected value of Rs. 125,000 when it is evaluated with this
utility function. This demonstrates the aversion to taking risks that was discussed
above. Note that in the above, the problem is to maximize profit, therefore the utility function 
is taken as logarithm, however, if the problem is the minimization of certain cost, then the usual 
practice is to take exponential function as utility function.
This is the basic idea underlying utility functions.

\subsubsection{Motivation for studying risk-sensitive reinforcement learning}
Consider an irreducible
aperiodic Markov chain $\{X_n\}$ on a finite state space $S = \{1, 2, \dots, s\}$, with transition matrix 
$P = [[p(j|i)]] i,j \in S$. 
Let $c:S \times S \to \mathbb{R}$ denote a prescribed `running cost' function and $C$ be the $s \times s$ matrix
whose $(i,j)$-th 
entry is $e^{c(i,j)}$. The aim is to evaluate risk-sensitive cost defined as
\begin{align}
\limsup_{n \to \infty} \frac{1}{n}\ln\left(E[e^{\sum_{m=0}^{n-1}c(X_m, X_{m+1})}]\right). \nonumber 
\end{align}That this limit exists follows from the
multiplicative ergodic theorem for Markov chains (see Theorem 1.2 of Balaji and Meyn (2000) \cite{bmrisk}, the
sufficient condition (4) therein is trivially verified for the finite state case here).

Like other cost criteria, one can propose and justify
iterative algorithms for solving the dynamic programming
equation for risk-sensitive setting \cite{bormn}. The issue we are interested in here is how to
do so, even approximately, when the exact model is either
unavailable or too unwieldy to afford analysis, but on
the other hand simulated or real data is available easily,
based on which one may hope to `learn' the solution in
an incremental fashion. 

One important point to note here is that the usual simulation based technique of calculating average cost  
does not work when the objective is a risk-sensitive cost. The reason is that average cost is defined as
\begin{align}
\lim_{n \to \infty}\frac{1}{n}E[\sum_{i=0}^{n-1}c(X_{i})], \nonumber 
\end{align}
where $c(i)$ is the cost of state $i$ and  $X_n$, $n\geq 0$ is 
an irreducible finite state Markov chain. Therefore the following iterative algorithm will almost surely converge to 
the average cost:
\begin{align}
\label{avg}
\theta_{n+1} = \theta_n +a(n)\left[c(X_n) - \theta_n\right], 
\end{align} where the step sizes satisfy the Robbins-Monro conditions. 
This follows from the ergodic theorem for irreducible Markov chains as well as the convergence analysis of 
stochastic approximation with Markov noise \cite{borkar}. On the contrary one needs to apply the
multiplicative ergodic theorem (\cite{bmrisk}) when the cost is risk-sensitive. However, this
does not have any closed-form limit. 
Moreover, one cannot even write iterative  algorithms like (\ref{avg}) in this setting because of the non-linear 
nature of the cost.
Due to the same 
reason, methods of \cite{marbach} also don't work in this setting when one is solving the full control 
problem.

This takes us into the domain of
reinforcement learning. In \cite{borkarq} and \cite{borkar_actor}, 
Q-learning and actor-critic methods have been proposed respectively for such a cost-criterion. 
These are `raw’ schemes in the sense that there is
no further approximation involved. Since complex control
problems lead to dynamic programming equations in very
large dimensions (`curse of dimensionality'), one often
looks for an approximation based scheme. One such learning algorithm with function approximation is
proposed in \cite{basu}. 

\subsection{Motivation for the problems considered in this thesis}
In this section we provide the motivation for the problems considered in this thesis. In Section \ref{highl} we 
summarize our solutions to these problems.
\subsubsection{Problem 1} There are many reinforcement learning applications
(precisely those where parameterization of value function is implemented) where non-additive Markov noise is present in one or both iterates thus requiring the current two-time scale framework to be extended to
include Markov noise (for example, in \cite[p.~5]{sutton2} it is mentioned that in order to generalize the analysis 
to Markov noise, the theory of two time-scale stochastic approximation needs to include the latter). This is 
our prime motivation to show a general analysis of convergence of stochastic iterates with Markov noise.
\subsubsection{Problem 2}
All learning control methods face a dilemma: They seek to learn action values conditional 
on subsequent optimal behavior, but they need to behave non-optimally in
order to explore all actions (to find the optimal actions). How can they learn about
the optimal policy while behaving according to an exploratory policy? The on-policy
approach is actually a compromise-it learns action values
not for the optimal policy, but for a near-optimal policy that still explores. A more
straightforward approach is to use two policies, one that is learned about and that
becomes the optimal policy, and one that is more exploratory and is used to generate behavior. 
The policy being learned about is called the \textit{target policy}, and the
policy used to generate behavior is called the \textit{behavior policy}. In this case we say
that learning is from data ``off'' the target policy, and the overall process is termed
\textit{off-policy learning}. It is well known that popular off-policy learning algorithms such as Q-learning
may diverge when function approximation is deployed \cite{baird}. Therefore as earlier, the off-policy learning
involves two-step procedure: off-policy prediction and off-policy control.

We consider the problem of estimating the value function corresponding to a target policy given the 
realization of a finite state Markov decision process under a behaviour policy which is different from the 
target policy. This is well known in literature as the off-policy evaluation problem. 
This is often the case 
when the simulation device can simulate states only according to a certain preset policy and 
cannot adaptively take in newer policies as they are computed. 
See \cite{suttonb} for 
additional uses.

It is well-known that for this problem the standard temporal difference learning with linear 
function approximation may diverge (\cite{baird}, \cite[Section 3]{emphatic_td}). Further, 
the usual single time-scale stochastic approximation kind of argument may not be useful 
as the associated ordinary differential equation (o.d.e) may not have the TD(0) solution as its globally asymptotically stable 
equilibrium.
In \cite{sutton1,sutton,maeith} the 
gradient temporal difference learning (GTD) algorithms were proposed to solve this problem. The 
per time-step computational complexity for these algorithms scales only linearly in the size $d$ of
the function approximator.
However, the authors 
make the assumption that either 
\begin{enumerate}
 \item one uses ``sub-sampling'' (see \cite[Section 4.1]{maeith},\cite{sutton1} for details) to filter 
the data relevant to target policy given the 
trajectory corresponding to behaviour policy, or
 \item the data itself is available in the off-policy setting i.e. one has direct access 
to quadruples of the form (state, action, reward, next state)
where the first component of the quadruples are sampled independently from 
the stationary distribution of the underlying Markov chain corresponding to the behaviour policy 
and the quadruples are formed according to the target policy.      
\end{enumerate}
Amongst all algorithms with the above assumptions, the TDC (temporal difference learning with gradient correction) algorithm 
was empirically 
found to be the most efficient in terms of the rate of convergence. It was shown in \cite{sutton} that such an 
algorithm can be proved to be convergent using the classical convergence proof for two time-scale stochastic approximation with 
martingale difference noise \cite{borkartt}. The reason for using two time-scale framework for the TDC algorithm is to make sure that the 
O.D.E's have globally asymptotically stable equilibrium. However, one can prove the convergence using single time-scale 
convergence analysis as in \cite[Theorem 3]{maeith}; however the extra condition on the step-size ratio $\eta$ 
mentioned there is hard to verify as the stationary distribution there is typically unknown or is hard to 
compute particularly in the face of large state/action spaces.

Note that such works incorporate the off-policy issue into the data as they don't take the full behaviour trajectory 
as input to the algorithm. The assumptions therein on off-policy algorithms   
are highly restrictive as 
\begin{enumerate}
 \item although in the first case the Markov chain sampled at increasing stopping times 
is time-homogeneous, its transition probabilities will be different from those of the Markov chain corresponding to 
behaviour policy.  
Further, we are interested in an \textit{online} learning scheme. Also,
 \item the second situation is not realistic too as the aforementioned 
 stationary distribution is usually unknown; 
one has access to only the trajectory corresponding to behaviour policy
from which the goal is to evaluate the target policy.    
\end{enumerate}
Keeping this in mind, another algorithm introduced in \cite{maeith}, namely, TDC with importance weighting solves the above off-policy 
evaluation problem in a more realistic scenario. The idea is to handle the off-policy issue in the algorithm rather than in the data by weighting 
the updates by the likelihood of action taken by the target policy (as opposed to the behavior policy). 
The advantage is that, unlike sub-sampling, here all the data from the given trajectory corresponding 
to the behaviour policy is used which is necessary in an \textit{online} learning scenario. 
Another advantage of this method is that we can allow both the behaviour and target policies to be
randomized unlike the sub-sampling scenario
where one can use only a deterministic policy to be a target policy.
However, to the best of 
our knowledge, both its theoretical and empirical convergence properties have not yet been analyzed. Note that   
one cannot represent the algorithm in the usual two time-scale stochastic approximation  framework to prove its convergence and 
one needs to extend such a framework to non-additive Markov noise and additive martingale difference noise. The Markov noise 
appears in the algorithm as the full trajectory of the realization of the  underlying Markov decision process corresponding to 
the behaviour policy and is taken as input to the algorithm.
\subsubsection{Problem 3} As mentioned in the Section \ref{markov}, the most important assumption
to prove convergence of stochastic approximation algorithms with Markov noise 
is the \textit{stability} of the iterates, i.e.,
\begin{align}
\label{stab}
\sup_n \|\theta_n\| < \infty \ a.s.
\end{align}
In the literature sufficient conditions that guarantee (\ref{stab}) are available (e.g. based on a Lyapunov function \cite[Chap. 6.7]{kushner}, \cite{andrieu1, 
andrieu2} and
scaled trajectory \cite[Chap 6, Theorem 9]{borkar}), etc. As mentioned in \cite{andrieu1, andrieu2}, proving stability of the iterates
is a tedious task with the Markovian dynamics due to the noise term $f(\theta_n,Y_{n+1}) - h(\theta_n)$.
In \cite{andrieu1}, the 
truncations
on adaptive truncation sets from \cite{fu_chen} has been extended to the case where the noise is Markov. 
It is clearly mentioned there that the procedure they follow is 
different
in some respects from the original
procedure proposed by \cite{fu_chen}. To prove that the number of re-initializations of the procedure described in 
\cite[Section 3.2]{andrieu1} is finite, they establish a bound on the probability that the $n$-th reinitialization time is 
finite in terms of the fluctuations of the noise sequence of the algorithm between successive re-initializations. 
In order to control the fluctuations some less classical assumptions have been imposed on the transition kernel as well 
as on the vector field $g(\cdot,\cdot)$ (see (DRI2) and (DRI3)) and the discussion thereafter. 
Further, the stability theorem  
stated in the second work also
requires assumptions such as
1) continuity of the transition kernel, 2) Lipschitz continuity of $f$ in the first component \textit{uniformly} w.r.t 
the second, and that 3) $f$ is jointly continuous.

In
this work, we investigate the dynamics of stochastic approximation with   Markov iterate-dependent noise 
when (\ref{stab}) is not known to be satisfied beforehand. We achieve the same by 
extending the  \textit{lock-in probability} framework of Borkar \cite{lock_in_original} to such a recursion. Note 
that similar results \cite{dev1,dev2} using large deviation theory were present in prior 
literature. However, assumptions made in Borkar's work 
are easily verifiable in applications.

The motivation for lock-in probability comes from a phenomenon noticed by W.B.Arthur
in simple urn models (\cite[Chap. 1]{borkar}) of increasing return economics:
if occurrences
predominantly of one type tend to fetch more occurrences of the same type, then after
some initial randomness the process gets \textit{locked into}
that type of (possibly undesirable) occurrence.
Moreover, it is known that
under reasonable conditions, every asymptotically stable equilibrium will have a positive probability
of emerging as $\lim_{n\to \infty}\theta_n$ \cite{arthur},
while this probability is zero for unstable equilibria under mild conditions on the noise \cite{brandiere,permantle}.

With this picture in mind and to give a quantitative explanation of this phenomenon,
Borkar defined lock-in probability \cite{lock_in_original} for the iterates of the form
\begin{align}
\theta_{n+1} = \theta_n + a(n)(h(\theta_n) + M_{n+1}), \label{RM}
\end{align}
where $\{M_n\}$ is a martingale difference noise sequence,
as the probability of convergence of $\theta_n$ to an asymptotically stable attractor $H$ of (\ref{ode_m})
 \textit{given} that the iterate is in a neighbourhood $B$ thereof after a \textit{sufficiently large} $n_0$, i.e.,
\begin{align}
P(\theta_n \to H | \theta_{n_0} \in B) \nonumber
\end{align}
for a compact $\bar{B}$ where $H \subset B \subset \bar{B} \subset G$ where $G$ is the domain of attraction of the local attractor. 
He also found a lower bound for this quantity by studying the \textit{local} behavior of iterates in a neighborhood of the attractor.
Clearly, $n_0$ depends on the specific $H$. Specifically, under the assumption
$E[\|M_{n+1}\|^2|\mathcal{F}_n]\leq K(1+\|\theta_n\|^2)\mbox{~a.s.}$ the bound obtained is $1-O(\sum_{i\geq n_0} a(i)^2)$
and under the more restrictive condition  $\|M_{n+1}\| \leq K_0 (1+\|\theta_n\|)\mbox{~a.s.}$, a
tighter bound of $1-O(e^{-\frac{1}{\sum_{i\geq n_0} a(i)^2}})$ has been obtained \cite{lock_in_original}. 
There are recent results \cite{gugan,sameer} which obtain tighter bounds under much weaker 
assumptions on martingale and step-size sequences.  

The fact that lock-in probability is not just a theoretical metric to explain the 
lock-in phenomenon of information economics
was shown by Kamal \cite{sameer}. If the iterates are \textit{tight} then lock-in probability results are used in
\cite{sameer} to prove almost sure
convergence of the stochastic approximation scheme albeit with 
(with only martingale difference noise) to the \textit{global} attractor.

The phenomenon described earlier can be observed in reinforcement learning (RL) applications  where the
limiting o.d.e.\ has multiple equilibria, e.g.,  several instances of stochastic gradient descent in machine learning.

We extend in our work the currently available lock-in probability estimates to the case where the vector field includes a Markov iterate-dependent
noise sequence. This is for instance the case with many reinforcement learning algorithms.


Although the recursion (\ref{q1}) covers most of the cases of stochastic approximation with  Markov iterate-dependent noise,
there are reinforcement learning scenarios where there can be a dependence on both the present and the next sample of the Markov iterate-dependent noise
in the vector field \cite{off-policy}. For such scenarios the general recursion is:
\begin{align}
\label{many}
\theta_{n+1} = \theta_n + a(n)f(\theta_n,Y_{n},Y_{n+1}).
\end{align}
One can write (\ref{many}) as
\begin{align}
\theta_{n+1} = \theta_n + a(n)\left[E[f(\theta_n,Y_{n},Y_{n+1})|\mathcal{F}_n] + M_{n+1}\right],\nonumber
\end{align}
where $\mathcal{F}_n=\sigma(\theta_m,Y_m,m\leq n)$ and $M_{n+1}=f(\theta_n,Y_{n},Y_{n+1})- E[f(\theta_n,Y_{n},Y_{n+1})|\mathcal{F}_n]$ 
is now a martingale difference sequence.
Therefore, with abuse of notation, the general recursion which takes care of 
Markov iterate-dependent noise can be
described as
\begin{align}
\label{main_m}
\theta_{n+1} = \theta_n + a(n)\left[f(\theta_n,Y_{n}) + M_{n+1}\right].
\end{align}
In fact, this also covers the situation where both Markov iterate-dependent \textit{and} martingale difference noise are present.
In this work, we give a lower bound on the lock-in probability estimate
of iterates of the form (\ref{main_m}) using the \textit{Poisson equation} based analysis
as in \cite{metivier, benveniste}. 
Under some assumptions in \cite{metivier} and some further assumptions,
we get a lower bound of $1-O(e^{-\frac{C}{\sum_{i=n_0}^{\infty}a(i)^2}})$ for the recursion (\ref{main_m}),
and thus also for the special case (\ref{q1}). Therefore, with the more general assumption of Markov iterate-dependent noise, we recover
the \textit{same bounds} available for the setting of martingale noise \cite[p.~38]{borkar1} albeit under
some additional assumptions on the Markov iterate-dependent process and step size sequence.

Very few results \cite{gen} are 
available on non-asymptotic rate of
convergence of general stochastic approximation iterates (\ref{q1}), see also \cite{rakhlin} for stochastic gradient descent and 
\cite{tdg} for finite sample analysis of temporal difference learning.
But lock-in probability estimates can be used to calculate an upper bound for the sample complexity
 of stochastic approximation \cite[chap. 4.2]{borkar},\cite{sameer}.
Given a desired accuracy $\epsilon >0$ and confidence $\gamma$, the 
sample complexity estimate is defined to be the minimum number of iterations
$N(\epsilon, \gamma)$ after which the iterates are within a certain neighbourhood
(which is a function of $\epsilon$) of $H$ with probability at least $1-\gamma$.
This is slightly different from the sample complexity estimate
arising in the context of
consistent supervised learning algorithms in statistical learning theory \cite{lorenzo}. The differences are:
\begin{enumerate}
 \item In the case of statistical learning theory, sample complexity corresponds to
the number of \textit{i.i.d training samples} needed for the algorithm to successfully learn a target function. However,
in our case, we have a \textit{recursive} scheme whose sample complexity depends on the step-size.
\item Ours is a \textit{conditional} estimate, i.e.,
 the estimate is conditioned on the event $\{\theta_{n_0} \in B\}$ where $B$ is
an open subset of the domain of attraction of $H\subset B$ and has compact closure, and $n_0$ is sufficiently large.
\end{enumerate}
Another point worth noting is that sample complexity results are much weaker than lock-in probability and do not
require existence of Lyapunov function.

\subsubsection{Problem 4} In the approximation architectures as mentioned in Section \ref{risks} 
an important problem is to obtain a good error bound 
for the approximation. This has been pointed out by Borkar in the future 
work sections of \cite{borkar_mont,basu,borkar_conf}. While \cite{basu} provides such a bound
when the problem is of policy evaluation, 
it is also mentioned there that the bound obtained 
is not good when the state space is large. 

In our work we investigate the problems with the existing bound and 
then improve upon the same. We show that good approximations are captured in our bounds whereas the earlier bound will 
infer them as bad approximation.

\subsection{Highlights of the thesis contributions}
\label{highl}
As described earlier convergence analysis of stochastic approximation with Markov Noise has been studied extensively in the seminal work of \cite{metivier,met_prior,benveniste}. The main aim of this thesis is to present convergence analysis of such recursions or such recursions with two time-scales under general assumptions which has not been done earlier, thus facilitating applications in reinforcement learning. Chapter \ref{chap:intro} - \ref{chap:risk} describe the four problems considered in this thesis. Chapter \ref{chap:con} 
describes conclusions and future directions based on the work in this thesis. In the following we summarize the contributions 
of Chapter \ref{chap:intro} - \ref{chap:risk}.
\subsubsection{Problem 1}
We present for the first time an asymptotic convergence analysis of two time-scale 
stochastic approximation driven by ``controlled'' Markov noise. In particular, the
faster and slower recursions have non-additive iterate-dependent Markov noise components that depend 
on an additional control sequence and this is  
in
addition to martingale difference noise. We analyze the asymptotic behavior of our
framework by relating it to limiting differential inclusions in both time scales that are
defined in terms of the ergodic occupation measures associated with the controlled
Markov processes. Note that the results of \cite{borkar} 
assume that the state space of the controlled Markov process is Polish which 
may impose  additional conditions that are hard to verify. In this section, other 
than proving our two time-scale results, we prove many 
of the results in \cite{borkar} (which were only
stated there) assuming the state space to be compact and thus can be easily verified. Additionally, 
we generalize the global attractor assumption of the faster o.d.e in \cite[Chapter 6, (A1)]{borkar1} to 
local attractors. We then prove almost sure convergence of the iterates  under the requirement that  
the faster iterate belongs to a compact subset of the intersection (over all $\theta$)
of the domain of attraction of the local attractors \textit{eventually}. The requirement  on the 
faster iterate is much stronger than the usual local attractor 
statement for the Kushner-Clarke lemma \cite[Section II.C]{metivier} which requires 
the iterates to enter a compact set in the domain of attraction for the local attractor 
\textit{infinitely often} only. Additionally, we prove the tracking lemma of \cite[Lemma 2.2]{borkar} 
using the Borel-Cantelli Lemma (the main idea is to 
prove $X_n \to 0 \mbox{~~a.s}$ by showing that $\forall \epsilon>0$, 
$\sum_{n=1}^{\infty}P(|X_n| >\epsilon) <\infty$ for a sequence of random variables $X_n, n\geq 0$) 
so that we can allow the Lipschitz constant of the vector fields to depend on the state space. 
Although there is a recent work \cite{vyaji} which relaxes many of the assumptions of our work 
(such as vector fields in both iterates 
as well as $\lambda(\theta)$ are set-valued Marchaud maps), however, this is done under the assumption that  
$\lambda(\theta)$ is a global attractor and the constant in the  
point-wise boundedness assumption of the Marchaud map  does not depend on the state space of the Markov process. 

This work is described in Chapter \ref{chap:intro}.

\subsubsection{Problem 2}

In this work we give a rigorous almost sure convergence  proof of TDC algorithm with importance weighting by formulating 
it into the two time-scale stochastic approximation framework with non-additive Markov noise and additive martingale 
difference noise. 
To the best of our knowledge this 
is the first time an almost sure convergence proof of off-policy temporal difference learning algorithm with linear function
approximation is presented for step-sizes satisfying the standard Robbins-Monro conditions. We also support these theoretical results 
by providing empirical results. Our results show that due to 
the aforementioned importance weighting factor, 
\textit{online} TDC with importance weighting performs 
much better than the sub-sampling version of TDC for standard off-policy counterexamples  
when the behaviour policy is much different from the 
target policy. 

Recently, emphatic temporal difference learning has been introduced in \cite{emphatic_td} to solve the off-policy evaluation 
problem. 
However, such algorithms are proven to be almost surely convergent for special step-size sequences and 
weakly convergent for a large range of step-sizes \cite{yu_new}. 

Another related work is the much complex off-policy learning algorithms that obtain
the benefits of weighted importance sampling (to reduce variance)  with $O(d)$ computational
complexity \cite{weight}. However, nothing is known about the convergence of such algorithms.
In this context, we empirically show that in the case of TDC with importance weighting 
the variance of the difference between true value 
function and the estimated one for standard off-policy counterexamples such as \cite{baird} becomes 
small eventually. 

This work is described in Chapter \ref{chap:off}.

\subsubsection{Problem 3}
This work compiles several aspects of the dynamics of stochastic approximation 
algorithms with Markov iterate-dependent noise when the
iterates are not known to be stable beforehand. We achieve the same
by extending the lock-in probability (i.e., the probability of convergence
to a specific attractor of the limiting o.d.e. given that the iterates are
in its domain of attraction after a sufficiently large number of iterations
(say) $n_0$) framework to such recursions. Specifically, with the more restrictive 
assumption of Markov iterate-dependent noise supported on a
bounded subset of the Euclidean space we give a lower bound for the lock-in probability, 
leading in turn to the following:

\begin{enumerate}
 \item Let $H$ be an asymptotically stable attractor of (\ref{ode_m}) and $G$ its domain of attraction. 
If $\{\theta_n\}$ is \textit{asymptotically tight} (which is a much weaker condition than (\ref{stab})) 
and \linebreak $\liminf_n P(\theta_n \in G) =1,$ then $P(\theta_n \to H) =1$ under reasonable 
set of assumptions satisfied in application areas such as reinforcement learning \cite{borkar_meyn}.
To the best of our knowledge this is the first time 
an almost sure  convergence proof for such recursion is presented 
without assuming the stability of the iterates, however, following the classic Poisson equation
model of Metivier and Priouret \cite{metivier} for such recursion which is designed keeping 
in mind the stability  of the iterates.
Additionally, a  simple test for asymptotic tightness is also provided.
 \item We show that for common step-size sequences such as
$\{\frac{1}{n^k}\}, \frac{1}{2} < k \leq 1$ and $\frac{1}{n (\log n)^{k}}, k\leq 1$, if
the iterates belong to some special open set with compact closure in the domain of attraction of the local attractor
infinitely often w.p. 1, the iterates are stable and converge a.s. to the local attractor.
 \item We show that our results can be used to analyze the tracking ability of 
general (not necessarily linear) stochastic approximation driven by another ``slowly'' varying 
stochastic approximation process when the iterates are not known to be stable. 
Note that in two time-scale stochastic approximation the coupled o.d.e has no attractor. 
Therefore we need to consider two 
quantities describing difference (over compact time interval) 
between algorithm and o.d.e.,  one for the coupled algorithm/o.d.e and
another for the slower algorithm/o.d.e. This gives rise to a situation where the 
conditioning event in the martingale concentration
inequality will not belong to the first $\sigma$-field in the current collection of $\sigma$-fields (unlike in case of  
single timescale stochastic approximation where the conditioning event always belongs to the 
first $\sigma$-field in the current collection of $\sigma$-fields \cite[p~40]{borkar1}). 

Such results are useful in the context of 
adaptive algorithms \cite{konda} as not much is known about the stability of frameworks with different 
timescales. There is some recent work \cite{dalal} that also calculates lock-in probability for 
multiple timescales, however, under the assumption that the vector fields are ``linear''.
\item We give a sample complexity estimate
for the setting where the recursion is a \textit{stochastic fixed point} iteration driven by a Markov iterate-dependent noise. This
shows a \textit{quantitative} estimate of \textit{large vs. small step size trade-off} well known in stochastic
approximation literature that is shown to be useful in choosing the \textit{optimal step-size}.

\end{enumerate}
This work is described in Chapter \ref{chap:lock}.

\subsubsection{Problem 4}
In this work we obtain the several informative error bounds on function approximation for the policy evaluation
algorithm proposed by \cite{basu} when the aim is to find the risk-sensitive cost represented using exponential
utility. We
also give examples where all our bounds achieve the \textit{actual error} whereas the earlier bound given by \cite{basu}  
is much weaker in comparison. We show that this happens due to the absence of a difference term in the earlier
bound which is always present in all our bounds when the state space is large. Additionally, we discuss how all
our bounds compare with each other. We also describe a temporal difference learning algorithm in this 
setting and show that if the stationary distribution is available (in the case of doubly stochastic transition kernel), 
then one can choose the features appropriately so that the algorithm converges to the actual risk-sensitive cost.

This work is described in Chapter \ref{chap:risk}.
\subsection{Notations used to denote assumptions}
In the following chapters we have used \textbf{(Ai)}, \textbf{(Bi)}, \textbf{(Si)} and \textbf{(S'i)} (where $i \geq 1$) to denote 
assumptions under which the stated theorems and lemmas are true. Unless stated explicitly, \textbf{(Ai)},  \textbf{(Bi)}, \textbf{(Si)} 
and \textbf{(S'i)} will denote 
the same defined in the corresponding chapter.
\chapter{Two Time-scale Stochastic Approximation with Controlled Markov noise}
\label{chap:intro}

\section{Brief introduction and Organization}
Here we present a more general framework of two time-scale stochastic approximation
with ``controlled'' Markov noise, i.e., the noise is 
not simply Markov; rather it is driven by the 
iterates and an additional control process as well. 
We analyze the asymptotic behaviour of our framework by relating it to limiting differential inclusions in both timescales that 
are defined in terms of the ergodic occupation measures associated with the controlled Markov processes. 
To the best of our knowledge there are related works  such as \cite{Tadic, konda, konda_actor, tadic_new} 
where two time-scale stochastic approximation 
algorithms with algorithm iterate dependent non-additive Markov noise is analyzed. In all of them the 
Markov noise in the recursion is handled using the classic Poisson equation based approach of \cite{benveniste, metivier} and 
applied to the asymptotic analysis of many algorithms used in machine learning, system identification, signal 
processing, image analysis and automatic control. However, we show that our method 
also works if there is another additional control process as well and if the
underlying Markov process has non-unique stationary distributions. 
Additionally, our assumptions 
are quite different from the assumptions made in the mentioned literature and we give a detailed
comparison in Section \ref{def}.

The results described in this chapter are mainly based on the proof techniques of \cite{borkartt} and \cite{borkar}. To the best of our knowledge, we believe that the assumption of 
Polish state space of \cite{borkar} will impose extra conditions that are hard to verify. In our work, 
we prove many of the results in \cite{borkar} (which were only stated there) 
assuming the state space to be compact. Also, in our work we allow the Lipschitz constant of the vector field to depend on the state space of the Markov process which was assumed independent of the state space in \cite{borkar}. We also relax the assumption of global attractor \cite{borkartt} by allowing the o.d.e in the faster timescale to have local attractors.  

The organization of this chapter is as follows: Section \ref{secdef} formally defines the problem and 
provides background and assumptions. Section \ref{mres} shows the main
results. Section \ref{relax} discusses how one of our assumptions of Section \ref{secdef} can be relaxed. 
Finally, we conclude by providing
some future research directions. 

\section{Background, Problem Definition, and Assumptions}
\label{secdef}

In the following we describe the preliminaries and notation used in our proofs. 
Most of the definitions and notation are from \cite{benaim,borkar1,Aubin}.
\subsection{Definition and Notation}

Let $F$ denote a set-valued function mapping each point $\theta \in \mathbb{R}^m$ to a set $F(\theta) \subset \mathbb{R}^m$. $F$ is called
a \textit{Marchaud map} if the following hold: 
\begin{enumerate}[label=(\roman*)]
\item  $F$ is \textit{upper-semicontinuous} in the sense that if $\theta_n \to \theta$ and $w_n \to w$ with $w_n \in F(\theta_n)$ for all $n\geq 1$, then 
$w \in F(\theta)$. In order words, the graph of $F$ defined as $\{(\theta,w):  w \in F(\theta)\}$ is closed. 
\item $F(\theta)$ is a non-empty compact convex subset of $\mathbb{R}^m$ for all $\theta \in \mathbb{R}^m$.
\item $\exists c >0$ such that for all $\theta \in \mathbb{R}^m$,
\begin{equation}
\sup_{z\in F(\theta)} \|z\| \leq c(1+\|\theta\|),\nonumber 
\end{equation}
where $\|.\|$ denotes any norm on $\mathbb{R}^m$. 
\end{enumerate}


\textit{A solution for the differential inclusion (D.I.)} 
\begin{equation}
\label{diffin}
\dot{\theta}(t) \in F(\theta(t)) 
\end{equation}
with initial point $ \theta_0 \in \mathbb{R}^m$ is an absolutely continuous (on compacts) 
mapping  $\theta :  \mathbb{R} \to \mathbb{R}^m$ such that
$\theta(0) =\theta_0$ and 
\begin{equation}
\dot{\theta}(t) \in F(\theta(t))\nonumber 
\end{equation}
for almost every $t \in \mathbb{R}$. 
If $F$ is a Marchaud map, it is well-known that (\ref{diffin}) has solutions 
(possibly non-unique) through every initial point. The differential inclusion  (\ref{diffin}) 
induces a \textit{set-valued dynamical system} $\{\Phi_t\}_{t\in \mathbb{R}}$ defined by
\begin{equation}
\Phi_t(\theta_0) = \{\theta(t) : \theta(\cdot) \mbox{ is a solution to  
(\ref{diffin}) with $\theta(0) =\theta_0$}\}.\nonumber 
\end{equation}  \indent
Consider the autonomous ordinary differential equation (o.d.e.)
\begin{equation}
\label{ode1}
\dot{\theta}(t)=h(\theta(t)), 
\end{equation}
where $h$ is Lipschitz continuous. One can write (\ref{ode1}) in the 
format of (\ref{diffin}) by taking $F(\theta)=\{h(\theta)\}$. 
It is well-known that (\ref{ode1}) is well-posed, i.e., it has a \textit{unique solution} for every initial point.
Hence the set-valued dynamical system induced by the o.d.e. or \textit{flow} is  $\{\Phi_t\}_{t\in \mathbb{R}}$ with 
\begin{equation}
\Phi_t(\theta_0) = \{\theta(t)\},\nonumber 
\end{equation}
where $\theta(\cdot)$ is the solution to  
(\ref{ode1}) with $\theta(0) =\theta_0$.
It is also well-known that $\Phi_t(\cdot)$ is a \textit{continuous function} for all $t \in \mathbb{R}$.
\\ \indent     
A set $A \subset \mathbb{R}^m$ is said to be \textit{invariant} (for $F$) if for all 
$\theta_0\in A$ there exists a solution $\theta(\cdot)$ of (\ref{diffin})
with $\theta(0) = \theta_0$ such that $\theta(\mathbb{R}) \subset A$. 
\\ \indent
Given a set $A \subset \mathbb{R}^m$ and $\theta'',w''\in A$, we write $\theta''\hookrightarrow_A w''$ 
if for every $\epsilon > 0$ and $T>0$ $\exists n \in \mathbb{N}$, solutions $\theta_1(\cdot), \dots, \theta_n(\cdot)$ 
to  (\ref{diffin}) and real numbers $t_1, t_2, \dots, t_n$ greater than 
$T$ such that
\begin{enumerate}[label=(\roman*)]
 \item  $\theta_i(s) \in A$ for all $0 \leq s \leq t_i$ and for all $i=1, \dots, n,$
 \item  $\|\theta_i(t_i) - \theta_{i+1}(0)\| \leq \epsilon$ for all $i=1, \dots, n-1,$
 \item  $\|\theta_1(0) - \theta''\| \leq \epsilon$ and $\|\theta_n(t_n) - w''\| \leq \epsilon.$
\end{enumerate}
The sequence $(\theta_1(\cdot), \dots, \theta_n(\cdot))$ is called an $(\epsilon, T)$ 
chain (in $A$ from $\theta''$ to $w''$) for $F$. A set $A \subset \mathbb{R}^m$
is said to be \textit{internally chain transitive}, provided that $A$ is compact and $\theta'' 
\hookrightarrow_A w''$ for all $\theta'',w''\in A$. It can be 
proved that in the above case, $A$ is an invariant set. 
\\ \indent
A compact invariant set $A$ is called an \textit{attractor} for $\Phi$, provided that there is a neighbourhood $U$ of $A$
(i.e., for the induced topology) with the property that 
$d(\Phi_t(\theta''), A) \to 0$ as $t\to \infty$ \textit{uniformly} in $\theta'' \in U$.
Here $d(X, Y) = \sup_{\theta'' \in X}\inf_{w'' \in Y}\|\theta''-w''\|$ for $X,Y \subset \mathbb{R}^m.$
Such a $U$ is called a \textit{fundamental neighbourhood} of the attractor $A$.
\textit{An attractor of a well-posed o.d.e.} is an attractor for the set-valued dynamical system induced by the o.d.e.    
\\ \indent
The set 
\begin{equation}
\omega_\Phi(\theta'') = \bigcap_{t \geq 0} \overline{\Phi_{[t, \infty)}(\theta'')} \nonumber 
\end{equation}
is called the \textit{$\omega$-limit} set of a point $\theta'' \in \mathbb{R}^m$. 
If $A$ is a set, then 
\begin{equation}
B(A) = \{\theta'' \in \mathbb{R}^m :  \omega_\Phi(\theta'') \subset A\} \nonumber 
\end{equation}
denotes its \textit{basin of attraction}. A \textit{global attractor} for $\Phi$ is an attractor $A$
whose basin of attraction consists of all $\mathbb{R}^m$. 
Then the following lemma will be useful for our proofs, see \cite{benaim} for a proof.
\begin{lemma}
\label{ga}
Suppose $\Phi$ has a global attractor $A$. Then every internally chain transitive set lies in $A$.
\end{lemma}

We also require another result which will be useful to apply our results to the RL application we mention. Before stating it 
we recall some definitions from Appendix 11.2.3 of \cite{borkar1}:
\\ \indent
A point $\theta^*\in \mathbb{R}^m$ is called \textit{Lyapunov stable} for the o.d.e (\ref{ode1}) if for all $\epsilon >0$, 
there exists a $\delta >0$ such that every trajectory of (\ref{ode1}) initiated in the 
$\delta$-neighbourhood of $\theta^*$ remains in its $\epsilon$-neighbourhood. 
$\theta^*$ is called \textit{globally asymptotically stable} if $\theta^*$ is Lyapunov stable and 
\textit{all} trajectories of the o.d.e. converge to it. 
\begin{lemma}
\label{ga2}
Consider the autonomous o.d.e. $\dot{\theta}(t)=h(\theta(t))$ where $h$ is Lipschitz continuous. Let $\theta^*$ be 
globally asymptotically stable. Then $\theta^*$ is the global 
attractor of the o.d.e. 
\end{lemma}
\begin{proof}
We refer the readers to Lemma~1 of \cite[Chapter 3]{borkar1} for a proof. 
\end{proof}

We end this subsection with a notation which will be used frequently in the convergence statements in the following 
sections.
\begin{mydef}
For function $\theta(\cdot)$ defined on $[0,\infty)$, the notation ``$\theta(t) \to A$ as $t \to \infty$'' 
means that $A=\cap_{t\geq 0} \overline{\{\theta(s):s \geq t\}}$.
Similar definition applies for a sequence $\{\theta_n\}$.    
\end{mydef}

\subsection{Problem Definition}
\label{def}

Our goal is to perform an asymptotic analysis of the following coupled recursions:  
\begin{eqnarray}
\theta_{n+1} &= \theta_n + a(n)\left[h(\theta_n, w_n, Z^{(1)}_n) + M^{(1)}_{n+1}\right],\label{eqn1}\\
w_{n+1} &= w_n + b(n)\left[g(\theta_n, w_n, Z^{(2)}_n) + M^{(2)}_{n+1}\right],\label{eqn2}
\end{eqnarray}
where $\theta_n \in \mathbb{R}^d, w_n \in \mathbb{R}^k, n\geq 0$ and $\{Z^{(i)}_n\}, \{M^{(i)}_{n}\}, i=1, 2$ 
are random processes that we describe below. 
\\ \indent 
We make the following assumptions:  
\begin{enumerate}[label=\textbf{(A\arabic*)}]
 \item $\{Z^{(i)}_n\}$ takes values in a compact metric space $S^{(i)}, i=1,2$. Additionally, 
the processes $\{Z^{(i)}_n\}, i = 1, 2$ are controlled 
Markov processes that are controlled by three different control processes: the iterate sequences $\{\theta_m\}, \{w_m\}$ and a 
random process $\{A^{(i)}_n\}$ taking values in a compact metric space $U^{(i)}$
respectively with their individual dynamics specified by
\begin{equation}
P(Z^{(i)}_{n+1} \in B^{(i)} |Z^{(i)}_m, A^{(i)}_m, \theta_m, w_m, m\leq n) = \int_{B^{(i)}} p^{(i)}(dy|Z^{(i)}_n, A^{(i)}_n, \theta_n, w_n), n\geq 0, \nonumber 
\end{equation} 
for $B^{(i)}$ Borel in $S^{(i)}, i = 1, 2,$ respectively.  

\begin{remark}
In this context one should note that 
\cite{benveniste, metivier} requires the Markov process to take value in a normed Polish space. 
\end{remark}

\begin{remark}
In \cite{borkar} it is assumed that the state space where the controlled Markov Process takes values is Polish. This space is then compactified
using the fact that a Polish space can be homeomorphically embedded into a dense subset of a compact metric space. The vector 
field $h(.,.) :  \mathbb{R}^d \times S \to \mathbb{R}^d$ is considered bounded when 
the first component lies in a compact set. This would, however, require a continuous 
extension of $h': \Bbb R^d \times \phi(S) \to \Bbb R^d$ defined by $h'(x,s') = h(x,\phi^{-1}(s'))$ 
to $\Bbb R^d \times \overline{\phi(S)}$. 
Here $\phi(\cdot)$ is the homeomorphism defined by 
$\phi(s) = (\rho(s, s_1), \rho(s, s_2), \dots) \in [0,1]^{\infty}$, and 
$\{s_i\}$ and $\rho$ is a countable dense subset and metric of the Polish space
 respectively. A sufficient condition for the above  
is $h'$ to be uniformly continuous \cite[Ex:13, p.~99]{Rudin}.  
However, this is hard to verify.  
This is the main motivation for us to take the range of the Markov process
as compact for our problem. However, there are other reasons for taking compact state space which will be clear 
in the proofs of this section and the next. 
\end{remark}

\item $h :  \mathbb{R}^{d+k} \times S^{(1)} \to \mathbb{R}^d$ is  
jointly continuous as well as Lipschitz in its first two arguments uniformly w.r.t the third. The latter condition means that
\begin{equation}
\forall z^{(1)} \in S^{(1)}, \|h(\theta, w, z^{(1)}) - h(\theta', w', z^{(1)})\| \leq L^{(1)}(\|\theta-\theta'\| + \|w - w'\|).\nonumber
\end{equation}
Same thing is also true for $g$ where the Lipschitz constant is $L^{(2)}$.
Note that the Lipschitz constant $L^{(i)}$ does not depend on $z^{(i)}$ for $i=1,2$.
\begin{remark}
We later relax the uniformity of the Lipschitz constant w.r.t the Markov process state space by putting 
suitable moment assumptions on Markov process.  
\end{remark}

\item $\{M^{(i)}_n\}, i=1, 2$ are martingale difference sequences
w.r.t increasing $\sigma$-fields
\begin{equation}
\mathcal{F}_n = \sigma(\theta_m, w_m, M^{(i)}_{m}, Z^{(i)}_m, m \leq n, i = 1, 2), n \geq 0,\nonumber 
\end{equation}
satisfying 
\begin{equation}
E[\|M^{(i)}_{n+1}\|^2|\mathcal{F}_n] \leq K(1 + \|\theta_n\|^2 + \|w_n\|^2), i = 1, 2,\nonumber 
\end{equation}
for $n \geq 0$ and a given constant $K>0$.
\item The stepsizes $\{a(n)\}, \{b(n)\}$ are positive scalars satisfying
\begin{equation}
\sum_n a(n) = \sum_n b(n) = \infty, \sum_{n}(a(n)^2 + b(n)^2) < \infty, \frac{a(n)}{b(n)} \to 0.\nonumber 
\end{equation}
Moreover, $a(n), b(n), 
n \geq 0$ are non-increasing. 

Before stating the assumption on the transition kernel $p^{(i)}, i=1, 2$ we need
to define the metric in the space of probability measures $\mathcal{P}(S)$. Here we mention the definitions
and main theorems on the spaces of probability measures that we use in our proofs 
(details can be found in Chapter 2 of \cite{borkar2}). 
We denote the metric by $d$ and is defined as
\begin{equation}
d(\mu, \nu) = \sum_{j} 2^{-j}|\int f_j d\mu - \int f_j d\nu|, \mu, \nu \in \mathcal{P}(S),\nonumber 
\end{equation}
where $\{f_j\}$ are countable dense in the unit ball of $C(S)$.    
Then the following are equivalent: 
\begin{enumerate}[label=(\roman*)]
 \item $d(\mu_n, \mu) \to 0,$
 \item for all bounded  $f$ in $C(S)$, 
\begin{equation}
\int_{S} fd\mu_n \to \int_{S} f d\mu,  
\end{equation} 
\item $\forall f$ bounded and uniformly continuous, 
\begin{equation}
\int_{S} fd\mu_n \to \int_{S} fd\mu.\nonumber 
\end{equation} 
\end{enumerate}
Hence we see that $d(\mu_n, \mu) \to 0$ iff 
$\int_{S} f_jd\mu_n \to \int_{S} f_jd\mu$ for all $j$. Any such sequence of functions $\{f_j\}$ is called a convergence
determining class in $\mathcal{P}(S)$. 
Sometimes we also denote $d(\mu_n, \mu) \to 0$ using the notation $\mu_n \Rightarrow \mu$.  
\\ \indent
Also, we recall the characterization of relative compactness in $\mathcal{P}(S)$ that relies on the 
definition of tightness. $\mathcal{A}\subset\mathcal{P}(S)$ is 
a tight set if for any $\epsilon >0$, there exists a compact $K_\epsilon \subset S$ such that $\mu(K_\epsilon) > 1-\epsilon$ for all $\mu \in \mathcal{A}$. 
Clearly, if $S$ is compact then any $\mathcal{A}\subset\mathcal{P}(S)$ is tight. 
By Prohorov's theorem, $\mathcal{A}\subset\mathcal{P}(S)$ is relatively compact if and only if it is tight.
\\ \indent
With the above definitions we assume the following:
\item The map $S^{(i)} \times U^{(i)} \times \mathbb{R}^{d+k} \ni (z^{(i)}, a^{(i)}, \theta, w)  
\to p^{(i)}(dy|z^{(i)}, a^{(i)}, \theta,w) \in \mathcal{P}(S^{(i)})$ is continuous. 
\textbf{(A5)} is much simpler than the assumptions on $n$-step 
transition kernel in \cite[Part II,Chap. 2, Theorem 6]{benveniste}.

Additionally, unlike \cite[p~140 line 13]{borkar}, we do not require the extra assumption of the 
continuity in the $\theta$ variable of $p(dy|z,a,\theta)$ to be uniform on compacts w.r.t 
the other variables. This means: $\forall \epsilon >0, K \subset S \times U ~~ \mbox{and $K$ compact,} ~~\exists \delta_K$ (i.e. 
this same $\delta$ works for all $(z,a)\in K$) s.t. if 
$\|\theta-\theta'\| < \delta_K$, then $d(p(dy|z,a,\theta), p(dy|z,a,\theta')) < \epsilon ~~\forall (z,a) \in K$. 
Here $d(.,.)$ is the metric in the space of probability measures \cite[Chap. 2]{borkar2}

The above is much stronger than the combination of the state space being compact and continuity of the transition kernel. 
We only use uniform continuity 
of functions like $g'|_A$ where 
$g'(z,a,\theta) = \int_{S} f(y)p(dy|z,a,\theta)$ and $A=S \times U \times (\{\theta_n\} \bigcup \theta)$ or similar functions in our proofs. 
 

For $\theta_n = \theta, w_n = w$ for all $n$ with a fixed deterministic 
$(\theta, w) \in \mathbb{R}^{d+k}$ and under any stationary randomized control $\pi^{(i)}$, 
it follows from Lemma 2.1 and Lemma 3.1 of \cite{borkar}
that 
the time-homogeneous
Markov processes $Z^{(i)}_n, i=1, 2$ have (possibly non-unique) invariant 
distributions $\eta^{(i)}_{\theta,w,\pi^{(i)}}, i = 1, 2$. 
. 

Now, it is well-known that the ergodic occupation measure defined as 
\begin{equation}
\Psi^{(i)}_{\theta, w, \pi^{(i)}}(dz, da):= \eta^{(i)}_{\theta,w,\pi^{(i)}}(dz) \pi^{(i)}(z, da) \in \mathcal{P}(S^{(i)} \times U^{(i)}) \nonumber
\end{equation}satisfies the following: 
\begin{equation}
\label{eqn3} 
\int_{S^{(i)}}f^{(i)}(z) \Psi^{(i)}_{\theta, w, \pi^{(i)}}(dz, U^{(i)}) = \int_{S^{(i)}\times U^{(i)}}\int_{S^{(i)}}f^{(i)}(y)p^{(i)}(dy|z,a, \theta, w)\Psi^{(i)}_{\theta, w, \pi^{(i)}}(dz, da)
\end{equation}
for $f^{(i)}:  S^{(i)} \to \mathcal{R} \in C_b(S^{(i)})$.
\end{enumerate}
We denote by $D^{(i)}(\theta,w), i=1,2$ the set of all such ergodic occupation measures for the prescribed $\theta$ and $w$. In the following we prove
some properties of the map $(\theta,w) \to D^{(i)}(\theta,w)$.

\begin{lemma}
\label{lemma1}
For all $(\theta,w)$, $D^{(i)}(\theta,w)$ is convex and compact.
\end{lemma}
\begin{proof}
The proof trivially follows from \textbf{(A1)}, \textbf{(A5)} and (\ref{eqn3}).
\end{proof}

\begin{lemma}
\label{upsem}
The map $(\theta,w) \to D^{(i)}(\theta,w)$ is upper-semi-continuous.
\end{lemma}
\begin{proof}
Let $\theta_n \to \theta, w_n \to w$ and $\Psi^{(i)}_n \Rightarrow \Psi^{(i)} \in \mathcal{P}(S^{(i)} \times U^{(i)})$ such that  $\Psi^{(i)}_n \in D^{(i)}(\theta_n, w_n)$. Let 
$g^{(i)}_n(z,a) = \int_{S^{(i)}}f^{(i)}(y)p^{(i)}(dy|z,a, \theta_n, w_n)$ and  $g^{(i)}(z,a) = \int_{S^{(i)}}f^{(i)}(y)p^{(i)}(dy|z,a, \theta, w).$
From (\ref{eqn3}) we get that
\begin{align} 
\int_{S^{(i)}}f^{(i)}(z) \Psi^{(i)}(dz,U^{(i)}) &= \lim_{n\to \infty}\int_{S^{(i)}}f^{(i)}(z) \Psi^{(i)}_n(dz,U^{(i)})\nonumber\\ 
                                        &= \lim_{n\to \infty}\int_{S^{(i)}\times U^{(i)}}\int_{S^{(i)}}f^{(i)}(y)p^{(i)}(dy|z,a, \theta_n, w_n)\Psi^{(i)}_n(dz,da)\nonumber\\        
                                        &= \lim_{n\to \infty}\int_{S^{(i)}\times U^{(i)}}g^{(i)}_n(z,a)\Psi^{(i)}_n(dz,da).\nonumber
\end{align}
Now, $p^{(i)}(dy|z,a, \theta_n, w_n) \Rightarrow p^{(i)}(dy|z,a, \theta, w)$ 
implies $g^{(i)}_n(\cdot, \cdot) \to  g^{(i)}(\cdot, \cdot)$ pointwise. We prove that the 
convergence is indeed uniform. It is enough to prove that this sequence 
of functions is equicontinuous. Then along with pointwise convergence
it will imply uniform convergence on compacts \cite[p.~168, Ex: 16]{Rudin}. This is also 
a place where \textbf{(A1)} is used.  
\\ \indent
Define $g' : S^{(i)} \times U^{(i)} \times \mathbb{R}^{d+k} \to \mathbb{R}$ by $g'(z',a', \theta',w')  =  \int_{S^{(i)}}f^{(i)}(y)p^{(i)}(dy|z,a', \theta', w')$. 
Then $g'$ is continuous. Let $A= S^{(i)} \times U^{(i)} \times (\{\theta_n\} \cup \theta) \times (\{w_n\} \cup w)$.
So, $A$ is compact and $g'|_{A}$ is uniformly continuous. 
This implies that for all $\epsilon >0$, there exists $\delta >0 $ 
such that if $\rho'(s_1, s_2) < \delta, \mu'(a_1, a_2) < \delta, \|\theta_1-\theta_2\| < \delta,   
\|w_1-w_2\| < \delta,$ then $|g'(s_1, a_1, \theta_1, w_1) - g'(s_2, a_2, \theta_2, w_2)|< \epsilon$  where $s_1, s_2 \in S^{(i)}, 
a_1, a_2 \in U^{(i)},
 \theta_1, \theta_2 \in (\{\theta_n\} \cup \theta), w_1, w_2 \in(\{w_n\} \cup w)$
and $\rho'$ and $\mu'$ denote the metrics in $S^{(i)}$ and $U^{(i)}$ respectively. 
Now use this same $\delta$ for the $\{g^{(i)}_n(\cdot, \cdot)\}$ to get for all $n$ the following for $\rho'(z_1, z_2) < \delta, 
\mu'(a_1, a_2) < \delta$: 
\begin{align}
|g^{(i)}_n(z_1,a_1) - g^{(i)}_n(z_2,a_2)| = |g'(z_1,a_1, \theta_n,w_n) - g'(z_2,a_2, \theta_n,w_n)| < \epsilon.\nonumber 
\end{align}
Hence $\{g^{(i)}_n(\cdot, \cdot)\}$ is equicontinuous. For large $n$, $\sup_{(z,a) \in S^{(i)} \times U^{(i)}}|g^{(i)}_n(z,a) - g^{(i)}(z,a)| < \epsilon/2$ 
because of uniform convergence of $\{g^{(i)}_n(\cdot, \cdot)\}$, hence $\int_{S^{(i)}\times U^{(i)}}|g^{(i)}_n(z,a) - g^{(i)}(z,a)|\Psi^{(i)}_n(dz,da) < \epsilon/2$. 
Now (for $n$ large),
\begin{align}
\label{limit} 
&|\int_{S^{(i)}\times U^{(i)}}g^{(i)}_n(z,a)\Psi^{(i)}_n(dz,da) - \int_{S^{(i)}\times U^{(i)}}g^{(i)}(z,a)\Psi^{(i)}(dz,da)|\nonumber\\ 
&= |\int_{S^{(i)}\times U^{(i)}}[g^{(i)}_n(z,a) - g^{(i)}(z,a)] \Psi^{(i)}_n(dz,da) + \int_{S^{(i)}\times U^{(i)}}g^{(i)}(z,a)\Psi^{(i)}_n(dz,da) \nonumber\\ 
&-\int_{S^{(i)}\times U^{(i)}}g^{(i)}(z,a)\Psi^{(i)}(dz,da)|\nonumber\\  
&< \epsilon/2 + |\int_{S^{(i)}\times U^{(i)}}g^{(i)}(z,a)\Psi^{(i)}_n(dz,da) - \int_{S^{(i)}\times U^{(i)}}g^{(i)}(z,a)\Psi^{(i)}(dz,da)|\nonumber\\
&< \epsilon.
\end{align}
 The last inequality follows the fact that $\Psi^{(i)}_n \Rightarrow \Psi^{(i)}$. Hence from  (\ref{limit}) we get,
\begin{align} 
\int_{S^{(i)}}f^{(i)}(z) \Psi^{(i)}(dz, U^{(i)}) = \int_{S^{(i)}\times U^{(i)}}\int_{S^{(i)}}f^{(i)}(y)p^{(i)}(dy|z,a, \theta, w)\Psi^{(i)}(dz,da)\nonumber
\end{align} 
proving that the map is upper-semi-continuous.   
\end{proof}

Define $\tilde{g}(\theta, w, \nu) = \int g(\theta,w,z)\nu(dz, U^{(2)})$ for $\nu \in P(S^{(2)}\times U^{(2)})$ 
and $\hat{g}_\theta(w) = \{\tilde{g}(\theta, w, \nu):  \nu \in D^{(2)}(\theta, w)\}.$
\begin{lemma}
\label{march}
$\forall \theta \in \mathbb{R}^d$, $\hat{g}_\theta(\cdot)$ is a Marchaud map. 
\end{lemma}
\begin{proof}
\begin{enumerate}[label=(\roman*)]
  \item Convexity and compactness follow trivially from the same for the map $(\theta,w) \to D^{(2)}(\theta, w)$.
  \item \begin{align}
         &\|\tilde{g}(\theta, w, \nu)\|\nonumber\\
         &= \|\int g(\theta,w,z)\nu(dz,U^{(2)})\|\nonumber\\
         &\leq \int\|g(\theta,w,z)\|\nu(dz,U^{(2)})\nonumber\\
         &\leq \int L^{(2)}(\|w\|+\|g(\theta,0,z)\|)\nu(dz,U^{(2)})\nonumber\\
         &\leq \max (L^{(2)}, L^{(2)}\int \|g(\theta,0,z)\| \nu(dz,U^{(2)}))(1+\|w\|).\nonumber
        \end{align} Clearly, $K(\theta) = \max (L^{(2)}, L^{(2)}\int \|g(\theta,0,z)\| \nu(dz,U^{(2)})) > 0$. The 
above is true for all $\tilde{g}(\theta,w,\nu) \in \hat{g}_\theta(w), \nu \in D^{(2)}(\theta, w)$.
   \item Let $w_n \to w, \tilde{g}(\theta, w_n, \nu_n) \to m, \nu_n \in D^{(2)}(\theta, w_n)$. 
Now, $\{\nu_n\}$ is tight, hence has a convergent
sub-sequence $\{\nu_{n_k}\}$ with $\nu$ being the limit. Then using the arguments similar to the proof of Lemma \ref{upsem} 
one can show that  $m=\tilde{g}(\theta, w, \nu)$
whereas $\nu \in D^{(2)}(\theta, w)$ follows
directly from the upper-semi-continuity of the map $w \to D^{(2)}(\theta,w)$ for all $\theta$. 
\end{enumerate}
\end{proof}

\subsection{Other assumptions needed for two time-scale convergence analysis}
\label{assump}

We now list the other assumptions required for two time-scale convergence analysis:  
\begin{enumerate}[label=\textbf{(A\arabic*)}]
\setcounter{enumi}{5}
\item for all $\theta \in \mathbb{R}^d$, the differential inclusion
\begin{equation}
\label{fast}
\dot{w}(t) \in \hat{g}(\theta,w(t)) 
\end{equation}
has a singleton global attractor $\lambda(\theta)$ 
where $\lambda :  \mathbb{R}^d \to \mathbb{R}^k$ is a Lipschitz map with constant $K$.
Additionally,  there exists a continuous function $V: \mathbb{R}^{d+k} \to [0,\infty)$ satisfying the hypothesis of 
Corollary 3.28 of \cite{benaim} with $\Lambda = \{(\theta,\lambda(\theta)):\theta \in \mathbb{R}^d\}$.
This is the most important
assumption as it links the fast and slow iterates.

\item Stability of the iterates: $\sup_n(\|\theta_n\| + \|w_n\|) < \infty$ a.s. 
\end{enumerate}

Let $\bar{\theta}(\cdot), t\geq 0$ be the continuous, piecewise linear 
trajectory defined by $\bar{\theta}(t(n))=\theta_n, n\geq 0$, with 
linear interpolation on each interval $[t(n), t(n+1))$, i.e.,  
\begin{equation}
\bar{\theta}(t) = \theta_n + (\theta_{n+1} - \theta_n)\frac{t-t(n)}{t(n+1)-t(n)}, t \in [t(n), t(n+1)).\nonumber 
\end{equation}
The following theorem is our main result:
\begin{theorem}[Slower timescale result]Under assumptions \textbf{(A1)-(A7)}, 
\label{thm}
\begin{equation}
(\theta_n, w_n) \to \cup_{\theta^* \in A_0}(\theta^*, \lambda(\theta^*)) \mbox{a.s. as $n \to \infty$.},\nonumber 
\end{equation}
\end{theorem}
where $A_0 = \cap_{t\geq 0}\overline{\{\bar{\theta}(s): s\geq t\}}$ 
is almost everywhere an internally chain transitive set 
of the differential inclusion
\begin{equation}
\label{slower_di}
\dot{\theta}(t) \in \hat{h}(\theta(t)), 
\end{equation}
where $\hat{h}(\theta)=\{\tilde{h}(\theta,\lambda(\theta),\nu) :  \nu \in D^{(1)}(\theta, \lambda(\theta))\}$. 
We call (\ref{fast}) and (\ref{slower_di}) as the faster and slower d.i. 
to correspond with faster and slower recursions, respectively.
\begin{corollary}
\label{main_col}
Under the additional assumption that the inclusion 
\begin{equation}
\dot{\theta}(t)\in \hat{h}(\theta(t))), \nonumber 
\end{equation}
has a global attractor set $A_1$,  
\begin{equation}
(\theta_n, w_n) \to \cup_{\theta^* \in A_1}(\theta^*, \lambda(\theta^*)) \mbox{a.s. as $n \to \infty$.}\nonumber 
\end{equation} 
\end{corollary}
\begin{remark}
To prove the required convergence one has to prove that the internally chain 
transitive set mentioned there is a subset of 
the set $\{(\theta, \lambda(\theta)) :  \theta \in \mathbb{R}^d\}$ (this 
does not directly follow from the asymptotic stability of the 
faster d.i (\ref{fast}) for every $\theta$). For this reason we need 
$V$ which is Lyapunov function for the differential inclusion 
\begin{equation}
\dot{w}(t) \in \hat{g}(\theta(t),w(t)), \dot{\theta}(t) = 0.\nonumber  
\end{equation} with $\Lambda =\{(\theta, \lambda(\theta)), \theta \in \mathbb{R}^d\}$ and then apply Proposition 
3.27 or Corollary 3.28 of \cite{benaim}. 
\end{remark}
\begin{remark}
In case where the set $D^{(2)}(\theta,w)$ is singleton, we can relax \textbf{(A6)} to local attractors also. The relaxed 
assumption will be
\begin{enumerate}[label=\textbf{(A\arabic*)'}]
\setcounter{enumi}{5} 
\item The function $\hat{g}(\theta, w) = \int g(\theta, w, z)\Gamma^{(2)}_{\theta,w}(dz)$ is Lipschitz 
continuous where $\Gamma^{(2)}_{\theta,w}$ is the only element of $D^{(2)}(\theta,w)$. 
Further, for all $\theta \in \mathbb{R}^d$, the o.d.e
\begin{equation}
\label{cpledode}
\dot{w}(t) = \hat{g}(\theta, w(t))  
\end{equation}
has an asymptotically stable equilibrium $\lambda(\theta)$ with domain of attraction $G_\theta$ 
where $\lambda :  \mathbb{R}^d \to \mathbb{R}^k$ is a Lipschitz map with constant $K$.
Also, assume that $\bigcap_{\theta} G_\theta$ is non-empty.
Moreover, the function $V': G \to 
[0,\infty)$ defined by $V'(\theta,w) = V_\theta(w)$ is continuously differentiable where $V_\theta(\cdot)$ is the Lyapunov function 
(for definition see \cite[Chapter 11.2.3]{borkar1})
for the o.d.e. (\ref{cpledode}) with $\lambda(\theta)$ 
as its attractor, 
and $G=\bigcup_{\theta \in \mathbb{R}^d} \{\{\theta\} \times  G_\theta$\}. 
This extra condition is needed
so that the set graph($\lambda$):=$\{(\theta,\lambda(\theta)): \theta \in \mathbb{R}^d\}$ becomes an asymptotically stable set of the coupled o.d.e 
\begin{equation}
\dot{w}(t) = \hat{g}(\theta(t),w(t)), \dot{\theta}(t) = 0.\nonumber  
\end{equation} 
\end{enumerate}
Note that \textbf{(A6)'} allows multiple attractors (at least one of them 
have to be a point, others can be sets) for the faster o.d.e 
for every $\theta$.
 
Then the statement of Theorem \ref{thm} will be modified as in the following: 
\begin{theorem}[Slower timescale result when $\lambda(\theta)$ is a local attractor]Under assumptions \textbf{(A1)-(A5), (A6)'} and  \textbf{(A7)},  
on the event ``$\{w_n\}$ belongs to a compact subset $B$ (depending 
on the sample point) of $\bigcap_{\theta \in \mathbb R^d} G_\theta$ \textbf{eventually}", 
\label{thm_local}
\begin{equation}
(\theta_n, w_n) \to \cup_{\theta^* \in A_0}(\theta^*, \lambda(\theta^*)) \mbox{a.s. as $n \to \infty$.}\nonumber 
\end{equation}
\end{theorem}
The requirement  on $\{w_n\}$ is much stronger than the usual local attractor 
statement for Kushner-Clarke lemma \cite[Section II.C]{metivier} which requires 
the iterates to enter a compact set in the domain for attraction of the local attractor 
\textit{infinitely often} only. The reason for imposing this strong assumption is 
that graph($\lambda$) is not a subset of 
any compact set in $\mathbb{R}^{d+k}$, and hence the usual tracking 
lemma kind of arguments do not go through directly. One has to 
relate the limit set of the coupled iterate $(\theta_n, w_n)$ to graph($\lambda$) (See the proof of Lemma \ref{fast_res2}).

\end{remark}

For the special case when the d.i is the o.d.e (for the case described in \textbf{(A6)'}), one can construct the 
Lyapunov function $V'$ for the coupled o.d.e from the Lyapunov function $V_{\theta}$ for the faster o.d.e for every 
$\theta$. The reason is that 
\begin{equation}
\langle \nabla_{\theta,w}V'(\theta,w), \hat{G}(\theta,w)\rangle =  \langle \nabla_w V_\theta(w), \hat{g}(\theta,w) \rangle \nonumber
\end{equation}
where $\hat{G}(\theta,w) = (0, \hat{g}(\theta,w))$.
Such $V_\theta$ exist due to the converse Lyapunov theorem \cite{kraso}. 
This is not known for the general d.i. case as 
we are not aware of  a converse Lyapunov theorem in the case of d.i.
That is why a separate Lyapunov function $V$
for the coupled d.i. is assumed.

We present the    proof of our main results
in the next section.

\section{Main Results}
\label{mres}

We first discuss an extension of the single time-scale controlled Markov noise framework of \cite{borkar}
under our assumptions to prove our main results. Note that the results of \cite{borkar} 
assume that the state space of the controlled Markov process is Polish which 
may impose  additional conditions that are hard to verify. In this section, other 
than proving our two time-scale results, we prove many 
of the results in \cite{borkar} (which were only
stated there) assuming the state space to be compact and thus can be easily verified. 
In \cite{borkar} the map $\theta \to D(\theta)$ was argued to be upper semicontinuous using the fact that (4) thereof is 
preserved under convergence in $P(S \times U)$. We observe (Lemma 4) that this 
may not be enough to prove the same as in this case 
the control $\theta$ is also changing with $n$ unlike the proof of $D(\theta)$ being closed where only the measure changes with $n$. 
As shown there with a Polish space, we can prove 
uniform convergence of $$g_n(\cdot,\cdot) = \int_{S} f(y)p(dy|\cdot,\cdot,\theta_n) \to  \int_{S} f(y)p(dy|\cdot,\cdot,\theta)= g(\cdot,\cdot)$$ only on some compact
subset $S_c \times U$ of $S \times U$  
where $\theta_n \to \theta, f \in C_b(S)$ and $p(\cdot|\cdot,\cdot,\cdot)$ is the continuous transition kernel of the controlled Markov 
process. Here pointwise convergence is obvious. However, we can get uniform continuity only for $g'|_A$ where 
$g'(z, a, \theta) = \int_{S} f(y)p(dy|z,a, \theta)$ and $A=S_c \times U \times (\{\theta_n\} \bigcup \theta)$. For  
similar reasons the proof of upper semicontinuity of $\hat{h}(\cdot)$ (defined near  (13) of \cite{borkar})
as well as the proof of  (12) of \cite{borkar} 
may require conditions that are hard to verify if Polish state space is assumed. 
However, under the compact state space assumptions our proofs are seen to go through
easily.
There are other minor differences from the proofs of \cite{borkar}: 
 \begin{enumerate}
  \item In Lemma 3.1 of \cite{borkar}, we need to use martingale convergence theorem for square integrable martingales to 
conclude a.s. convergence. We correct that in Lemma 7.
  \item As we assume state space to be compact, assumptions such as $(\star)$ (\cite[p~140]{borkar}) and $(\dagger)$ (\cite[p~141]{borkar})
are not required.
 \end{enumerate}
   
We begin by describing the intuition behind the proof techniques in \cite{borkar}.  
\\ \indent
The space $C([0, \infty); \mathbb{R}^d)$ of continuous
functions from $[0,\infty)$ to $\mathbb{R}^d$ is topologized with the 
coarsest topology such that the map that takes any $f \in C([0, \infty); \mathbb{R}^d)$
to its restriction to $[0,T]$ when viewed as an element of the space $C([0, T]; \mathbb{R}^d)$, 
is continuous for all $T>0$. In other words, 
$f_n \to f$ in this space iff $f_n|_{[0,T]} \to f|_{[0,T]}$.
The other notations used below are the same as those in \cite{borkar,borkar1}. We present a few for easy reference. 
\\ \indent
Consider the single time-scale stochastic approximation recursion with controlled Markov noise:
\begin{equation}
\label{cont_mar}
x_{n+1} = x_n + a(n)\left[h(x_n, Y_n) + M_{n+1}\right]. 
\end{equation}
Define time instants $t(0)=0, t(n)=\sum_{m=0}^{n-1} a(m), n\geq 1$. 
Let $\bar{x}(t), t\geq 0$ be the continuous, piecewise linear 
trajectory defined by $\bar{x}(t(n))=x_n, n\geq 0$, with 
linear interpolation on each interval $[t(n), t(n+1))$, i.e.,  
\begin{equation}
\bar{x}(t) = x_n + (x_{n+1} - x_n)\frac{t-t(n)}{t(n+1)-t(n)}, t \in [t(n), t(n+1)).\nonumber 
\end{equation}
Now, define $\tilde{h}(x,\nu)=\int h(x,z)\nu(dz,U)$ for $\nu \in P(S \times U)$. 
Let $\mu(t), t\geq 0$ be the random process defined by 
$\mu(t)=\delta_{Y_n,Z_n}$ for $t \in [t(n), t(n+1)), n\geq 0$, where $\delta_{y,a}$ 
is the Dirac measure corresponding to $y$. Consider 
the non-autonomous o.d.e. 
\begin{equation}
\label{auto}
\dot{x}(t) = \tilde{h}(x(t), \mu(t)). 
\end{equation}
Let $x^s(t), t\geq s$, denote the solution to (\ref{auto}) with $x^s(s)=\bar{x}(s)$, for $s\geq0$. 
Note that $x^s(t), t\in [s, s+T]$ and $x^s(t), t\geq s$ can be viewed as elements of $C([0, T]; \mathbb{R}^d)$ and 
$C([0, \infty); \mathbb{R}^d)$ respectively. With this abuse of notation, it is easy to see that 
$\{x^s(\cdot)|_{[s, s+T]}, s\geq 0\}$ is a 
pointwise bounded and equicontinuous family of functions in $C([0, T]; \mathbb{R}^d)~\forall T >0$.
By Arzela-Ascoli theorem, it is relatively compact. From Lemma~2.2 of \cite{borkar} 
one can see that $\forall s(n)\uparrow \infty, \{\bar{x}(s(n)+.)|_{[s(n), s(n)+T]}, n\geq 1\}$ 
has a limit point in $C([0, T]; \mathbb{R}^d)~\forall T >0$.
With the above topology for $C([0, \infty); \mathbb{R}^d)$, 
$\{x^s(\cdot), s\geq 0\}$ 
is also relatively compact in $C([0, \infty); \mathbb{R}^d)$ and 
$\forall s(n)\uparrow \infty, \{\bar{x}(s(n)+.), n\geq 1\}$ has a limit point in $C([0, \infty); \mathbb{R}^d)$.  
\\ \indent
One can write from (\ref{cont_mar})
the following:
\begin{equation}
\bar{x}(u(n)+t) = \bar{x}(u(n)) + \int_{0}^{t}h(\bar{x}(u(n)+\tau), \nu(u(n)+\tau))d\tau + W^n(t),\nonumber 
\end{equation}
where $u(n)\uparrow \infty, \bar{x}(u(n)+.) \to \tilde{x}(\cdot), \nu(t) = (Y_n,Z_n)$ for $t \in [t(n), t(n+1)), n\geq 0$ 
and $W^n(t) = W(t+u(n)) - W(u(n)), W(t) = W_n + (W_{n+1} - W_n)\frac{t-t(n)}{t(n+1)- t(n)},  
 W_n=\sum_{k=0}^{n-1}a(k)M_{k+1}, 
n\geq 0$. 
From here one cannot directly take limit on both sides as limit points of $\nu(s+.)$ as $s \to \infty$
is not meaningful. Now, $h(x,y)=\int h(x,z)\delta_{y,a}(dz \times U)$. 
Hence by defining 
$\tilde{h}(x,\rho)=\int h(x,z)\rho(dz)$ and $\mu(t) = \delta_{\nu(t)}$ one can write the above as
\begin{equation}
\label{mu} 
\bar{x}(u(n)+t) = \bar{x}(u(n)) + \int_{0}^{t}\tilde{h}(\bar{x}(u(n)+\tau), \mu(u(n)+\tau))d\tau + W^n(t).
\end{equation}The advantage is that the space $\mathcal{U}$ of measurable functions 
from $[0, \infty)$ to $\mathcal{P}(S \times U)$ is compact metrizable, so subsequential limits exist. Note that $\mu(\cdot)$ is not a member
of $\mathcal{U}$, rather we need to fix a sample point, i.e., $\mu(.,\omega) \in \mathcal{U}$. 
For ease of understanding, we abuse the terminology and talk about the limit points $\tilde{\mu}(\cdot)$ of $\mu(s+.)$. 
\\ \indent 
From (\ref{mu}) 
one can infer that the limit $\tilde{x}(\cdot)$ of $\bar{x}(u(n)+.)$ satisfies the o.d.e. 
$\dot{x}(t) = \tilde{h}(x(t), \mu(t))$ with $\mu(\cdot)$ replaced by $\tilde{\mu}(\cdot)$. 
Here each $\tilde{\mu}(t), t \in \mathbb{R}$ in $\tilde{\mu}(\cdot)$ is generated through different limiting
processes each one associated with the compact metrizable space $U_t$ = space of measurable 
functions from $[0,t]$ to $\mathcal{P}(S \times U)$. This will be problematic if we want to further explore the process $\tilde{\mu}(\cdot)$
and convert the non-autonomous o.d.e. into an autonomous one. 
\\ \indent 
Hence the main result is proved using one auxiliary lemma \cite[Lemma~2.3]{borkar} other
than the tracking lemma (Lemma~2.2 of \cite{borkar}). Let $u(n(k)) \uparrow \infty$ be such that
$\bar{x}(u(n(k))+.) \to \tilde{x}(\cdot)$ and 
$\mu(u(n(k))+.) \to \tilde{\mu}(\cdot)$, then using Lemma~2.2 of \cite{borkar} 
one can show that $x^{u(n(k))}(\cdot) \to \tilde{x}(\cdot)$. 
Then the auxiliary lemma shows that the 
o.d.e. trajectory $x^{u(n(k))}(\cdot)$ associated with $\mu(u(n(k))+.)$ 
tracks (in the limit) the o.d.e. trajectory associated with $\tilde{\mu}(\cdot)$. Hence Lemma~2.3 of \cite{borkar} links 
the two limiting processes $\tilde{x}(\cdot)$ and $\tilde{\mu}(\cdot)$ 
in some sense. Note that Lemma~2.3 of \cite{borkar} involves only 
the o.d.e. trajectories, not the interpolated trajectory of the algorithm. 

Consider the iteration 
\begin{equation}
\label{eps2}
\theta_{n+1} = \theta_n + a(n)\left[h(\theta_n, Y_n) + \epsilon_n + M_{n+1}\right],
\end{equation}
where $\epsilon_n \to 0$ and the rest of the notations are same as  \cite{borkar}. 
Specifically, $\{Y_n\}$ is the controlled Markov process driven by $\{\theta_n\}$
and $M_{n+1}, n\geq 0$ is a martingale difference sequence. 
Let $\bar{\theta}(t), t\geq 0$ be the continuous, piecewise linear 
trajectory of (\ref{eps2}) defined by $\bar{\theta}(t(n))=\theta_n, n\geq 0$, with 
linear interpolation on each interval $[t(n), t(n+1))$. Also, 
let $\theta^s(t), t\geq s$, denote the solution to (\ref{auto}) with $\theta^s(s)=\bar{\theta}(s)$, for $s\geq0$. 
\\ \indent
The convergence analysis of (\ref{eps2}) requires some changes  
in Lemma~2.2 and 3.1 of  \cite{borkar}. The modified versions of them are precisely 
the following two lemmas.   
\begin{lemma}
\label{track_fast}
For any $T >0$, $\sup_{t\in [s,s+T]}\|\bar{\theta}(t) - \theta^s(t)\| \to 0,$ a.s. as $s\to \infty$. 
\end{lemma}
\begin{proof}
The proof follows from the Lemma 2.2 and the remark 3 thereof (p. 144)  of \cite{borkar}.
\end{proof}

Now, $\mu$ can be viewed as a random variable taking values in $\mathcal{U}$ = the space of measurable functions from $[0,\infty)$ to
$\mathcal{P}(S \times U)$. This space is topologized with the coarsest topology such that the map
\begin{equation}
\nu(\cdot) \in \mathcal{U} \to \int_{0}^{T} g(t) \int fd\nu(t)dt \in \mathbb{R}\nonumber 
\end{equation}
is continuous for all $f \in C(S), T>0, g \in L_2[0,T]$. Note that $\mathcal{U}$ is compact metrizable.
\begin{lemma}
\label{eps}
Almost surely every limit point of $(\mu(s+.), \bar{\theta}(s+.))$ as $s\to \infty$ is 
of the form $(\tilde{\mu}(\cdot), \tilde{\theta}(\cdot))$
where $\tilde{\mu}(\cdot)$ satisfies $\tilde{\mu}(t) \in D(\tilde{\theta}(t))$ a.e. $t$.  
\end{lemma}
\begin{proof}
Suppose that $u(n)\uparrow \infty$, $\mu(u(n)+.) \to \tilde{\mu}(\cdot)$ and $\bar{\theta}(u(n)+.) \to \tilde{\theta}(\cdot)$.
Let $\{f_i\}$ be countable dense in the unit ball of $C(S)$, hence a separating 
class, i.e., $\forall i, \int f_i d\mu = \int f_i d\nu$ 
implies $\mu=\nu$. For 
each $i$,
\begin{equation}
\zeta^i_n = \sum_{m=1}^{n-1}a(m)(f_i(Y_{m+1}) - \int f_i(y)p(dy|Y_m,Z_m, \theta_m)), n \geq 1, \nonumber 
\end{equation}
is a zero-mean martingale with $\mathcal{F}_n = \sigma(\theta_m, Y_m, Z_m, m\leq n)$. 
Moreover, it is a square integrable martingale due to the 
fact that $f_i$'s are bounded and each $\zeta^i_n$ is a finite sum. Its quadratic 
variation process 
\begin{equation}
A_{n}=\sum_{m=0}^{n-1}a(m)^2E[(f_i(Y_{m+1}) - 
\int f_i(y)p(dy|Y_m,Z_m, \theta_m))^2|\mathcal{F}_m] + E[(\zeta^i_0)^2]\nonumber 
\end{equation}
is almost surely convergent. 
By the martingale convergence theorem, 
 $\zeta^i_n, n\geq 0$ converges a.s. $\forall i$. As before let $\tau(n,t)=\min\{m \geq n:  t(m) \geq t(n)+t\}$ for $t\geq0, n\geq0$. 
Then as $n\to\infty$,
\begin{equation}
\sum_{m=n}^{\tau(n,t)} a(m)(f_i(Y_{m+1})-\int f_i(y)p(dy|Y_m,Z_m,\theta_m))\to 0,\mbox{ a.s.}\nonumber 
\end{equation}
for $t >0$. 
By our choice of $\{f_i\}$ and the fact that 
$\{a(n)\}$ is an eventually non-increasing sequence (the latter property is used only here and in Lemma \ref{slowmu}), we have
\begin{equation}
\sum_{m=n}^{\tau(n,t)}(a(m) - a(m+1))f_i(Y_{m+1}) \to 0,\mbox{ a.s.}\nonumber 
\end{equation}
From the foregoing, 
\begin{equation}
\sum_{m=n}^{\tau(n,t)} (a(m+1)f_i(Y_{m+1})-a(m)\int f_i(y)p(dy|Y_m,Z_m,\theta_m))\to 0,\mbox{ a.s.}\nonumber 
\end{equation}
$\forall t >0$,
which implies
\begin{equation}
\sum_{m=n}^{\tau(n,t)} a(m)(f_i(Y_{m})-\int f_i(y)p(dy|Y_m,Z_m,\theta_m))\to 0,\mbox{ a.s.}\nonumber 
\end{equation}
$\forall t >0$ due to the fact that 
$a(n) \to 0$ and $f_i(\cdot)$ are bounded. 
This implies
\begin{equation}
\int_{t(n)}^{t(n)+t}(\int(f_i(z) - \int f_i(y)p(dy|z,a,\hat{\theta}(s)))\mu(s,dzda))ds \to 0,\mbox{ a.s.}\nonumber 
\end{equation}
and that in turn implies
\begin{equation}
\int_{u(n)}^{u(n)+t}(\int(f_i(z) - \int f_i(y)p(dy|z,a,\hat{\theta}(s)))\mu(s,dzda))ds \to 0,\mbox{ a.s.}\nonumber 
\end{equation}
(this is true because $a(n)\to 0$ and 
$f_i(\cdot)$ is bounded)
where $\hat{\theta}(s) = \theta_n$ when $s \in [t(n), t(n+1))$ for $n\geq 0$. Now, one can claim from the above that
\begin{equation}
\int_{u(n)}^{u(n)+t}(\int(f_i(z) - \int f_i(y)p(dy|z,a,\bar{\theta}(s)))\mu(s,dzda))ds \to 0,\mbox{ a.s.}\nonumber 
\end{equation}
This is due to the fact that the map $S \times U \times \mathbb{R}^{d} \ni (z,a,\theta) \to \int f(y)p(dy|z,a,\theta)$ is continuous and hence 
uniformly continuous on the compact set $A = S \times U \times M$ 
where $M$ is the compact set s.t. $\theta_n\in M~\forall n$.    
Here we also use the fact that $\|\bar{\theta}(s) - \theta_m\|=\|h(\theta_m, Y_m) + \epsilon_m + M_{m+1}\|(s-s_m) \to 0, s\in [t_m, t_{m+1})$ 
as the first two terms inside the norm in the R.H.S are bounded.  
The above convergence is equivalent to 
\begin{equation}
\int_{0}^{t}(\int(f_i(z) - \int f_i(y)p(dy|z,a,\bar{\theta}(s+u(n)))\mu(s+u(n),dzda))ds \to 0,\mbox{ a.s.}\nonumber 
\end{equation} 
Fix a sample point in the probability one set on which the convergence 
above holds for all $i$. Then the convergence 
above leads to 
\begin{equation}
\label{conjun1}
\int_{0}^{t}(\int f_i(z) - \int f_i(y)p(dy|z,a, \tilde{\theta}(s)))\tilde{\mu}(s, dzda)ds =0~\forall i. 
\end{equation}
Here we use one part of the proof from Lemma~2.3 of \cite{borkar} that if $\mu^n(\cdot) \to \mu^{\infty}(\cdot) \in \mathcal{U}$ then 
for any $t>0$,
\begin{equation}
\int_{0}^{t} \int \tilde{f}(s,z,a)\mu^n(s,dzda)ds - \int_{0}^{t} \int \tilde{f}(s,z,a)\mu^{\infty}(s,dzda)ds \to 0,\nonumber 
\end{equation}
for all $\tilde{f} \in 
C([0,t] \times S \times A)$ and the fact that $\tilde{f}_n(s,z,a) = \int f_i(y)p(dy|z,a,\bar{\theta}(s+u(n)))$ 
converges uniformly to $\tilde{f}(s,z,a) = \int f_i(y)p(dy|z,a,\tilde{\theta}(s))$. 
To prove the latter, define $g:C([0,t]) \times [0,t] \times S \times A \to \mathbb{R}$ by $g(\theta(\cdot), s,z,a) 
 = \int f_i(y)p(dy|z,a, \theta(s)))$. To see that $g$ is continuous we need to check that if $\theta_n(\cdot) \to \theta(\cdot)$ 
uniformly and $s(n) \to s$, then $\theta_n(s(n)) \to \theta(s)$. This is because $\|\theta_n(s(n)) - \theta(s)\| = \|\theta_n(s(n)) - \theta(s(n)) + \theta(s(n)) - \theta(s)\| \leq \|\theta_n(s(n)) - \theta(s(n))\| + 
\|\theta(s(n)) - \theta(s)\|$. The first and second terms go to zero 
due to the uniform convergence of $\theta_n(\cdot), n\geq 0$ and continuity of $\theta(\cdot)$ respectively.  
 Let $A = \{\bar{\theta}(u(n)+.)|_{[u(n),u(n)+t]}, n\geq 1\} \cup \tilde{\theta}(\cdot)|_{[0,t]}$. 
$A$ is compact as it is the
union of a sequence of functions and their limit. 
So, $g|_{(A \times [0,t]\times S \times U)}$ is uniformly 
continuous. Then using the same arguments as in Lemma~\ref{upsem} 
we can show equicontinuity of $\{\tilde{f}_n(.,.)\}$, that results in
uniform convergence and thereby (\ref{conjun1}).  
An application of Lebesgue's theorem in conjunction 
with (\ref{conjun1}) shows that
\begin{equation}
\int (f_i(z) - \int f_i(y)p(dy|z,a,\tilde{\theta}(t)))\tilde{\mu}(t, dzda) = 0~\forall i\nonumber 
\end{equation}
for a.e. $t$. 
By our choice of $\{f_i\}$, this leads to 
\begin{equation}
\tilde{\mu}(t, dy \times U) = \int p(dy|z,a,\tilde{\theta}(t))\tilde{\mu}(t, dzda)\nonumber 
\end{equation}
a.e. $t$. Therefore the conclusion follows by disintegrating such measure as the product of marginal on $S$ and
the regular conditional law on $U$ (\cite[p~140]{borkar}). 
\end{proof}

\begin{remark}
Note that the above invariant distribution does not come ``naturally''; rather it arises from the assumption made to match the natural 
timescale intuition for the controlled Markov noise component, i.e., the slower iterate should see the average effect of the Markov component.
\end{remark}

The proof of the following lemma, in this case, will be unchanged from its original version, 
so we just mention it for completeness and refer the reader to Lemma 2.3 of \cite{borkar} for its proof.
\begin{lemma}
\label{eps31}
Let $\mu^n(\cdot) \to \mu^{\infty}(\cdot) \in \mathcal{U}$. Let $\theta^n(\cdot), n=1, 2, \dots, \infty$ denote 
solutions to (\ref{auto}) corresponding to the case where $\mu(\cdot)$ is replaced by $\mu^n(\cdot)$, for $n=1,2,\dots \infty$. 
Suppose $\theta^n(0) \to \theta^{\infty}(0)$. Then 
\begin{equation}
\lim_{n \to \infty} \sup_{t\in [0, T]}\|\theta^n(t) - \theta^{\infty}(t)\| = 0 \nonumber 
\end{equation}
for every $T >0$. 
\end{lemma}
\begin{lemma}
Almost surely, $\{\theta_n\}$ converges to an internally 
chain transitive invariant set of the differential inclusion
\begin{equation}
\label{inc}
\dot{\theta}(t) \in \hat{h}(\theta(t)), 
\end{equation}
where $\hat{h}(\theta)=\{\tilde{h}(\theta,\nu) :  \nu \in D(\theta)\}$. 
\end{lemma}
\begin{proof}
Lemma~\ref{eps31} shows that every limit point $(\tilde{\mu}(\cdot), \tilde{\theta}(\cdot))$ 
of $(\mu(s+.),\bar{\theta}(s+.))$ as $s\to \infty$ is such that
 $\tilde{\theta}(\cdot)$ satisfies (\ref{auto}) with $\mu(\cdot) = \tilde{\mu}(\cdot)$. 
Hence, $\tilde{\theta}(\cdot)$ is absolutely continuous. Moreover, 
using Lemma~\ref{eps}, one can see that it satisfies (\ref{inc}) a.e. $t$, hence is a solution to the differential inclusion (\ref{inc}).  
Hence the proof follows.
\end{proof}
  
\begin{lemma}[Faster timescale result]
 $(\theta_n, w_n) \to \{(\theta, \lambda(\theta)) :  \theta \in \mathbb{R}^d\}$ a.s.
\end{lemma}

\begin{proof}
We first rewrite (\ref{eqn1}) as
\begin{equation}
\theta_{n+1} = \theta_n + b(n)\left[\epsilon_n + M^{(3)}_{n+1}\right],\nonumber 
\end{equation}
where $\epsilon_n = \frac{a(n)}{b(n)}h(\theta_n, w_n, Z^{(1)}_n)\to 0$ as $n\to \infty$ a.s. and $M^{(3)}_{n+1} = \frac{a(n)}{b(n)} M^{(1)}_{n+1}$ for $n\geq 0$. 
Let $\alpha_n=(\theta_n, w_n), \alpha=(\theta,w) \in \mathbb{R}^{d+k}, G(\alpha,z)=(0, g(\alpha,z)), \epsilon'_n=(\epsilon_n, 0), 
M^{(4)}_{n+1}= (M^{(3)}_{n+1}, M^{(2)}_{n+1})$. Then one can write 
 (\ref{eqn1}) and (\ref{eqn2}) in the framework of (\ref{eps2}) as
\begin{equation}
\label{stability}
\alpha_{n+1} = \alpha_n + b(n)\left[G(\alpha_n,Z^{(2)}_n) + \epsilon'_n +  M^{(4)}_{n+1}\right],
\end{equation}
with $\epsilon'_n \to 0$ as $n \to \infty$.  
$\alpha_n, n\geq 0$ converges almost surely to an internally chain transitive set of the differential inclusion
\begin{equation}
\dot{\alpha}(t) \in \hat{G}(\alpha(t)),\nonumber 
\end{equation}
where $\hat{G}(\alpha) = \{\tilde{G}(\alpha, \nu) :  \nu \in D^{(2)}(\theta,w)\}$. 
In other words, 
$(\theta_n, w_n), n\geq 0$ converges to an internally chain transitive set of the differential inclusion
\begin{equation}
\label{coupled_di}
\dot{w}(t) \in \hat{g}_{\theta(t)}(w(t)), \dot{\theta}(t) = 0.\nonumber 
\end{equation}The rest follows from the second part of \textbf{(A6)}.
\end{proof}

\begin{remark} Under the conditions mentioned in Remark 4 the above 
 faster timescale result should be modified as follows:

\begin{lemma}[Faster timescale result when $\lambda(\theta)$ is a local attractor]
\label{fast_res2}
Under assumptions \textbf{(A1) - (A5), (A6)'} and  \textbf{(A7)},  
on the event ``$\{w_n\}$ belongs to a compact subset $B$ (depending 
on the sample point) of $\bigcap_{\theta \in \mathbb R^d} G_\theta$ \textit{eventually}'',
\begin{equation}
(\theta_n, w_n) \to \{(\theta, \lambda(\theta)) :  \theta \in \mathbb{R}^d\} \mbox{~~a.s.} \nonumber 
\end{equation}
\end{lemma}
\begin{proof}
Fix a sample point $\omega$.
The proof follows from these observations: 
\begin{enumerate}
 \item continuity of flow for the coupled o.d.e around the initial point,
 \item $\sup_n \|\theta_n\| = M_1 < \infty$,
 \item the fact that the 
set graph($\lambda$) is Lyapunov stable ($V'(\cdot)$ as mentioned in 
\textbf{(A6)'} will be a Lyapunov function for this set), and
  \item the fact that $\bigcap_{t\geq 0} \overline{\bar{\alpha}(s): s \geq t}$ is an internally chain transitive set of 
the coupled o.d.e 
\begin{equation}
\label{copuled_ode}
\dot{w}(t) = \hat{g}(\theta(t),w(t)), \dot{\theta}(t) = 0,  
\end{equation}
where $\bar{\alpha}(\cdot)$ is the interpolated trajectory of the coupled iterate $\{\alpha_n\}$.
\end{enumerate}
As $\{\theta: \|\theta\| \leq M_1\} \times B \subset \bigcup_{\theta \in \mathbb{R}^d} \{\{\theta\} \times G_\theta \}$,  
the first three observations show that for all $\epsilon>0$, there exists a $T_\epsilon >0$ 
such that any o.d.e trajectory for (\ref{copuled_ode}) with starting point 
on the compact set $\{\theta: \|\theta\| \leq M_1\} \times B$ reaches the 
$\epsilon$-neighbourhood of graph($\lambda$) after time $T_\epsilon$.
Further, 
\begin{equation}
\bigcap_{t\geq 0} \overline{\bar{\alpha}(s): s \geq t} \subset  \{\theta: \|\theta\| \leq M_1\} \times B. \nonumber 
\end{equation}
Then one can use 
the last observation by choosing $T > T_{\epsilon}$ to 
show the required convergence to the set  graph($\lambda$).
\end{proof}
\end{remark}

In other words, $\|w_n - \lambda(\theta_n)\| \to 0$ a.s., i.e, $\{w_n\}$ asymptotically tracks $\{\lambda(\theta_n)\}$ a.s.
\\ \indent 

\begin{remark}
One interesting question in this context is to analyze whether one can extend the single timescale local attractor 
convergence statements to the two time-scale setting under some \textit{verifiable conditions}. More specifically, 
if there is a
global attractor $A_1$ for
\begin{equation}
\dot{\theta}(t) \in \hat{h}(\theta(t)), \nonumber
\end{equation}
then can one provide verifiable conditions to show 
\begin{equation}
P [(\theta_n , w_n ) \to \cup_{\theta \in A_1} (\theta , \lambda(\theta))] > 0. \nonumber
\end{equation}
Here $\lambda(\theta)$ is a local attractor as mentioned in \textbf{(A6)'}.

There are two ways in which this could possibly be tried: 

\begin{enumerate}
 \item Use Theorem \ref{thm_local} where we show that 
 on the event $\{w_n\}$ belongs to a compact subset $B$ (depending 
on the sample point) of $\bigcap_{\theta \in \mathbb R^d} G_\theta$ ``eventually'',
\begin{equation}
(\theta_n, w_n) \to \cup_{\theta^* \in A_1}(\theta^*, \lambda(\theta^*)) \mbox{a.s. as $n \to \infty$,}\nonumber 
\end{equation}which  is an extension of Kushner-Clarke Lemma to the two time-scale case.  
Therefore the task would be to impose verifiable assumptions so that
$P$($\{w_n\}$ belongs to a compact subset $B$ (depending 
on the sample point) of $\bigcap_{\theta \in \mathbb R^d} G_\theta$ ``eventually'') $>$ 0. 
In a stochastic approximation scenario it is not immediately clear how one could possibly 
impose verifiable assumptions so that 
such a probabilistic statement becomes true.
\item The second approach would be to extend the analysis of \cite{benaimode, benaim} to the two time-scale case.
In our opinion this is very hard as this analysis is based on the attractor introduced by Benaim et al. whereas 
the coupled o.d.e (\ref{copuled_ode}) which tracks the coupled iterate 
$(\theta_n,w_n)$ (therefore the interpolated trajectory of the coupled iterate will be 
an asymptotic pseudo-trajectory \cite{benaimode} for (\ref{copuled_ode}))  has no attractor. The reason is that one cannot obtain a fundamental neighbourhood for sets like 
$\cup_{\theta \in A_1} (\theta, \lambda(\theta))$ as the $\theta$ component will remain constant for any trajectory 
of the above coupled o.d.e.
\end{enumerate}

However, in the following we discuss one possible approach (without coupling both the iterates) to 
the closest implementation of the above
for the following recursion:
\begin{align}
\label{fast_a}
w_{n+1} &= w_n + b(n)\left[g(\theta_n, w_n, Z^{(2)}_n) + M^{(2)}_{n+1}\right]
\end{align}
that are driven by a single timescale stochastic approximation process:
\begin{align}
\label{slow_p}
\theta_{n+1} = \theta_n + a(n)\left[h(\theta_n,Z^{(1)}_n) + M^{(1)}_{n+1}\right].
\end{align} Note that there is a unilateral coupling between (\ref{fast_a}) and (\ref{slow_p}) in that 
(\ref{fast_a}) depends on (\ref{slow_p}) but not the other way. 
Here we  assume that the step-sizes $a(n),b(n), n\geq 0$ satisfies only the  
Robbins-Monro conditions. Assume that $Z^{(2)}_n$ is an uncontrolled Markov process.
Now, assume that the o.d.e 
\begin{align}
\dot{\theta}(t)=\hat{h}(\theta(t)) \nonumber
\end{align}
has a globally asymptotically stable attractor $\theta^*$. Consider the following version of (\ref{fast_a})
\begin{align}
w'_{n+1} &= w'_n + b(n)\left[g(\theta^*, w'_n, Z^{(2)}_n) + M^{(2)}_{n+1}\right].\nonumber
\end{align}
Let $\bar{w'}(\cdot)$ be the interpolated trajectory of $w'_n,n\geq 0$. 
Then using Benaim's result \cite[Theorem 7.3]{benaimode} one can see that if $Att(\bar{w'})\cap G_{\theta^*} \neq \phi$ then 
\begin{align}
 P(w_n' \to \lambda(\theta^*)) > 0. \nonumber
\end{align}
Therefore 
\begin{align}
P((\theta_n,w_n') \to (\theta^*, \lambda(\theta^*)) > 0. \nonumber
\end{align}

However, one can easily observe that 
\begin{align}
\|w_{n} - w_n'\| \leq L^{(2)} \sum_{i=1}^{n} \|\theta_i - \theta^*\| b(i) \Pi_{k=i}^{n-1} (1 + b(k+1)). \nonumber 
\end{align} which shows how $w_n'$ is related to the actual iterate of interest $w_n$. However, 
for common step-size sequence such as $b(n)=\frac{1}{n}, \frac{1}{1+n\log n}$ the R.H.S sequence 
diverges to infinity implying that such result should be used in a non-asymptotic sense. Also,
with $b(n) = \frac{1}{n}$, $\|w_{n} - w_n'\| \leq O(n)$ whereas with 
$b(n) = \frac{1}{1+ n\log n}$ it is $O(\log n)$ suggesting to use the latter to get the closest 
statement of the following:
\begin{align}
P((\theta_n,w_n) \to (\theta^*, \lambda(\theta^*)) > 0. \nonumber
\end{align}
\end{remark}
Now, consider the non-autonomous o.d.e.
\begin{equation}
\label{slow_tt}
\dot{\theta}(t) = \tilde{h}(\theta(t),\lambda(\theta(t)),\mu(t)), 
\end{equation}
where $\mu(t) = \delta_{Z^{(1)}_n,A^{(1)}_n}$ when $t \in [t(n), t(n+1))$ for $n\geq 0$ and $\tilde{h}(\theta,w,\nu)=\int h(\theta,w,z) \nu(dz)$. 
Let $\theta^s(t), t\geq s$ denote the solution to (\ref{slow_tt}) with 
$\theta^s(s) = \bar{\theta}(s)$, for $s \geq 0$. Then
\begin{lemma}
\label{track_slow}
For any $T >0, \sup_{t\in [s,s+T]}\|\bar{\theta}(t) - \theta^s(t)\| \to 0,$ a.s.  
\end{lemma}
\begin{proof}
The slower recursion corresponds to
\begin{equation}
\theta_{n+1} = \theta_n + a(n)\left[h(\theta_n, w_n, Z^{(1)}_n) + M^{(1)}_{n+1}\right].\nonumber  
\end{equation}
Let $t(n+m) \in [t(n), t(n) + T]$.  Let $[t] =\max\{t(k) :  t(k) \leq t\}$. Then by construction,
\begin{align}
\bar{\theta}(t(n+m)) &= \bar{\theta}(t(n)) + \sum_{k=0}^{m-1} a(n+k)h(\bar{\theta}(t(n+k)), w_{n+k}, Z^{(1)}_{n+k}) + \delta_{n, n+m} \nonumber\\
                &= \bar{\theta}(t(n)) + \sum_{k=0}^{m-1} a(n+k)h(\bar{\theta}(t(n+k)), \lambda(\bar{\theta}(t(n+k))), Z^{(1)}_{n+k})\nonumber\\
                &+\sum_{k=0}^{m-1} a(n+k)(h(\bar{\theta}(t(n+k)), w_{n+k}, Z^{(1)}_{n+k})- h(\bar{\theta}(t(n+k)), \lambda(\theta_{n+k}), Z^{(1)}_{n+k}))\nonumber\\
                &+ \delta_{n, n+m},\nonumber
\end{align}
where $\delta_{n, n+m}=\zeta_{n+m}- \zeta_{n}$ with  $\zeta_{n} = \sum_{m=0}^{n-1}a(m)M^{(1)}_{m+1}, n \geq 1$.
\begin{align}
 \theta^{t(n)}(t(m+n)) &= \bar{\theta}(t(n)) + \int_{t(n)}^{t(n+m)} \tilde{h}(\theta^{t(n)}(t), \lambda(\theta^{t(n)}(t)),  \mu(t))dt\nonumber\\
                  &= \bar{\theta}(t(n)) + \sum_{k=0}^{m-1} a(n+k)h(\theta^{t(n)}(t(n+k)), \lambda(\theta^{t(n)}(t(n+k))), Z^{(1)}_{n+k})\nonumber\\
                  &+ \int_{t(n)}^{t(n+m)} (h(\theta^{t(n)}(t), \lambda(\theta^{t(n)}(t), \mu(t))) - h(\theta^{t(n)}([t]), \lambda(\theta^{t(n)}([t]), \mu([t]))))dt.\nonumber
\end{align}
Let $t(n) \leq t \leq t(n+m)$. Now, if $0 \leq k \leq (m-1)$ and $t \in (t(n+k), t(n+k+1)],$
\begin{align}
\|\theta^{t(n)}(t)\| &\leq \|\bar{\theta}(t(n)\| + \|\int_{t(n)}^{t} \tilde{h}(\theta^{t(n)}(\tau), \lambda(\theta^{t(n)}(\tau)), \mu(\tau))d\tau\|\nonumber\\
                &\leq \|\theta_n\| + \sum_{l=0}^{k-1} \int_{t(n+l)}^{t(n+l+1)} (\|h(0,0,Z^{(1)}_{n+l})\|+ L^{(1)}(\|\lambda(0)\|+(K+1)\|\theta^{t(n)}(\tau)\|))d\tau\nonumber\\
                & +\int_{t(n+k)}^t(\|h(0,0,Z^{(1)}_{n+k})\|+ L^{(1)}(\|\lambda(0)\|+(K+1)\|\theta^{t(n)}(\tau)\|))d\tau\nonumber\\
                &\leq C_0 + (M+L^{(1)}\|\lambda(0)\|)T + L^{(1)}(K+1)\int_{t(n)}^{t}\|\theta^{t(n)}(\tau)\|d\tau,\nonumber 
\end{align}
where $C_0 = \sup_n \|\theta_n\| < \infty, \sup_{z \in S^{(1)}}\|h(0,0,z)\| = M$. 
By Gronwall's inequality, it follows that 
\begin{equation}
\|\theta^{t(n)}(t)\| \leq (C_0 + (M+L^{(1)}\|\lambda(0)\|)T)e^{L^{(1)}(K+1)T}.\nonumber 
\end{equation}
\begin{align}
\|\theta^{t(n)}(t) - \theta^{t(n)}(t(n+k))\| &\leq \int_{t(n+k)}^t\|h(\theta^{t(n)}(s), \lambda(\theta^{t(n)}(s)), Z^{(1)}_{n+k})\|ds\nonumber\\
                                   &\leq (\|h(0,0,Z^{(1)}_{n+k})\|+ L^{(1)}\|\lambda(0)\|)(t-t(n+k))\nonumber\\
                                   &+L^{(1)}(K+1)\int_{t(n+k)}^t\|\theta^{t(n)}(s)\|ds\nonumber\\
                                   &\leq C_Ta(n+k),\nonumber
\end{align}
where $C_T=(M+ L^{(1)}\|\lambda(0)\|) + L^{(1)}(K+1)(C_0 + (M+L^{(1)}\|\lambda(0)\|)T)e^{L^{(1)}(K+1)T}.$
Thus, 
\begin{align}
&\|\int_{t(n)}^{t(n+m)} (h(\theta^{t(n)}(t), \lambda(\theta^{t(n)}(t)), \mu(t)) - h(\theta^{t(n)}([t]), \lambda(\theta^{t(n)}([t])), \mu([t])))dt\|\nonumber\\
&\leq \sum_{k=0}^{m-1}\int_{t(n+k)}^{t(n+k+1)}\|h(\theta^{t(n)}(t), \lambda(\theta^{t(n)}(t)), Z^{(1)}_{n+k}) - h(\theta^{t(n)}([t]), \lambda(\theta^{t(n)}([t])), Z^{(1)}_{n+k})\|dt\nonumber\\
&\leq L\sum_{k=0}^{m-1}\int_{t(n+k)}^{t(n+k+1)}\|\theta^{t(n)}(t) - \theta^{t(n)}(t(n+k))\|dt\nonumber\\
&\leq C_TL\sum_{k=0}^{m-1}a(n+k)^2\nonumber\\
&\leq C_TL\sum_{k=0}^{\infty}a(n+k)^2 \to 0 \mbox{~as $n \to \infty$},\nonumber
\mbox{where $L=L^{(1)}(K+1).$}
\end{align}
Hence
\begin{align}
\|\bar{\theta}(t(n+m))-\theta^{t(n)}(t(n+m))&\leq L\sum_{k=0}^{m-1}a(n+k)\|\bar{\theta}(t(n+k)) - \theta^{t(n)}(t(n+k))\|\nonumber\\
& + C_TL\sum_{k=0}^{\infty}a(n+k)^2 + \sup_{k\geq 0}\|\delta_{n,n+k}\|\nonumber\\
& + L^{(1)}\sum_{k=0}^{m-1}a(n+k)\|w_{n+k} -\lambda(\theta_{n+k})\|\nonumber\\
&\leq L\sum_{k=0}^{m-1}a(n+k)\|\bar{\theta}(t(n+k)) - \theta^{t(n)}(t(n+k))\|\nonumber\\
& + C_TL\sum_{k=0}^{\infty}a(n+k)^2 + \sup_{k\geq 0}\|\delta_{n,n+k}\|\nonumber\\
& +  L^{(1)}T\sup_{k\geq 0} \|w_{n+k} - \lambda(\theta_{n+k})\|,\nonumber \mbox{ a.s.}
\end{align}
Define
\begin{equation}
K_{T,n} =  C_TL\sum_{k=0}^{\infty}a(n+k)^2 + \sup_{k\geq 0}\|\delta_{n,n+k}\|
 + L^{(1)}T\sup_{k\geq 0} \|w_{n+k} - \lambda(\theta_{n+k})\|. \nonumber 
\end{equation}
Note that $K_{T,n} \to 0$ a.s. 
The remainder of the proof follows in the exact 
same manner as the tracking lemma, see Lemma 1, Chapter 2 of \cite{borkar1}. 
\end{proof}

\begin{lemma}
\label{ode}
Suppose, $\mu^n(\cdot) \to \mu^{\infty}(\cdot) \in U^{(1)}$. Let $\theta^n(\cdot), n=1, 2, \dots, \infty$ 
denote solutions to (\ref{slow_tt}) corresponding to the case where $\mu(\cdot)$ 
is replaced by $\mu^n(\cdot)$, for $n=1,2, \dots, \infty$. 
Suppose $\theta^n(0) \to \theta^{\infty}(0)$. Then 
\begin{equation}
\lim_{n \to \infty} \sup_{t\in [0, T]}\|\theta^n(t) - \theta^{\infty}(t)\| \to 0 \nonumber 
\end{equation}
for every $T >0$. 
\end{lemma} 
\begin{proof}
It is shown in  Lemma~2.3 of \cite{borkar} that
\begin{equation}
\int_{0}^{t}\int \tilde{f}(s,z)\mu^{n}(s,dz)ds - \int_{0}^{t}\int \tilde{f}(s,z)\mu^{\infty}(s,dz)ds \to 0 \nonumber 
\end{equation}
for any $\tilde{f} \in C([0,T]\times S)$. 
Using this, one can see that
\begin{equation}
\|\int_{0}^{t} (\tilde{h}(\theta^{\infty}(s),\lambda(\theta^{\infty}(s)),  \mu^n(s)) - 
\tilde{h}(\theta^{\infty}(s), \lambda(\theta^{\infty}(s)), \mu^{\infty}(s)))ds \| \to 0.\nonumber 
\end{equation}
This follows because $\lambda$ is continuous and $h$ is jointly continuous in its arguments.
As a function of $t$, the integral on the left is equicontinuous and pointwise bounded. By the Arzela-Ascoli theorem, this 
convergence must in fact be uniform for $t$ in a compact set. Now for $t>0$, 
\begin{align}
&\|\theta^n(t)-\theta^{\infty}(t)\|\nonumber\\ 
&\leq \|\theta^n(0) - \theta^{\infty}(0)\| + \int_{0}^{t} \|\tilde{h}(\theta^n(s), \lambda(\theta^n(s)),\mu^n(s)) - \tilde{h}(\theta^{\infty}(s), \lambda(\theta^{\infty}(s)), \mu^{\infty}(s))\|ds\nonumber\\
&\leq \|\theta^n(0) - \theta^{\infty}(0)\| + \int_{0}^{t} (\|\tilde{h}(\theta^n(s), \lambda(\theta^n(s)),\mu^n(s)) - \tilde{h}(\theta^{\infty}(s), \lambda(\theta^{\infty}(s)), \mu^{n}(s))\|)ds\nonumber\\ 
&+ \int_{0}^{t} (\|\tilde{h}(\theta^{\infty}(s),\lambda(\theta^{\infty}(s)), \mu^{n}(s)) - \tilde{h}(\theta^{\infty}(s),\lambda(\theta^{\infty}(s)), \mu^{\infty}(s))\|)ds.\nonumber
\end{align}
Now, using the fact that $\lambda$ is Lipschitz with constant $K$ the remaining 
part of the proof follows in the same manner as Lemma~2.3 of \cite{borkar}. 
\end{proof}

Note that Lemma~\ref{ode} shows that every limit point $(\tilde{\mu}(\cdot), \tilde{\theta}(\cdot))$ 
of $(\mu(s+.),\bar{\theta}(s+.))$ as $s\to \infty$ is such that
 $\tilde{\theta}(\cdot)$ satisfies (\ref{slow_tt}) with $\mu(\cdot) = \tilde{\mu}(\cdot)$. 
\begin{lemma}
\label{slowmu}
Almost surely every limit point of $(\mu(s+.),\bar{\theta}(s+.))$ as 
$s \to \infty$ is of the form $(\tilde{\mu}(\cdot), \tilde{\theta}(\cdot))$, where $\tilde{\mu}(\cdot)$ 
satisfies $\tilde{\mu}(t) \in D^{(1)}(\tilde{\theta}(t), \lambda(\tilde{\theta}(t)))$. 
\end{lemma}
\begin{proof}
Suppose that $u(n)\uparrow \infty$, $\mu(u(n)+.) \to \tilde{\mu}(\cdot)$ and $\bar{\theta}(u(n)+.) \to \tilde{\theta}(\cdot)$.
Let $\{f_i\}$ be countable dense in the unit ball of $C(S)$, hence it is a separating 
class, i.e., $\forall i \int f_i d\mu = \int f_i d\nu$ 
implies $\mu=\nu$. 
For 
each $i$,
\begin{equation}
\zeta^i_n = \sum_{m=1}^{n-1}a(m)(f_i(Z^{(1)}_{m+1}) - \int f_i(y)p(dy|Z^{(1)}_m, A^{(1)}_m, \theta_m, w_m)),\nonumber 
\end{equation}
is a zero-mean martingale with $\mathcal{F}_n = \sigma(\theta_m, w_m, Z^{(1)}_m,A^{(1)}_m, m\leq n), n\geq 1$. 
Moreover, it is a square-integrable martingale due to the 
fact that $f_i$'s are bounded and each $\zeta^i_n$ is a finite sum. Its quadratic variation process 
\begin{equation}
A_n=\sum_{m=0}^{n-1}a(m)^2  E[(f_i(Z^{(1)}_{m+1}) - 
\int f_i(y)p(dy|Z^{(1)}_m, A^{(1)}_m, \theta_m, w_m))^2|\mathcal{F}_m] + E[(\zeta^i_0)^2]\nonumber 
\end{equation}
is almost surely convergent. 
By the martingale convergence theorem, $\{\zeta^i_n\}$  
converges a.s. Let $\tau(n,t)=\min\{m \geq n:  
t(m) \geq t(n)+t\}$ for $t\geq0, n\geq0$. Then as $n\to\infty$,
\begin{equation}
\sum_{m=n}^{\tau(n,t)} a(m)(f_i(Z^{(1)}_{m+1})-\int f_i(y)p(dy|Z^{(1)}_m, A^{(1)}_m,\theta_m, w_m))\to 0,\mbox{ a.s.,}\nonumber 
\end{equation}
for $t >0$. 
By our choice of $\{f_i\}$ and the fact that 
$\{a(n)\}$ are eventually non-increasing,
\begin{equation}
\sum_{m=n}^{\tau(n,t)}(a(m) - a(m+1))f_i(Z^{(1)}_{m+1}) \to 0,\mbox{a.s.}\nonumber 
\end{equation} 
Thus,
\begin{equation}
\sum_{m=n}^{\tau(n,t)} a(m)(f_i(Z^{(1)}_m)-\int f_i(y)p(dy|Z^{(1)}_m, A^{(1)}_m,\theta_m,w_m))\to 0,\mbox{ a.s.}\nonumber 
\end{equation}
which implies
\begin{equation}
\int_{t(n)}^{t(n)+t}(\int(f_i(z) - \int f_i(y)p(dy|z,a,\hat{\theta}(s),\hat{w}(s)))\mu(s,dzda))ds \to 0,\mbox{ a.s.}\nonumber 
\end{equation} 
Recall that $u(n)$ can be any general sequence other than $t(n)$. Therefore
\begin{equation}
\int_{u(n)}^{u(n)+t}(\int(f_i(z) - \int f_i(y)p(dy|z,a,\hat{\theta}(s),\hat{w}(s)))\mu(s,dzda))ds \to 0,\mbox{ a.s.,}\nonumber 
\end{equation}
(this follows from the fact that $a(n)\to 0$ and $f_i$'s are bounded)
where $\hat{\theta}(s) = \theta_n$ and $\hat{w}(s) = w_n$ when $s \in [t(n), t(n+1)), n\geq 0$. Now, one can claim from the above that
\begin{equation}
\int_{u(n)}^{u(n)+t}(\int(f_i(z) - \int f_i(y)p(dy|z,a,\bar{\theta}(s), \lambda(\bar{\theta}(s))))\mu(s,dzda))ds \to 0,\mbox{ a.s.}\nonumber 
\end{equation}
This is due to the fact that the map $S^{(1)} \times U^{(1)} \times \mathbb{R}^{d+k} \ni (z,a,\theta,w) \to \int f_i(y)p(dy|z,a,\theta,w)$ is continuous and hence 
uniformly continuous on the compact set $A =   S^{(1)} \times U^{(1)} \times M_1 \times M_2$ 
where $M_1$ is the compact set s.t. $\theta_n \in M_1~\forall n$ and 
 $M_2= \{w :  \|w\| \leq \max(\sup\|w_n\|, K')\}$
where $K'$ is the bound for the compact set $\lambda(M_1)$.    
Here we also use the fact that $\|w_m - \lambda(\bar{\theta}(s))\|\to 0$ for $s\in [t_m, t_{m+1})$ as $\lambda$ is Lipschitz and $\|w_m -\lambda(\theta_m)\| \to 0$. 
The above convergence is equivalent to
\begin{equation}
\int_{0}^{t}(\int(f_i(z) - \int f_i(y)p(dy|z,a,\bar{\theta}(s+u(n)), \lambda(\bar{\theta}(s+u(n)))))\mu(s+u(n),dzda))ds \to 0\mbox{ a.s.}\nonumber 
\end{equation} 
Fix a sample point in the probability one set on which the convergence above holds for all $i$. Then the convergence 
above leads to 
\begin{equation}
\label{conjun}
\int_{0}^{t}(\int f_i(z) - \int f_i(y)p(dy|z,a, \tilde{\theta}(s), \lambda(\tilde{\theta}(s))))\tilde{\mu}(s, dzda)ds =0~\forall i. 
\end{equation}
For showing the above, we use one part of the proof from  Lemma~2.3 of \cite{borkar} that if $\mu^n(\cdot) \to \mu^{\infty}(\cdot) \in \mathcal{U}$ then 
for any $t$,
\begin{equation}
\int_{0}^{t} \int \tilde{f}(s,z,a)\mu^n(s,dzda)ds - \int_{0}^{t} \int \tilde{f}(s,z,a)\mu^{\infty}(s,dzda)ds \to 0\nonumber 
\end{equation}
for all $\tilde{f} \in 
C([0,t] \times S^{(1)} \times U^{(1)})$. In addition, we make use of the fact that $\tilde{f}_n(s,z,a) =   \int f_i(y)p(dy|z,a,\bar{\theta}(s+u(n)), \lambda(\bar{\theta}(s+u(n))))$ 
converges uniformly to  $\tilde{f}(s,z,a) = \linebreak \int f_i(y)p(dy|z,a,\tilde{\theta}(s), \lambda(\tilde{\theta}(s)))$. To prove this, define
$g :C([0,t]) \times [0,t] \times S^{(1)} \times U^{(1)} \to \mathbb{R}$ by  $g(\theta(\cdot), s,z,a) 
 = \int f_i(y)p(dy|z,a, \theta(s),\lambda(\theta(s)))$. 
Let $A' = \{\bar{\theta}(u(n)+.)|_{[u(n),u(n)+t]}, n \geq 1\} \cup \tilde{\theta}(\cdot)|_{[0,t]}$. 
Using the same argument as in Lemma~\ref{eps} and \textbf{(A6)}, i.e., $\lambda$ is Lipschitz (the latter 
helps to claim that if $\theta_n(\cdot) \to \theta(\cdot)$
uniformly then $\lambda(\theta_n(\cdot)) \to \lambda(\theta(\cdot))$ uniformly),  
it can be seen that $g$ is continuous. Then $A'$ is compact  
as it is a union of a sequence of functions and its limit. 
So, $g|_{(A'\times [0,t] \times S^{(1)} \times U^{(1)})}$ is uniformly 
continuous. Then a similar argument as in Lemma~\ref{upsem} shows 
equicontinuity of $\{\tilde{f}_n(.,.)\}$ that results in 
uniform convergence and thereby (\ref{conjun}).
An application of Lebesgue's theorem in conjunction 
with (\ref{conjun}) shows that
\begin{equation}
\int (f_i(z) - \int f_i(y)p(dy|z,a,\tilde{\theta}(t), \lambda(\tilde{\theta}(t)))\tilde{\mu}(t, dzda) = 0~\forall i\nonumber 
\end{equation}
for a.e. $t$. 
By our choice of $\{f_i\}$, this leads to
\begin{equation}
\tilde{\mu}(t, dy \times U^{(1)}) = \int p(dy|z,a,\tilde{\theta}(t), \lambda(\tilde{\theta}(t)))\tilde{\mu}(t, dzda),\nonumber 
\end{equation}
a.e. $t$.     
\end{proof}

Lemma~\ref{ode} shows that every limit point $(\tilde{\mu}(\cdot), \tilde{\theta}(\cdot))$ of $(\mu(s+.),\bar{\theta}(s+.))$ as $s\to \infty$ is such that
$\tilde{\theta}(\cdot)$ satisfies (\ref{slow_tt}) with $\mu(\cdot) = \tilde{\mu}(\cdot)$. Hence, $\tilde{\theta}(\cdot)$ is absolutely continuous. Moreover, 
using Lemma~\ref{slowmu}, one can see that it satisfies (\ref{slower_di}) a.e. $t$, hence is a solution to the differential inclusion (\ref{slower_di}).  
\\ \indent

\begin{proof}[{Proof of Theorem \ref{thm} and \ref{thm_local}}] 
From the previous three lemmas it is easy to see that \linebreak
$A_0 = \cap_{t\geq 0}\overline{\{\bar{\theta}(s): s\geq t\}}$ is almost everywhere an internally chain transitive set of (\ref{slower_di}).
\end{proof}

\begin{corollary}
\label{main_col}
Under the additional assumption that the inclusion 
\begin{equation}
\dot{\theta}(t)\in \hat{h}(\theta(t))), \nonumber 
\end{equation}
has a global attractor set $A_1$,  
\begin{equation}
(\theta_n, w_n) \to \cup_{\theta^* \in A_1}(\theta^*, \lambda(\theta^*)) \mbox{a.s. as $n \to \infty$.}\nonumber 
\end{equation} 
\end{corollary}
\begin{proof}
Follows directly from Theorem \ref{thm} and Lemma~\ref{ga}.
\end{proof}

\section{Discussion on the assumptions: Relaxation of (A2)} 
\label{relax}
We discuss relaxation of the uniformity of the Lipschitz constant 
w.r.t state of the controlled Markov process for the 
vector field. The modified 
assumption here is
\begin{enumerate}[label=\textbf{(A\arabic*)'}]
\setcounter{enumi}{1}
 \item $h :  \mathbb{R}^{d+k} \times S^{(1)} \to \mathbb{R}^d$ is  
jointly continuous as well as Lipschitz in its first two arguments with the third argument fixed to same value and 
Lipschitz constant is a function of this value. The latter condition means that
\begin{equation}
\forall z^{(1)} \in S^{(1)}, \|h(\theta, w, z^{(1)}) - h(\theta', w', z^{(1)})\| \leq L^{(1)}(z^{(1)})(\|\theta-\theta'\| + \|w - w'\|).\nonumber
\end{equation}
A similar condition holds for $g$ where the Lipschitz constant is $L^{(2)}: S^{(2)} \to \mathbb{R}^+$.
\end{enumerate}
Note that this allows $L^{(i)}(\cdot)$ to be an unbounded measurable function making it discontinuous due to \textbf{(A1)}.
The straightforward solution for implementing this is to additionally assume the following:
\begin{enumerate}[label=\textbf{(A\arabic*)}]
\setcounter{enumi}{7}
\item $\sup_n L^{(i)}(Z^{(i)}_n) < \infty$ a.s.  
\end{enumerate}
still allowing $L^{(i)}(\cdot)$ to be an unbounded function.
As all our proofs in Section \ref{mres} are shown for every sample point of a probability 1 set, 
our proofs will go through. In the following we give such an example for the case where the Markov 
process is uncontrolled. 

It is enough to consider examples with locally compact $S^{(i)}$ (because then we can take the standard one-point compactification 
and define $L^{(i)}$ arbitrarily at the extra point).

Let $S^{(i)}=\Bbb Z$ and let $Z^{(i)}_n, n \geq 0$ be the Markov Chain on $\Bbb Z$ starting at $0$ with transition probabilities
$p(n,n+1)=p$ and $p(n,n-1)=1-p$. We assume $1/2 < p < 1$. Let $L^{(i)}(n) = \big( \frac{1-p}{p} \big)^n$.

Note that $Z^{(i)}_n,n \geq 0$ is a transient Markov Chain with $Z^{(i)}_n \to +\infty$ a.s. 
From this it follows that $\inf_n Z^{(i)}_n>-\infty$, and 
thus $\sup_n L^{(i)}(Z^{(i)}_n)< \infty$ almost surely. 
It follows that $(L^{(i)}(Z^{(i)}_n))_{n \in \Bbb N}$ is a bounded sequence with probability $1$, but this bound is 
clearly not deterministic since there is a non-zero probability that the sample path reaches large negative values.

However in the following we discuss on the idea of using moment assumptions to
analyze the convergence of single timescale controlled Markov noise framework of \cite{borkar}. 
We show that the iterates (\ref{eps2}) (with $\epsilon_n=0$)  converge to an internally 
chain transitive set of the o.d.e. (\ref{auto}). For this we prove 
Lemma \ref{track_fast} under the following assumptions: 
For all $T >2, i=1,2$,  
\begin{enumerate}[label=\textbf{(S\arabic*)}]
 \item The controlled Markov process $Y_n$ as described in \cite{borkar} takes values in a compact metric space. 
\item $1 \geq a(n)>0$, $\sum_n a(n) = \infty$, $\sum_n a(n)^2 < \infty$ and $a(n+1)\leq a(n), n\geq 0$. 
\item $h :  \mathbb{R}^{d} \times S \to \mathbb{R}^d$ Lipschitz in its first argument w.r.t the second. 
The condition means that
\begin{equation}
\forall z \in S, \|h(\theta, z) - h(\theta', z)\| \leq L(z)(\|\theta-\theta'\|).\nonumber.
\end{equation}
 \item Let  $\phi(n,T) = \max(m: a(n) + a(n+1) + \dots + a(n+m) \leq T)$ with the bound depending on $T$. 
Then \begin{equation}
        \sup_n E\left[\left(\sup_{0 \leq m \leq \phi(n,T)}L(Y_{n+m})\right)^{16}\right] < \infty. \nonumber
       \end{equation}
 \item \begin{equation}
        \sup_n E\left[e^{8\sum_{m=0} ^ {\phi(n,T)} a(n+m) L(Y_{n+m})}\right] < \infty. \nonumber
       \end{equation}
Note that \textbf{(S4)} and \textbf{(S5)} are trivially satisfied in the case when $L(z)= L$ for all $z \in S$ i.e. the 
case of Section \ref{secdef}.
\begin{remark}
As long as one can prove Lemma \ref{track_fast} for all $T >2$ it will hold for all $T>0$, thus one can combine \textbf{(S4)} and 
\textbf{(S5)} into the following assumption:
\begin{equation}
\sup_n E\left[e^{8T\sup_{0 \leq m \leq \phi(n,T)} L(Y_{n+m})}\right] < \infty. \nonumber 
\end{equation} 
As an instance where such an assumption is verified, consider the Markov process of \cite[ (3.4)]{metivier} defined by 
\begin{equation}
Y_{n+1}= A(\theta_n) Y_n + B(\theta_n) W_{n+1} \nonumber
\end{equation}
where $A(\theta), B(\theta), \theta \in \mathbb{R}^d$, are
$k \times k$-matrices and $(W_n)_{n\geq O}$ are independent and 
identically distributed $\mathbb{R}^k$-valued random variables.  
Assume that the following 
conditions hold true for all $x,y \in S$:
\begin{enumerate}
 \item $L(Y_n)$ is a non-decreasing sequence.
 \item For $r>0, R>0$, \begin{equation}
\sup_{\|\theta\| \leq R} e^{rL(A(\theta) x + B(\theta) y)} \leq L_R {\alpha_R}^r e^{r L(x)} + M_R e^{C_R L(y)} \nonumber 
\end{equation}
for some $C_R, M_R, L_R >0$ and $\alpha_R < 1$.
\end{enumerate}Then  
\begin{align}
\nonumber
&E\left[e^{rL(Y_n)}|Y_{n-1} = x, \theta_{n-1} = \theta\right] \\ \nonumber
&\leq \int e^{rL\left(A(\theta)x + B(\theta) y\right)} \mu_n (dy) \\ \nonumber
&\leq L_R {\alpha_R}^r  e^{rL(x)} + M_R E\left[e^{C_R L(W_n)}\right] \\ \nonumber
&=L_R{\alpha_R}^r  e^{r L(x)} + K_R,\nonumber
\end{align}
with $K_R = M_R E\left[e^{C_R L(W_n)}\right]$ (this follows from the fact that $W_n$ are i.i.d 
if we assume that $E\left[e^{C_R L(W_1)}\right] < \infty$). Choosing large values of $r$, one can show that 
\begin{equation}
E\left[e^{rL(Y_n)}|Y_{n-1} = x, \theta_{n-1} = \theta\right] \leq \beta_R e^{r L(x)}+ K_R \nonumber
\end{equation}
where $\beta_R = L_R{\alpha_R}^r < 1$.
Using the above, for large $r$ 
\begin{align}
E\left[e^{r L(Y_n)}\right] = E\left[E\left[e^{r L(Y_n)}| Y_{n-1}, \theta_{n-1}\right]\right] \leq \beta_R  E\left[e^{rL(Y_{n-1})}\right] + K_R, \nonumber
\end{align}
which shows that 
\begin{equation}
\sup_n E\left[e^{rL(Y_n)}\right] < \infty.\nonumber  
\end{equation} Choosing $r > 8T$,  
\begin{equation}
\sup_n E\left[e^{8TL(Y_n)}\right] < \infty.\nonumber  
\end{equation}
 Note that this is a much weaker assumption that \textbf{(A8)}.
\end{remark}

 \item The noise sequence $M_n, n \geq 0$ (need not be a martingale difference sequence) 
satisfies the following condition \begin{equation}
        \sup_n E\left[\left(\sum_{m=0} ^ {\phi(n,T)} \|M_{n+m+1}\|\right)^4\right] < \infty. \nonumber
       \end{equation}
\item $\sup_n \|\theta_n\| < \infty$.
\end{enumerate}
With the above assumptions we prove the following tracking lemma:
\begin{lemma}
\label{track}
For any $T >0, \sup_{t\in [s,s+T]}\|\bar{\theta}(t) - \theta^s(t)\| \to 0,$ a.s.  
\end{lemma}

\begin{proof}
Let $t(n) \leq t \leq t(n+m)$. Now, if $0 \leq k \leq (m-1)$ and $t \in (t(n+k), t(n+k+1)],$
\begin{align}
\|\theta^{t(n)}(t)\| &\leq \|\bar{\theta}(t(n)\| + \|\int_{t(n)}^{t} \tilde{h}(\theta^{t(n)}(\tau), \mu(\tau))d\tau\|\nonumber\\
                &\leq \|\theta_n\| + \sum_{l=0}^{k-1} \int_{t(n+l)}^{t(n+l+1)} (\|h(0,Y_{n+l})\|+ L(Y_{n+l})\|\theta^{t(n)}(\tau)\|))d\tau\nonumber\\
                & +\int_{t(n+k)}^t(\|h(0,Y_{n+k})\|+ L(Y_{n+k})\|\theta^{t(n)}(\tau)\|))d\tau\nonumber\\
                &\leq C_0 + MT+ \int_{t(n)}^{t} L(Y(\tau)) \|\theta^{t(n)}(\tau)\|d\tau \nonumber
\end{align}
where $Y(\tau) = Y_n$ if $\tau \in [t(n),t(n+1))$. Then it follows from an application of Gronwall inequality that
\begin{equation}
\|\theta^{t(n)}(t)\| \leq C e^{\int_{t(n)}^{t} L(Y(\tau)) d\tau} \mbox{~~a.e. $t$}\nonumber
\end{equation}
where $C=C_0 + MT$.
Next, 
\begin{align}
\|\theta^{t(n)}(t) - \theta^{t(n)}(t(n+k))\| &\leq \int_{t(n+k)}^t\|h(\theta^{t(n)}(s), Y_{n+k})\|ds\nonumber\\
                                   &\leq \|h(0,Y_{n+k})\|(t-t(n+k))+ L(Y_{n+k})\int_{t(n+k)}^t\|\theta^{t(n)}(s)\|ds\nonumber\\
                                   &\leq M a(n+k) + C L(Y_{n+k})\int_{t(n+k)}^te^{\int_{t(n)}^s L(Y(\tau)) d\tau} ds.\nonumber
\end{align}
Then 
\begin{align}
&\|\int_{t(n)}^{t(n+m)} (h(\theta^{t(n)}(t), \mu(t)) - h(\theta^{t(n)}([t]), \mu([t])))dt\|\nonumber\\
&\leq \sum_{k=0}^{m-1}\int_{t(n+k)}^{t(n+k+1)}\|h(\theta^{t(n)}(t), Y_{n+k}) - h(\theta^{t(n)}([t]), Y_{n+k})\|dt\nonumber\\
&\leq \sum_{k=0}^{m-1}L(Y_{n+k}) \int_{t(n+k)}^{t(n+k+1)} \|\theta^{t(n)}(t) - \theta^{t(n)}(t(n+k))\|dt\nonumber\\
&\leq \sum_{k=0}^{m-1} c_k\nonumber
\end{align}
where 
\begin{equation}
c_k = L(Y_{n+k})a(n+k)^2\left[M + CL(Y_{n+k}) e^{\sum_{i=0}^{k} a(n+i) L(Y_{n+i})}\right].\nonumber 
\end{equation}
\begin{align}
\|\bar{\theta}(t(n+m))-\theta^{t(n)}(t(n+m))\| &\leq \sum_{k=0}^{m-1}L(Y_{n+k})a(n+k)\|\bar{\theta}(t(n+k)) - \theta^{t(n)}(t(n+k))\|\nonumber\\
& + \sum_{k=0}^{m-1} c_k + \|\delta_{n,n+m}\|,\nonumber 
\end{align}
where $\delta_{n,n+m}=\sum_{k=n}^{n+m-1}a(k)M_{k+1}$.

Therefore using discrete Gronwall inequality we get
\begin{equation}
\|\bar{\theta}(t(n+m))-\theta^{t(n)}(t(n+m))\| \leq r(m,n)  e^{\sum_{k=0}^{m-1} a(n+k) L(Y_{n+k})}\nonumber
\end{equation}
where $r(m,n) = \sum_{k=0}^{m-1} (c_k + a(n+k) \|M_{n+k+1}\|)$.

Now, for some $\lambda \in [0,1]$, 
\begin{align}
&\|\theta^{t(n)}(t) - \bar{\theta}(t)\| \nonumber \\
&\leq (1-\lambda) \|\theta^{t(n)}(t(n+m+1)) - \bar{\theta}(t(n+m+1)) +\lambda \|\theta^{t(n)}(t(n+m))-\bar{\theta}(t(n+m))\| \nonumber \\
& + \max(\lambda, 1- \lambda) \int_{t(n+m)}^{t(n+m+1)} \|\tilde{h}(\theta^{t(n)}(s),\mu(s))\|ds \nonumber \\
&\leq r(m+1,n) e^{\sum_{k=0}^{m} a(n+k) L(Y_{n+k})} + a(n+m)\left[M + C L(Y_{n+m}) e^{\sum_{k=0}^{m} a(n+k) L(Y_{n+k})}\right]\nonumber. 
\end{align}
Therefore 
\begin{align}
\rho(n,T):= \sup_{t \in [t(n),t(n) + T]} \|\theta^{t(n)}(t) - \bar{\theta}(t)\| &\leq r(\phi(n,T+1),n) e^{\sum_{k=0}^{\phi(n,T)} a(n+k) L(Y_{n+k})} \nonumber \\
                                                                      &+ a(n)\left[M + C \sup_{0\leq m \leq \phi(n,T)} L(Y_{n+m}) e^{\sum_{k=0}^{\phi(n,T)} a(n+k) L(Y_{n+k})}\right].\nonumber 
\end{align}

Now to prove the a.s. convergence of the quantity in the left hand side as $n \to \infty$, we have using Cauchy-Schwartz 
inequality:
\begin{align}
\sum_{n=1}^{\infty} E[{\rho(n,T)}^2] \leq & 2K_T \sum_{n=1}^{\infty}\left(E\left[\left(r(\phi(n,T+1),n)\right)^4\right]\right)^{1/2} + 4M^2\sum_{n=0}^{\infty}a(n)^2 + \nonumber \\
                                          & 4C^2\sum_{n=1}^{\infty} a(n)^2 E\left[\left(\sup_{0\leq m \leq \phi(n,T)} L(Y_{n+m})\right)^2 e^{2\sum_{k=0}^{\phi(n,T)} a(n+k) L(Y_{n+k})}\right],\nonumber
\end{align}
where $K_T = \sqrt{\sup_n E[e^{4\sum_{k=0}^{\phi(n,T)} a(n+k) L(Y_{n+k})}]}$ which depends only on $T$ due to \textbf{(S5)}.
Now, the third term in the R.H.S is clearly finite from the assumptions \textbf{(S4)} and \textbf{(S5)}.
Now we analyze the first term i.e. 
\begin{align}
\label{num}
\sum_{n=1}^{\infty}\left(E\left[{r(\phi(n,T+1),n)}^4\right]\right)^{1/2} \leq & 2\sqrt{2}\sum_{n=1}^{\infty}\left(E\left[\left(\sum_{k=0}^{\phi(n,T)} c_k\right)^4\right]\right)^{1/2} \nonumber \\
                                                      & + 2\sqrt{2}\sum_{n=1}^{\infty} \left(E\left[\left(\sum_{k=0}^{\phi(n,T)} a(n+k) \|M_{n+k+1}\|\right)^4\right]\right)^{1/2}.
\end{align}

Next we analyze the first term in the R.H.S of (\ref{num}) again using Cauchy-Schwartz inequality:
\begin{align}
&\sum_{n=1}^{\infty}\left(E\left[\left(\sum_{k=0}^{\phi(n,T)} c_k\right)^4\right]\right)^{1/2} \nonumber \\
&\leq 8 M^2 \sum_{n=1}^{\infty} \phi(n,T)^2 a(n)^4  \left(E\left[\left(\sup_{0 \leq k \leq \phi(n,T)} L(Y_{n+k})\right)^4\right]\right)^{1/2}+ \nonumber \\
& 8 C^2 \sum_{n=1}^{\infty} \phi(n,T)^2 a(n)^4  \left(E\left[\left(\sup_{0 \leq k \leq \phi(n,T)} L(Y_{n+k})\right)^8 e^{4\sum_{i=0}^{\phi(n,T)} a(n+i) L(Y_{n+i})}\right]\right)^{1/2}.\nonumber 
\end{align}
Therefore the the R.H.S will be finite if we can show that 
$\sum_{n=1}^{\infty} \phi(n,T)^2 a(n)^4$ is finite. 
For common step-size sequence $a(n) =\frac{1}{n}$, $\phi(n,T)= O(n)$ thus the above series converges clearly. 
One can make the series converge 
for all $a(n)= \frac{1}{n^k}$ with $\frac{1}{2} < k \leq 1$ by putting 
assumptions on higher moments in \textbf{(S4)} and \textbf{(S5)} .

In the above we have used the following inequality repeatedly for non-negative random variables $X$ and $Y$:
\begin{equation}
\sqrt{E\left[\left(X+Y\right)^{2^n}\right]} \leq 2^{\frac{2n-1}{2}}\left[\sqrt{E[X^{2^n}]} +  \sqrt{E[Y^{2^n}]}\right]\nonumber
\end{equation}
with $n \in \mathbb{N}$.

Now, 
\begin{align}
\sum_{n=1}^{\infty} \left(E\left[\left(\sum_{k=0}^{\phi(n,T)} a(n+k) \|M_{n+k+1}\|\right)^4\right]\right)^{1/2} \nonumber \\
\leq \sum_{n=1}^{\infty} a(n)^2\left(E\left[\left(\sum_{k=0} ^ {\phi(n,T)} \|M_{n+k+1}\|\right)^4\right]\right)^{1/2}\nonumber
\end{align}
which is finite under assumption \textbf{(S5)} and the fact that $a(n)$ are non-increasing.
\end{proof}
\begin{remark}
However there are problems in the method of using moment assumptions.
 Even for the single timescale case the proof of Lemma 2.3 of \cite{borkar} does not go through.
This lemma is used to prove that 
the o.d.e. trajectory $\theta^{u(n(k))}(\cdot)$ associated with $\mu(u(n(k))+.)$ 
tracks (in the limit) the o.d.e. trajectory $\theta^{\infty}(\cdot)$ 
associated  with $\tilde{\mu}(\cdot)$ for the o.d.e (12).
Here  $u(n(k)) \uparrow \infty$ is such that
$\bar{\theta}(u(n(k))+.) \to \tilde{\theta}(\cdot)$ and 
$\mu(u(n(k))+.) \to \tilde{\mu}(\cdot)$. In this context,  
the statement of Lemma 2.3 \cite{borkar} will be changed to  
the following:
\begin{lemma}
\label{eps3}
\begin{equation}
\lim_{k \to \infty} \sup_{t\in [0, T]}\|\theta^{u(n(k))}(t) - \theta^{\infty}(t)\| = 0 \mbox{~a.s.}\nonumber 
\end{equation}
for every $T >1$ and $L:S\to [1,\infty)$. 
\end{lemma}
\proof
Following the same lines as in \cite[Lemma 2.3]{borkar}
\begin{align}
&\rho'(u(n(k)),T):=\sup_{t \in [0,T]} \|\theta^{u(n(k))}(t) - \theta^{\infty}(t)\| \nonumber \\
&\leq K_{T,u(n(k))}\left(\|\theta^{u(n(k))}(0) - \theta^{\infty}(0)\| + \sup_{t \in [0,T]} \|\int_{0}^{t}(\tilde{h}(\theta^{\infty}(s),\mu(u(n(k))+s))- \tilde{h}(\theta^{\infty}(s), \tilde{\mu}(s)))ds\|\right),\nonumber  
\end{align}
where the second term goes to zero a.s. and 
\begin{align}
K_{T,u(n(k))}&= 1+ \left[\int_{0}^T \left(\int L(z)\mu(u(n(k))+s,dzda)\right)ds\right] e^{\int_{0}^T \left(\int L(z)\mu(u(n(k))+s,dzda) \right)ds} \nonumber \\
             &= 1+ \left[\sum_{m=0}^{\phi(u(n(k)),T)}a(u(n(k)) + m)L(Y_{u(n(k))+m})\right]e^{\sum_{m=0}^{\phi(u(n(k)),T)}a(u(n(k)) + m)L(Y_{u(n(k))+m})} \nonumber \\
             &\leq 2 \left[\sum_{m=0}^{\phi(u(n(k)),T)}a(u(n(k)) + m)L(Y_{u(n(k))+m})\right]e^{\sum_{m=0}^{\phi(u(n(k)),T)}a(u(n(k)) + m)L(Y_{u(n(k))+m})}. \nonumber 
\end{align}

It is not clear how to apply Borel-Cantelli lemma here to prove as in the
proof of Lemma 15 of the as there is no step-size factor in the r.h.s of the above. 

Similar problem will arise if we are using this method for two time-scale analysis. Then even in Lemma 12 we have to make 
strong assumptions such as 
\begin{equation}
\sup_n E\left[\left(\sum_{m=0} ^ {\phi(n,T)} \|w_{n+m} - \lambda(\theta_{n+m})\| \right)^{8}\right] < \infty. \nonumber
\end{equation} 
The reason is that our proofs are based on Gronwall's inequality, using which we bound the 
difference between the o.d.e trajectory and algorithm's iterates by a quantity which is a 
product of a constant and a term going to zero (e.g. Lemma 2.2
of \cite{borkar}). Therefore
the fact that additional errors  vanish asymptotically may not be  useful in the proof 
if we use moment assumptions unless the term going 
to zero is determined by step-size (this is the case for Lemma 15). 
\end{remark}

\section{Conclusion}
We considered in this chapter two time-scale stochastic approximations with both iterates 
driven by different controlled Markov processes as well as martingale difference noise sequences 
and provided a proof of convergence of such recursions under general assumptions.  
Note that we have assumed stability of the iterates to prove our results. A nice future direction 
will be to find sufficient and verifiable condition for this assumption. To our knowledge this 
is not a natural extension of the result for single time-scale stochastic approximation 
(also called the  Borkar-Meyn theorem, see  \cite[Chapter 3]{borkar1} for details) into two time-scale as the coupled o.d.e 
has no attractor. One possible approach may be to use Lyapunov function.

In the next chapter, we shall study an application of the results developed here for proving 
convergence of off-policy temporal difference learning algorithm with linear function approximation
which is a reinforcement learning algorithm for the problem of prediction. For this problem, classical 
Poisson equation can be used to handle Markov noise if we assume state space of the underlying Markov chain is finite. 
However, we show that the results developed under general assumptions still can be applied for this special case.

\chapter{Off-policy temporal difference learning with linear function approximation in on-line learning environment}
\label{chap:off}
\section{Brief Introduction and Organization}
In this chapter we present an application of the general results developed in the previous chapter to show that the assumptions made for the general case are still satisfied for the special case. 
This is achieved by providing an almost sure convergence analysis
of off-policy temporal difference learning algorithm (only policy evaluation) with linear function approximation in on-line learning environment using the results of previous chapter. The algorithm
considered is TDC with importance weighting as mentioned in Chapter \ref{chap:introduction} as popular off-policy learning algorithm such as Q-learning \cite{watkins} may diverge when function approximation is deployed. An important point to note is that this is not a usual two time-scale scenario such as actor-critic method \cite{konda_actor}, however, two time-scale stochastic approximation is used so that the associated o.d.es have global attractor. 

The organization of this chapter is as follows: Section \ref{back} describes the TDC algorithm with importance weighting. 
Section \ref{amain_res} gives the convergence proof of the algorithm. Section \ref{empiric} 
shows empirical results supporting our theoretical results. 
Finally we conclude by providing some interesting future directions. 
\section{Background and description of TDC with importance weighting}
\label{back}
We need to estimate the value function for a target policy $\pi$ 
given the continuing evolution of the underlying MDP (with finite state and action spaces $S$ and $A$ respectively,   
specified by expected reward $r(\cdot,\cdot,\cdot)$ 
and transition probability kernel $p(\cdot|\cdot,\cdot)$) for a 
behaviour policy $\pi_b$ with $\pi \neq \pi_b$. 
Suppose, the above-mentioned on-policy trajectory is $(X_{n},A_n, R_{n},X_{n+1}), n\geq 0$ where  
$\{X_n\}$ is a time-homogeneous irreducible Markov chain with unique stationary distribution $\nu$ 
and generated from the behavior policy $\pi_b$. Here the quadruplet $(s,a,r,s')$ 
represents (current state, action, reward, next state). 
Also, assume that $\pi_b(a|s) > 0 ~\forall s \in S, a \in A$. We need 
to find the solution $\theta^*$ for the following: 
\begin{equation}
\label{fixpoint}
\begin{split}
0&=\sum_{s,a,s'}\nu(s)\pi(a|s)p(s'|s,a)\delta(\theta;s,a,s')\phi(s) = E[\rho_{X,A}\delta_{X,R,Y}(\theta)\phi(X)]\\ 
 &= b - A\theta,
\end{split}
\end{equation}
where 
\begin{itemize}
\item [(i)]$\theta \in \mathbb{R}^d$ is the parameter for value function,
\item [(ii)]$\phi: S\to \mathbb{R}^d$ is a vector of state features,
\item [(iii)] $X \sim \nu$,
\item [(iv)] $0<\gamma < 1$ is the discount factor,
\item [(v)] $E[R|X=s,Y=s'] = \sum_{a\in A}\pi_b(a|s)r(s,a,s')$,
\item [(vi)] $P(Y=s'|X=s) = \sum_{a\in A}\pi_b(a|s) p(s'|s,a)$,
\item [(vii)] $\delta(\theta; s,a,s')= r(s,a,s') + \gamma \theta^T\phi(s') - \theta^T\phi(s)$
is the temporal difference term with expected single-stage reward,
\item [(viii)] $\rho_{X,A} = \frac{\pi(A | X)}{\pi_b(A|X)}$,
\item [(ix)] $\delta_{X,R,Y}=R + \gamma \theta^T\phi(Y) - \theta^T\phi(X)$,
\item [(x)] $A=E[\rho_{X,A}\phi(X)(\phi(X) -\gamma\phi(Y))^T]$, $b=E[\rho_{X,A}R\phi(X)]$.
\end{itemize}
The desired approximate value function under the target policy $\pi$ is $V_{\pi}^*={\theta^*}^T\phi$. 
Let $V_\theta = {\theta}^T\phi$.
It is well-known (\cite{maeith}) that $\theta^*$ (solution to (\ref{fixpoint})) satisfies the projected fixed point equation namely
\begin{equation}
V_{\theta}= \Pi_{\mathcal{G},\nu}T^{\pi}V_{\theta},\nonumber 
\end{equation}
where 
\begin{equation}
\Pi_{\mathcal{G}, \nu}\hat{V} =
\arg\min_{f \in \mathcal{G}} (\|\hat{V} - f\|_{\nu}),\nonumber 
\end{equation}
with $\mathcal{G} = \{V_{\theta} | \theta \in \mathbb{R}^d\}$
and the Bellman operator
\begin{equation}
T^{\pi}V_\theta(s) = \sum_{s' \in S} \sum_{a\in A}\pi(a|s)p(s'|s, a)\left[\gamma V_\theta(s') + r(s, a, s')\right]. \nonumber 
\end{equation}
Here $\|\cdot\|_{\nu}$ is the weighted Euclidean norm defined by $\|f\|^2_{\nu}=\sum_{s\in S}f(s)^2 \nu(s)$,
Therefore to find $\theta^*$, the idea is to minimize the mean square projected 
Bellman error (MSPBE) $J(\theta)= \|V_{\theta} - \Pi_{\mathcal{G},\nu}T^{\pi}V_{\theta}\|^2_{\nu}$ using stochastic gradient descent.
It can be shown that the expression of gradient contains product of multiple expectations. Such framework can be modelled by 
two time-scale stochastic approximation where one iterate stores the quasi-stationary estimates of some of the expectations and the 
other iterate is used for sampling. 

%

We consider the TDC (Temporal Difference with Correction) algorithm with importance-weighting 
from Sections 4.2 and 5.2 of \cite{maeith}. 
The gradient in this case can be shown to satisfy 
\begin{align}
-\frac{1}{2}\nabla J(\theta)&=E[\rho_{X,A}\delta_{X,R,Y}(\theta)\phi(X)] - \gamma E[\rho_{X,A}\phi(Y)\phi(X)^T]w(\theta),\nonumber\\
w(\theta) &= E[\phi(X)\phi(X)^T]^{-1}E[\rho_{X,A}\delta_{X,R,Y}(\theta)\phi(X)].\nonumber 
\end{align}Define $\phi_n = \phi(X_n)$, $\phi'_n = \phi(X_{n+1})$, $\delta_n(\theta) = \delta_{X_n, R_n, X_{n+1}}(\theta)$ and 
$\rho_n=\rho_{X_n,A_n}$.
Therefore the associated iterations in this algorithm are: 
\begin{align}
\theta_{n+1} &= \theta_n + a(n) \rho_n\left[\delta_{n}(\theta_n)\phi_n - \gamma \phi'_{n}\phi_n^Tw_n\right],\label{tdc_slow} \\
w_{n+1} &= w_n + b(n) \left[(\rho_n\delta_{n}(\theta_n) - \phi_n^Tw_n)\phi_n\right], \label{tdc_fast}
\end{align}
\\ \indent 
with $\{a(n)\}, \{b(n)\}$ satisfying conditions which will be specified later. 
Note that the second term inside bracket in (\ref{tdc_slow}) is essentially an adjustment or correction
of the TD update so that it follows the gradient of the
MSPBE objective function thus helping in the desired convergence.

Note that the sub-sampling version of TDC algorithm (therefore the offline version of TDC
algorithm) can be written in the following way:
\begin{align}
\theta_{n+1} &= \theta_n + a(n) I_{\{A_n = \pi(X_n)\}}\left[\delta_{n}(\theta_n)\phi_n - \gamma \phi'_{n}\phi_n^Tw_n\right],\nonumber \\
w_{n+1} &= w_n + b(n) I_{\{A_n = \pi(X_n)\}}\left[(\delta_{n}(\theta_n) - \phi_n^Tw_n)\phi_n\right], \nonumber
\end{align}   
where $I_{\{A_n = \pi(X_n)\}} =1$ if $A_n = \pi(X_n)$ and $0$ otherwise.
In the rest of the chapter both the above algorithms will be denoted by ONTDC and OFFTDC respectively except the 
figures in Section \ref{empiric} where we mention the full name.
\section{Almost sure convergence proof of ONTDC}
\label{amain_res}

\begin{theorem}[Convergence of TDC with importance-weighting]
\label{th2}
Consider the iterations (\ref{tdc_fast}) and (\ref{tdc_slow}) of the TDC. Assume the following:
\begin{itemize}
 \item $\{a(n)\}, \{b(n)\}$ satisfy \textbf{(A4)}.
 \item $\{(X_n,R_n,X_{n+1}), n\geq0\}$ is such that $\{X_n\}$ is a time-homogeneous finite state irreducible Markov chain 
   generated from the behavior policy $\pi_b$ with unique stationary distribution $\nu$. 
$E[R_{n}|X_{n}=s,X_{n+1}=s'] = \sum_{a\in A} \pi_b(a|s)r(s,a,s')$ and $P(X_{n+1} =s'|X_{n}=s) = \sum_{a\in A}\pi_b(a|s)p(s'|s,a)$ 
where $\pi_b$ is the behaviour  policy,  
$\pi \neq \pi_b$.  
Also, $E[R_n^2 | X_n, X_{n+1}] < \infty~\forall n$ almost surely, and
\item $C=E[\phi(X)\phi(X)^T]$ and  $A=E[\rho_{X,R_n}\phi(X)(\phi(X) -\gamma\phi(X_{n+1}))^T]$ are non-singular where $X \sim \nu$. 
\item $\pi_b(a|s) > 0 ~\forall s \in S, a \in A$.
\item $\sup_n(\|\theta_n\| + \|w_n\|) < \infty$ w.p. 1. 
\end{itemize}
Then the parameter vector $\theta_n$ converges
with probability one as $n \to \infty$ to the TD(0) solution (\ref{fixpoint}). 
\end{theorem}
\begin{proof}
The iterations (\ref{tdc_fast}) and (\ref{tdc_slow}) can be cast into the framework of Section \ref{def} 
with \begin{itemize}
\item $Z^{(i)}_n = X_{n-1}$,
\item $h(\theta,w,z) = E[(\rho_n(\delta_{n}(\theta)\phi_n-\gamma \phi'_{n}\phi_n^Tw))|X_{n-1}=z]$,
\item $g(\theta,w,z)=E[((\rho_n\delta_{n}(\theta) - \phi_n^Tw)\phi_n)|X_{n-1}=z]$,
\item $M^{(1)}_{n+1}=\rho_n(\delta_{n}(\theta_n)\phi_n - \gamma \phi'_{n}\phi_n^Tw_n)-E[\rho_n(\delta_{n}(\theta_n)\phi_n - \gamma \phi'_{n}\phi_n^Tw_n)|X_{n-1}, \theta_n, w_n]$,
\item $M^{(2)}_{n+1}=(\rho_n\delta_{n}(\theta_n) - \phi_n^Tw_n)\phi_n - E[(\rho_n\delta_{n}(\theta_n) - {\phi_n}^T w_n)\phi_n|X_{n-1}, \theta_n, w_n]$,
\item $\mathcal{F}_n = \sigma(\theta_m, w_m, R_{m-1}, X_{m-1},A_{m-1}, m \leq n, i = 1, 2), n \geq 0$. 
\end{itemize}
Note that in (ii) and (iii) we can define $h$ and $g$ independent of $n$ due to time-homogeneity of $\{X_n\}$.  
\\ \indent
Now, we 
verify the assumptions \textbf{(A1)-(A7)} (mentioned in Sections \ref{def} and \ref{assump} of Chapter \ref{chap:intro}) for our application:
\begin{itemize}
\item  \textbf{(A1)}: $Z^{(i)}_n, \forall n, i=1,2$ takes values in compact metric space as $\{X_n\}$ is a finite state Markov chain.  
\item  \textbf{(A5)}: Continuity of transition kernel follows trivially from the fact that we have a finite state MDP.
 \item \textbf{(A2)} \begin{itemize}
\item \begin{align}
        &\|h(\theta, w,z) - h(\theta',w',z)\|\nonumber\\
        &=\|E[\rho_n(\theta-\theta')^T(\gamma \phi(X_{n+1}) - \phi(X_n))\phi(X_n) - \gamma \rho_n \phi(X_{n+1})\phi(X_n)^T(w-w')|X_{n-1}=z]\|\nonumber\\
        &\leq L(2\|\theta-\theta'\|M^2 + \|w-w'\|M^2)\nonumber,
       \end{align}
where $M=\max_{s\in S}\|\phi(s)\|$ with $S$ being the state space of the MDP and $L=\max_{(s,a)\in (S\times A)}\frac{\pi(a|s)}{\pi_b(a|s)}$. 
Hence $h$ is Lipschitz continuous in the 
first two arguments uniformly w.r.t the third. In the last inequality above, we use the 
Cauchy-Schwarz inequality. 
\item As with the case of $h$, $g$ can be shown to be Lipschitz continuous in the 
first two arguments uniformly w.r.t the third.
\item Joint continuity of $h$ and $g$ follows from (iii)(a) and (b) 
respectively as well as the finiteness of $S$.
\end{itemize}
\item \textbf{(A3)}: Clearly, $\{M_{n+1}^{(i)}\}, i=1,2$ are martingale difference sequences w.r.t. increasing $\sigma$-fields $\mathcal{F}_n$.
Note that $E[\|M_{n+1}^{(i)}\|^2 | \mathcal{F}_n] \leq K(1 + \|\theta_n\|^2 + \|w_n\|^2)$ a.s., $n\geq 0$ since
$E[R_n^2 | X_n, X_{n+1}] < \infty$ for all $n$ almost surely and $S$ is finite.
\item \textbf{(A4)}: This follows from the conditions (i) in the statement of Theorem \ref{th2}.  
\end{itemize}

Now, one can see that 
the faster o.d.e. becomes 
\begin{equation}
\dot{w}(t)=E[\rho_{X,A_n}\delta_{X,R_n,X_{n+1}}(\theta)\phi(X)] - E[\phi(X)\phi(X)^T]w(t).\nonumber 
\end{equation} 
Clearly, $C^{-1}E[\rho_{X,A_n}\delta_{X,R_n,X_{n+1}}(\theta)\phi(X)]$
is the globally asymptotically stable equilibrium of the o.d.e. The corresponding Lyapunov function 
$V(\theta,w) = \frac{1}{2} \|Cw - E[\rho_{X,A_n}\delta_{X,R_n,X_{n+1}}(\theta)\phi(X)]\|^2$ is 
continuously differentiable. Additionally, $\lambda(\theta)=C^{-1}E[\rho_{X,A_n}\delta_{X,R_n,X_{n+1}}(\theta)\phi(X)]$ 
and it is   
Lipschitz continuous in $\theta$, verifying \textbf{(A6)'}. For the slower o.d.e., the global 
attractor is $A^{-1}E[\rho_{X,A_n}R_n\phi(X)]$ verifying the additional assumption in Corollary \ref{main_col}. 
The attractor set here is a singleton. 
Also, \textbf{(A7)} is (v) in the statement of Theorem \ref{th2}. 
Therefore 
the assumptions $(\mathbf{A1}) - (\mathbf{A7})$ are verified. The proof would then follow
from Corollary \ref{main_col}. 
\end{proof}

\begin{remark}
The reason for using two time-scale framework for the TDC algorithm is to make sure that the O.D.E's have globally asymptotically stable equilibrium. 
\end{remark}
\begin{remark}
Because of the fact that the gradient is a product of two expectations the scheme 
is a  ``pseudo''-gradient descent which helps to find 
the global minimum here.
\end{remark}
\begin{remark}
Here we assume the stability of the iterates (\ref{tdc_fast}) and (\ref{tdc_slow}). 
Certain sufficient conditions have been sketched for showing  
stability of single timescale stochastic recursions with controlled Markov noise 
\cite[p.~75, Theorem 9]{borkar1}. This subsequently needs to be 
extended to the case of two time-scale recursions. 
   
Another way to ensure boundedness of the iterates is to use a projection operator. 
However, projection may introduce spurious fixed points on the boundary of the projection region and  
finding globally asymptotically stable equilibrium of a projected o.d.e. is hard. 
Therefore we do not use projection in 
our algorithm. 
\end{remark}
\begin{remark}
Convergence analysis for TDC with importance weighting along with 
eligibility traces cf. \cite[p.~74]{maeith} where it is called GTD($\lambda$)can be done similarly using our results. 
The main advantage is that it works for $\lambda < \frac{1}{L\gamma}$ ($\lambda\in [0,1]$ being the eligibility function) 
whereas the analysis in \cite{yu} is shown  
only for $\lambda$ very close to 1.   
\end{remark}
\section{Empirical results}
\label{empiric}
For the assessment of the algorithm experimentally we have compared the result 
on a variation of the classic Baird's off-policy counter-example \cite[Fig. 2.4]{maeith} 
and $\theta\rightarrow 2\theta$ problem \cite[Section 3]{emphatic_td}. In 
both cases, we compare the TD(0), OFFTDC and ONTDC. Unlike \cite{sutton} 
where updating was done synchronously
in dynamic-programming-like sweeps through the
state space, we consider the usual stochastic approximation scenario where only simulated sample trajectories 
are taken as input to the algorithms i.e. the algorithms do not use any knowledge of the probabilities 
for the underlying Markov decision process. For Baird's problem our performance metric is 
Root Mean Squared Error (RMSE) defined to be  
the square root of the average of the square of the deviation between 
true value function and the estimated value function. For $\theta \to 2\theta$ problem 
the $y$-axis is $\theta$ itself.
The average is taken over 1000 simulation runs and the metric 
is plotted  against the number of times 
$\theta_n$ is updated. While 
the analysis has been shown for the diminishing step-size case, we implement here the 
algorithm with constant step-sizes as in \cite{maeith,sutton}. 


The  $\theta\rightarrow 2\theta$ problem consists of only 2 states where $\theta$ 
and $2\theta$ are the estimated value of the states. According to its behavior policy 
with probability $p=\frac{1}{2}$ it stays on the same state and chooses the other 
state. The target policy is to choose the action that accesses the second 
state with probability 1 (See Fig. 1 in \cite[Section 3]{emphatic_td} for details). 
The constant step-sizes  are chosen as $a(n)=.075;b(n)=.05$ for the two time-scale algorithms and 
$\alpha=.075$ for single timescale algorithms. The simulations are run for 1000 
different sample paths. Rewards in all transitions 
are zero. The initial values are $\theta=1 $ and $w=0$. The results are summarized in Figure \ref{theta_2}.

\begin{figure}
  \centering
  \begin{minipage}[b]{0.49\textwidth}
    \includegraphics[width=\textwidth]{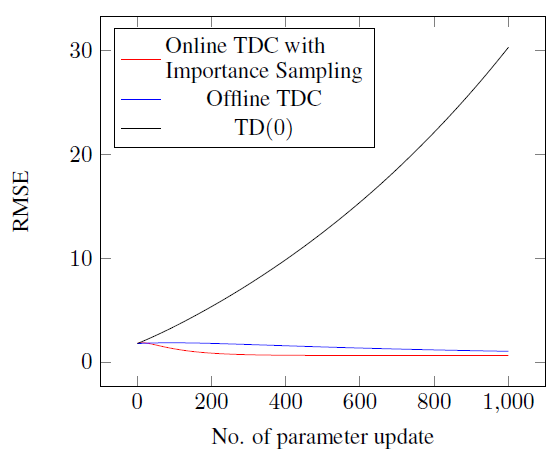}
    \caption{Comparison between TD(0), OFFTDC and ONTDC for Baird's counterexample}
\label{baird_off}  
\end{minipage}
  \hfill
\begin{minipage}[b]{0.49\textwidth}
    \includegraphics[width=\textwidth]{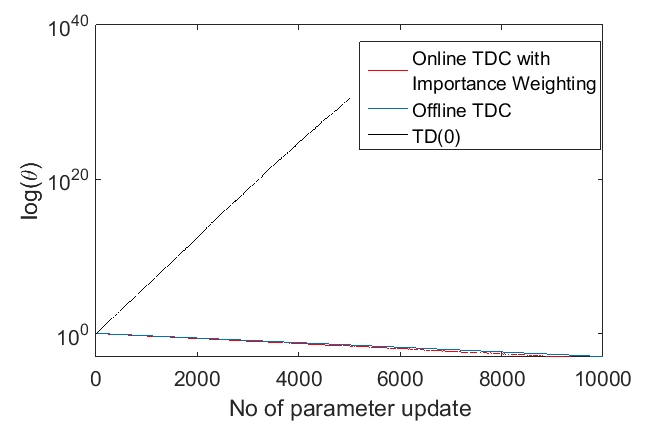}
    \caption{Comparison between TD(0), OFFTDC and ONTDC for $\theta \to 2 \theta$}
\label{theta_2}  
\end{minipage}
\end{figure}

\begin{figure}
\centering
  \begin{minipage}[b]{0.49\textwidth}
    \includegraphics[width=\textwidth]{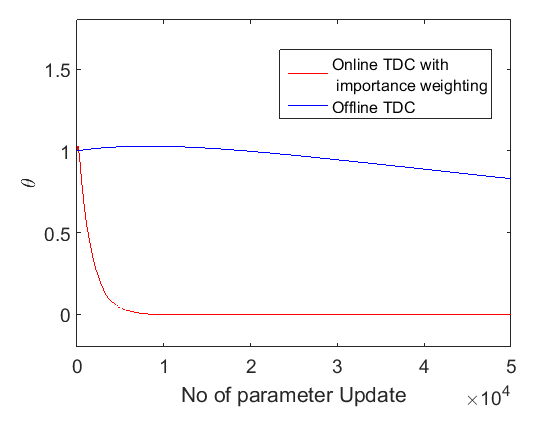}
    \caption{OFFTDC vs. ONTDC for $\theta \to 2\theta$ problem: $p$=.01}
    \label{offvsonp01}
\end{minipage}
  \hfill
\begin{minipage}[b]{0.49\textwidth}
    \includegraphics[width=\textwidth]{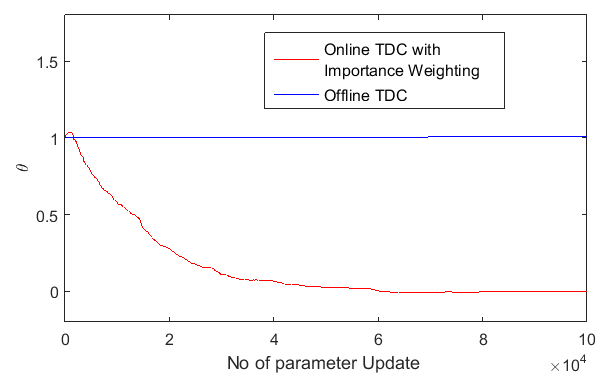}
    \caption{OFFTDC and ONTDC for $\theta \to 2\theta$ problem: $p$=.001}
    \label{offvsonp001}
\end{minipage}
\end{figure}

\begin{figure}
\centering
  \begin{minipage}[b]{0.49\textwidth}
    \includegraphics[width=\textwidth]{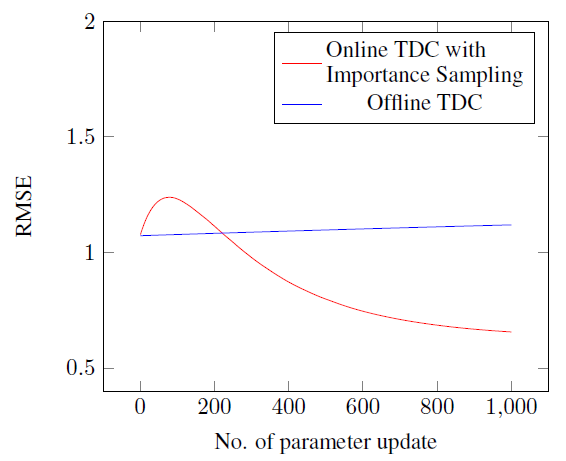}
    \caption{OFFTDC vs. ONTDC for Baird's counterexample: $q$=.01}
    \label{offvsonq01}
\end{minipage}
  \hfill
\begin{minipage}[b]{0.49\textwidth}
    \includegraphics[width=\textwidth]{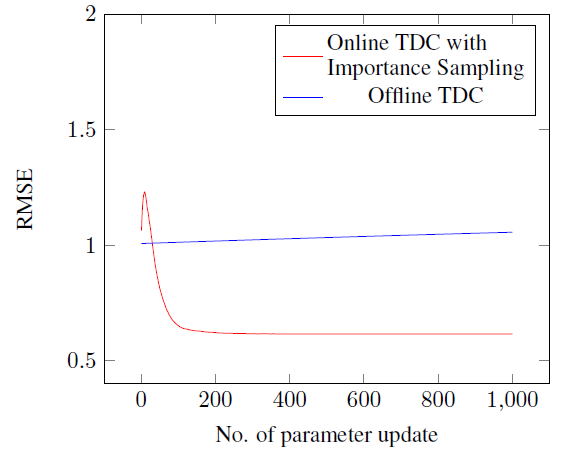}
    \caption{OFFTDC and ONTDC for Baird's counterexample: $q$=.001}
    \label{offvsonq001}
\end{minipage}
\end{figure}

Next we consider the '7-star' version of Baird's counter example from \cite[p~.17]{maeith} 
All the rewards in transitions are zero and true value function for
each state is zero. The value functions are approximated as 
$V(s)=2\theta(s)+\theta_0 $ $\forall s \in \{1,2 \ldots 6\}$ and 
$V(7)=\theta(7)+ 2\theta_0$. The behaviour policy is to 
choose the state $7$ with probability $q=\frac{1}{7}$ and choose 
uniformly states $1-6$ with probability $(1-q)=\frac{6}{7}$. 
The target policy is to choose the state $7 $ with probability 1. 
The step size chosen for this setting is $a=.005,b=.05$. The 
initial parameters are $\theta=(1,1,1,1,1,1,10,1)$ and $w=\mathbf{0}$. The results in this 
case are summarized in Figure \ref{baird_off}. 
\begin{figure}
 \centering
\includegraphics[width=3 in]{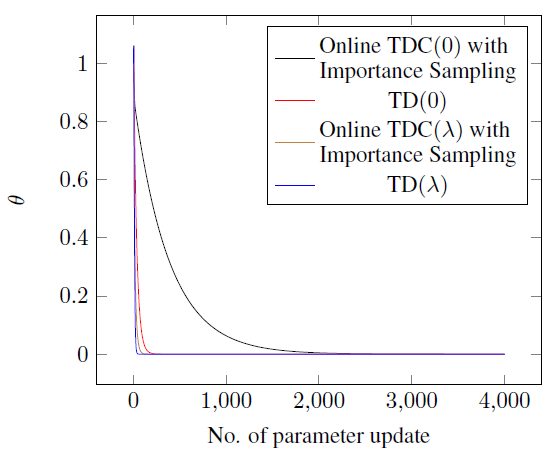}
\caption{on-policy learning on $\theta \to 2\theta$ problem}
\label{theta_on}   
\end{figure}

In both cases (Fig. \ref{theta_2} and \ref{baird_off}) ONTDC performs better than the OFFTDC. 
The difference becomes more apparent when
behaviour policy differs significantly from the target policy (Fig \ref{offvsonp01}, Fig \ref{offvsonp001},
Fig \ref{offvsonq01} and Fig \ref{offvsonq001}). The intuition is that in case of OFFTDC
the TD update is weighted by only step-size whereas in case of ONTDC it is 
additionally weighted by $\rho_n$. Therefore by changing the behaviour policy one can improve the rate of convergence 
of the algorithm. In the case  of on-policy learning for 
the $\theta \to 2\theta$ problem, Figure \ref{theta_on} shows that 
with eligibility traces the performance of ONTDC is much closer to $TD(\lambda)$ compared to 
the case with $\lambda=0$.

Although ONTDC uses importance weighting in its update, this is not importance sampling used in Monte-Carlo 
algorithms which is the source of high variance. Further, ONTDC 
does not have any follow-on trace like emphatic TD which has a high variance. 
We show in Fig \ref{varb} and Fig \ref{vart} that the variances of the performance metric for the 
ONTDC is negligible 
eventually for the two standard counterexamples.
\begin{figure}
\centering
  \begin{minipage}[b]{0.49\textwidth}
    \includegraphics[width=\textwidth]{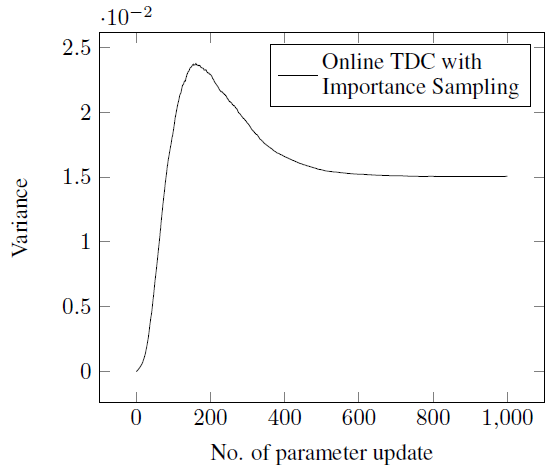}
    \caption{Variances of the performance metric for ONTDC: Baird's counterexample}
    \label{varb}
\end{minipage}
  \hfill
\begin{minipage}[b]{0.49\textwidth}
    \includegraphics[width=\textwidth]{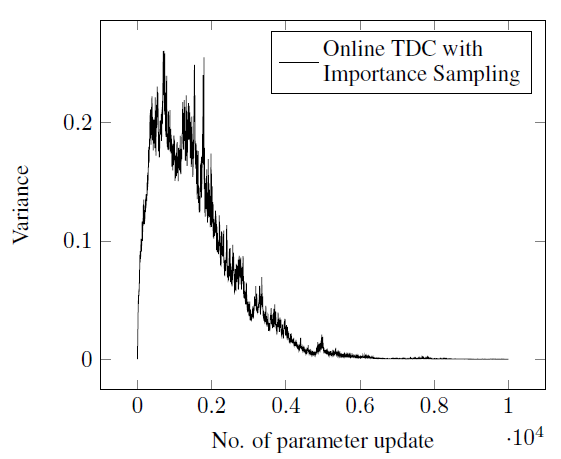}
    \caption{Variances of the performance metric for ONTDC: $\theta \to 2\theta$ problem}
    \label{vart}
\end{minipage}
\end{figure}


For both the aforementioned examples the results for the extension to eligibility traces (the algorithm is called GTD($\lambda$)
or TDC($\lambda$)) 
can be seen in Fig \ref{eb} and Fig \ref{et} with $\lambda =0.1$.

\begin{figure}
\centering
  \begin{minipage}[b]{0.49\textwidth}
    \includegraphics[width=\textwidth]{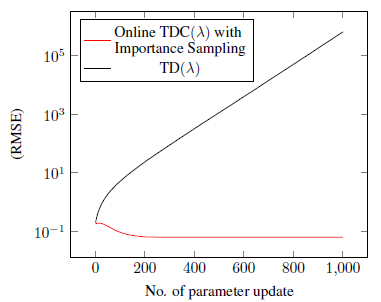}
    \caption{Comparison of TD($\lambda$) and TDC($\lambda$): Baird's counterexample}
    \label{eb}
\end{minipage}
  \hfill
\begin{minipage}[b]{0.49\textwidth}
    \includegraphics[width=\textwidth]{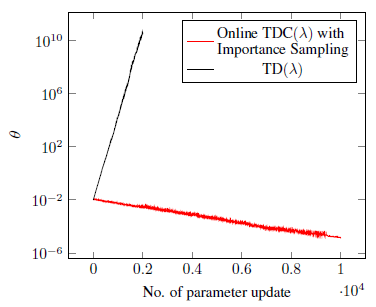}
    \caption{Comparison of TD($\lambda$) and TDC($\lambda$): $\theta \to 2\theta$ problem}
    \label{et}
\end{minipage}
\end{figure}


\begin{figure}
\centering
  \begin{minipage}[b]{0.49\textwidth}
    \includegraphics[width=\textwidth]{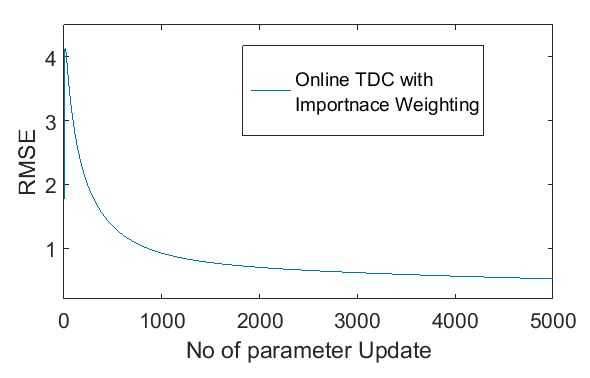}
    \caption{ONTDC with step size for Baird's counterexample $a(n)=\frac{.5}{n}$,
$b(n)=\frac{.125}{n^{.95}}$}
     \label{dimstepb}
\end{minipage}
  \hfill
\begin{minipage}[b]{0.49\textwidth}
    \includegraphics[width=\textwidth]{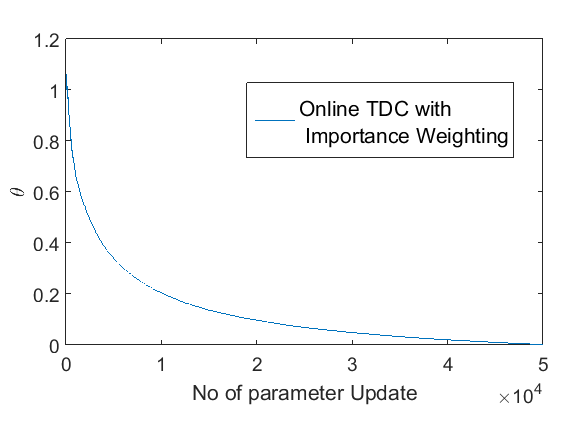}
    \caption{ONTDC with step size for $\theta \to 2\theta$ problem $a(n)=\frac{7}{n+100}$,
    \label{dimstept}
$b(n)=\frac{.5}{n^{.95}}$}
\end{minipage}
\end{figure}

Fig \ref{dimstepb} and Fig \ref{dimstept} shows the results of experiments  
where the step-size sequences obey the requirements in \textbf{(A5)}. We observe good convergence 
behaviour in this case that is also better when compared with the case of constant step-sizes as considered in this work.

\newpage
\section{Conclusion}
We presented almost sure 
convergence proof for an off-policy temporal difference learning algorithm that is also 
extendible to eligibility traces (for a sufficiently large 
range of $\lambda$) with linear function approximation under the assumption 
that the ``on-policy'' trajectory for a behaviour policy is only available. 
This has previously not been done to our knowledge. 

A future direction would be to similarly extend algorithms for off-policy control (\cite{greedygq}) to the more 
realistic settings as we consider in this chapter. 

Note that in this chapter we assume point-wise boundedness (also called the `stability') of the iterates to prove convergence of the stochastic 
approximation algorithms. However, finding sufficient verifiable conditions for this is very hard when the noise is Markov as 
well as when there are multiple timescales. In the next chapter, we compile several aspects
of the dynamics of stochastic approximation algorithms with Markov iterate-dependent noise
when the iterates are not known to be stable beforehand.


\chapter{Dynamics of stochastic approximation with
Markov iterate-dependent noise with iterate-stability
 not ensured}
\label{chap:lock}
\section{Brief Introduction and Organization}
As discussed in Chapter \ref{chap:introduction}, stability of the iterates is an important assumption to prove almost sure convergence of the stochastic approximation iterates. However, in the presence of Markov noise, it is very 
difficult to provide verifiable sufficient conditions for this assumption. In this chapter we investigate the dynamics 
of stochastic approximation with
Markov iterate-dependent noise when iterate-stability is
 not ensured by extending the lock-in probability framework of Borkar \cite{lock_in_original} to such recursions by using the classical Poisson equation \cite{metivier,met_prior}. We apply our results to the problem of tracking ability \cite{bhumesh,konda} which requires extending these results to two stochastic approximation iterates with different time-scales (the slower iterate is a single timescale stochastic approximation), where we can see that the coupled o.d.e has no attractor  as mentioned in \cite{off-policy}. Therefore we
need to consider two quantities describing difference (over compact time interval) between
algorithm and o.d.e., one for the coupled algorithm/o.d.e and another for the slower algorithm/o.d.e. This gives rise to a situation where the conditioning event in the martingale
concentration inequality will not belong to the first $\sigma$-field in the current collection of $\sigma$-
fields (unlike in case of single timescale stochastic approximation where the conditioning
event always belongs to the first $\sigma$-field in the current collection of $\sigma$-fields \cite[p~40]{borkar1}).
 
The organization of this chapter is as follows: Section \ref{sec_def} formally
defines the problem and provides background and assumptions.
Section \ref{main_res} shows our main \textit{lock-in probability}
results. 
Section \ref{a.s.conv} shows how to prove \textit{almost sure convergence} to a
\textit{local} attractor using our results along with \textit{asymptotic tightness}
of the iterates. Moreover, this section shows that stability of the iterates can be proved using our results.
Section \ref{track} analyzes the tracking ability of adaptive algorithms using our results.
Section \ref{sample} describes the results on sample complexity.
Finally, we conclude by providing
some future research directions.
\section{The Problem and Assumptions}
\label{sec_def}

In the following we describe the preliminaries and notation that we use in our proofs.
Most of the definitions and notation are from \cite{metivier, borkar1, sameer}. The notations used 
for ordinary differential equation is from \cite[Appendix 11.2]{borkar1}.
In the following we describe the lock-in probability settings based on the approach in \cite{metivier}.
The main idea is to assume existence of a solution to the Poisson equation (Assumption (M4) from Section III B of \cite{metivier}),
thus converting Markov iterate-dependent noise into a martingale difference sequence and additional additive errors. We refer the readers
to \cite[Part II,Chap. 2, Theorem 6]{benveniste}, \cite[Section III~D, Appendix A]{metivier} for details on
the existence and properties of solution of Poisson equation for a Markov iterate-dependent process.

In this work we prove 
almost sure  convergence for recursion (\ref{main_m}) 
without assuming stability of the iterates, however following the classic Poisson equation
model stated above where the assumptions are designed keeping in mind the stability of the iterates. 
To make up for this we need to strengthen 
some existing assumptions of \cite{metivier} (shown next), these are standard 
assumptions satisfied in application areas such as reinforcement learning.

Recall $h: \mathbb{R}^d \to \mathbb{R}^d$ from Section \ref{markov}.  Let $G \subset \mathbb{R}^d$ be open and let $V:G\to [0,\infty)$ be such that $\langle \nabla V, h \rangle : G \to \mathbb{R}$ is non-positive.
We shall assume as in \cite{borkar1}  that $H:=\{\theta:V(\theta) = 0\}$ is equal to the set $\{\theta: \langle \nabla V(\theta), h(\theta)\rangle=0\}$
and is a compact subset of $G$.
Thus $V$ is a strict Lyapunov function. Then $H$ is an asymptotically stable invariant set of the
differential equation $\dot{\theta}(t) = h(\theta(t))$.
Let there be an open set $B$ with compact closure such that
$H \subset B \subset \bar{B} \subset G$. In this setting, the lock-in probability is defined to be the probability that the sequence
$\{\theta_n\}$ is convergent to $H$, conditioned on the event that $\theta_{n_0}\in B$ for some $n_0$ sufficiently large.

Recall that,  Theorem 8 of
\cite[p.~37]{borkar1}, shows that for the case of martingale difference noise, $\mathbb P[\theta_n\rightarrow
H|\theta_{n_0}\in B] \geq 1 - O(e^{-\frac{1}{s(n_0)}}),$ where $s(n_0) :=
\sum_{m=n_0}^{\infty} a(m)^2$.
In this work we obtain these results when the noise is  Markov iterate-dependent under the following assumptions:
\begin{enumerate}[label=\textbf{(A\arabic*)}]
 \item $\limsup_{n\to \infty}\|Y_n\|< \bar{C}$ a.s. for some $\bar{C} > 0$.
This is stronger than   $\limsup_n E[\|Y_n\|^2] < \infty$ which is implied by 
\textbf{(M2)} of \cite{metivier}.
\item $\sup_{y}\|f(\theta,y)\| \leq K(1+\|\theta\|)$ for all $\theta$.
\begin{remark}
\textbf{(A2)} is a standard assumption satisfied in reinforcement learning scenarios as 
pointed in \cite[p~6]{lock_in_original}. Clearly, this is stronger than 
the hypothesis (F) on $f$ as mentioned in \cite[p. 143]{metivier}.
\end{remark}  
\item  The step-sizes $\{a(n)\}$ are non-increasing positive scalars satisfying
\begin{align}
\sum_n a(n) = \infty, \sum_{n}{a(n)}^2 < \infty.\nonumber
\end{align}

\item For every $\theta$, the Markov chain $\Pi_\theta$ has a unique invariant probability $\Gamma_\theta$. (\textbf{(M1)} from 
\cite{metivier}). 
Further, $h(\theta) = \int f(\theta,y)\Gamma_{\theta}(dy)$ is Lipschitz continuous in $\theta$ with Lipschitz constant
$0<L<\infty$.
\item $\|M_{n+1}\| \leq K'(1+\|\theta_n\|)$ a.s. $\forall n$.
\item For every $\theta$ the Poisson equation
$$(1 - \Pi_{\theta})v_\theta = f(\theta,\cdot)-\int f(\theta,y)\Gamma_\theta(dy)$$
has a solution $v_\theta$. This is \textbf{(M4)} from \cite{metivier}.
\item For all $R>0$ there exist constants $C_R>0$ such
that
\begin{enumerate}
\item $\sup_{\|\theta\|\leq R} \|v_{\theta}(x)\| \leq C_R (1+\|x\|)$.
\item $\|v_{\theta}(x)-v_{\theta'}(x)\| \leq C_R \|\theta - \theta'\|(1+ \|x\|)$ for all $\|\theta\| \leq R$, $\|\theta'\| \leq R$.
\end{enumerate}
This is \textbf{(M5)b,c} from \cite{metivier}.
\end{enumerate}

Under the above assumptions we show that
\[
 \mathbb P\left[\theta_n\rightarrow H|\theta_{n_0}\in B\right] \geq 1 - O\left(e^{-\frac{c}{s(n_0)}}\right)
\]
 for sufficiently large $n_0$.

We provide a more detailed discussion on assumptions \textbf{(A1)} and \textbf{(A2)} as well as 
possible relaxations of these in Section \ref{discuss}.
\section{Lock-in probability calculation for single timescale stochastic approximation}
\label{main_res}

In this subsection we give a lower bound for $\mathbb P[\theta_n\rightarrow
H|\theta_{n_0}\in B]$ in terms of $s(n_0)$ when $n_0$ is sufficiently
large based on the settings described in Section \ref{sec_def}. How large $n_0$ needs to be will be specified soon.
Before proceeding further we describe our notations
and recall some known results. For $\delta > 0$, $N_{\delta}(A)$ for a set $A$ denotes its $\delta$-neighborhood $\{y : \|y - x\| < \delta \ \forall \ x \in A\}$.
Let $H^a=N_{a}(H)$. Fix some $0 < \epsilon_1 < \epsilon$
and $\delta_B>0$ such that $N_{\delta_B}(H^{\epsilon_1})\subset H^{\epsilon} \subset B$.

Let
$T$ be an upper bound 
for the time required for a solution of the o.d.e.\ (\ref{ode_m}) to reach the set
$H^{\epsilon_1}$, starting from an initial condition in $\bar{B}$. 
The existence of such a $T$ independent of the starting point in $\bar{B}$ can be proved 
using the continuity of flow of the o.d.e (\ref{ode_m}) around the initial point and the fact that $H$ is an
asymptotically stable set of the same o.d.e; see Lemma 1  of \cite[Chapter 3]{borkar1} for a similar proof.
 Let $t(n)= \sum_{m=0}^{n
- 1}a(m)$, $n \geq 1$ with $t(0)=0$. Let $n_0 \geq 0, n_m = \min \{n:t(n)\geq t(n_{m-1})+T\}$ and $T_m = t(n_m)$, $m \geq 1$.
Define $\bar \theta(t)$ by: $\bar{\theta}(t(n)) = \theta_n$, with linear
interpolation on $[t(n), t(n+1))$ for all $n$.
Let $\theta^{t(n_m)}(\cdot)$ be the solution of the limiting o.d.e.\ (\ref{ode_m}) on $[t(n_m), t(n_{m+1}))$ with
 the initial condition $\theta^{t(n_m)}(t(n_m))=\bar \theta (t(n_m)) =  \theta_{n_m}$.
  Let
 \[
 \rho_m := \sup_{t \in [t(n_m), t(n_{m+1}))} \|\bar \theta(t) - \theta^{t(n_m)}(t)\|.
 \]

 We recall here a few key results from \cite{lock_in_original}. As shown there, 
if $\theta_{n_0}\in B$, and $\rho_m < \delta_B$ for all $m \geq 0$, then $\bar{\theta}(T_n)$ is in $H^{\epsilon} \subset B$
for all $n \geq 1$. This follows from the following: because of the way we defined $T$, it follows that $\theta^{T_0}(T_1) \in H^{\epsilon_1}$ . Since $\rho_0 < \delta_B$ 
and $N_{\delta_B} (H^{\epsilon_1}) \subset H^{\epsilon}$ , $\bar{\theta}(T_1) \in H^{\epsilon}$. Since $H^{\epsilon}$ is a positively invariant subset
of $\bar{B}$, it follows that $\theta^{T_1} (\cdot)$ lies in $H^{\epsilon}$ on $I_1$ , and that $\theta^{T_1} (T_2) \in H^{\epsilon_1}$ . Hence
$\bar{\theta}(T_2) \in H^{\epsilon}$. Continuing in this way it follows that for all $m \geq 1$, $\theta^{T_m}(\cdot)$ lies
inside $H^{\epsilon}$ on $I_m$. Therefore using discrete Gronwall's inequality we can show that $\sup_{t \geq T_0} \bar{\theta}(t) < \infty$.
It is also known (\cite{metivier}, section IIC) that if the sequence of iterates
 $\{\theta_n\}$ remains bounded almost surely on a prescribed set of
 sample points, and if on this set the iterates belongs to a compact set in the domain of 
attraction of any local attractor infinitely often then it converges almost surely on this
 set to that local attractor.
Using this fact gives  the following estimate on the
 probability of convergence, conditioned on $\theta_{n_0}\in B$ (\cite{borkar1}, Lemma 1, p.\ 33):
 \[
 P\left[\bar \theta(t)\rightarrow H | \theta_{n_0}\in B\right] \geq
 P\left[\rho_m < \delta_B \ \forall m \geq 0 | \theta_{n_0}\in B
 \right].
 \]
Let $\mathcal{B}_m$ denote the event that $\theta_{n_0}\in B$ and
 $\rho_k < \delta_B$ for $k=0,1,\ldots,m$. Clearly, $\mathcal{B}_m \in \mathcal{F}_{n_{m+1}}$. The following lower bound
 for the above probability has been obtained in (\cite{borkar1}, Lemma 2, p.\ 33):
 \[
 P\left[\rho_m < \delta_B \ \forall m \geq 0 | \theta_{n_0}\in B
 \right]\geq 1-\sum_{m=0}^\infty P\left[\rho_m \geq \delta_B |\mathcal{B}_{m-1} \right].
 \]

Subsequently the idea is to find an upper bound of $\rho_m$ consisting of errors (asymptotically negligible on $\mathcal{B}_{m-1}$) 
as well as martingale terms. Then for some large $n_0$, one may bound 
$P(\rho_m \geq \delta_B|\mathcal{B}_{m-1})$ using a suitable martingale 
concentration inequality. In the following we describe how to achieve the above in our setting. 

Using the Poisson equation
one can write the recursion (\ref{main_m}) as
\small
\begin{align}
\theta_{n+1}= \theta_{n}+ a(n)h(\theta_n) + a(n)\left[v_{\theta_n}(Y_{n})-\Pi_{\theta_n}v_{\theta_n}(Y_{n})+M_{n+1}\right]\nonumber
\end{align}
\normalsize
where $\Pi_{\theta} \phi(x) = \int \phi(y)\Pi_{\theta}(x;dy)$.
Let $\zeta_{n+1} = v_{\theta_n}(Y_{n}) -\Pi_{\theta_n}v_{\theta_n}(Y_{n})$.
We decompose
\begin{align}
\begin{aligned}
\zeta_{n+1} = v_{\theta_n}(Y_{n+1}) -\Pi_{\theta_n}v_{\theta_n}(Y_{n}) + v_{\theta_n}(Y_{n}) - v_{\theta_{n+1}}(Y_{n+1})+ v_{\theta_{n+1}}(Y_{n+1})- v_{\theta_{n}}(Y_{n+1})\nonumber
\end{aligned}
\end{align}
and set
\begin{gather*}
A_n = \sum_{k=0}^{n-1}a(k)\zeta^{(1)}_{k+1}, \ B_n = \sum_{k=0}^{n-1}a(k)\zeta^{(2)}_{k+1}, \ C_n = \sum_{k=0}^{n-1}a(k)\zeta^{(3)}_{k+1}, \nonumber \\
D_n = \sum_{k=0}^{n-1}a(k)M_{k+1}, n \geq 1\nonumber
\end{gather*}
where
\begin{gather*}
\zeta^{(1)}_{n+1} =  v_{\theta_n}(Y_{n+1}) -\Pi_{\theta_n}v_{\theta_n}(Y_{n}), \ \zeta^{(2)}_{n+1} = v_{\theta_n}(Y_{n}) - v_{\theta_{n+1}}(Y_{n+1}),\nonumber \\
\zeta^{(3)}_{n+1} = v_{\theta_{n+1}}(Y_{n+1})- v_{\theta_{n}}(Y_{n+1}).
\end{gather*}

Then one can easily see that as in the proof of Lemma 3 of \cite[p.~34]{borkar1}
\small
\begin{align}
\label{rho}
\rho_m \leq & (C a(n_0) + K_T CLs(n_0)) + K_T [ \max_{n_m \leq j \leq n_{m+1}} \|A_j - A_{n_m}\| + \nonumber \\
            & \max_{n_m \leq j \leq n_{m+1}} \|B_j - B_{n_m}\|+  \max_{n_m \leq j \leq n_{m+1}} \|C_j - C_{n_m}\| + \nonumber \\ 
            & \max_{n_m \leq j \leq n_{m+1}} \|D_j - D_{n_m}\|],
\end{align}
\normalsize
where $C$ is a bound on $\|h(\Phi_t(\theta)\|$, with $\Phi_t$ the time-$t$ flow map for the o.d.e
(\ref{ode_m}), $0\leq t \leq T+1$ and $\theta \in \bar{B}$. Also, $K_T=e^{LT}$.

Choose an $n_0^{(1)}$ such that
\begin{align}
\label{1}
(C a(n_0^{(1)}) + K_T CLs(n_0^{(1)})) < \delta_B/2.
\end{align}




The following important lemma shows that $\forall m \geq 1$, on $\mathcal{B}_{m-1}$, iterates are stable over $T$-length interval with the stability 
constant independent of $m$. This is enough for our proofs to go through and justifies the importance of assumptions 
\textbf{(A2)} and \textbf{(A5)}. 
\begin{lemma}
\label{T_stability}
On $\mathcal{B}_{m-1}, \|\theta_j\| \leq K''$ for any $n_m \leq j \leq n_{m+1}$ where the constant $K'$ is independent of $m$.
\end{lemma}
\begin{proof}
From the definition of $\mathcal{B}_{m-1}$, we know that $\theta_{n_m} \in B$ on this event. Let $\|\theta_{n_m}\| \leq \tilde{C}$ $\forall m$.
Clearly, for $n_m \leq j \leq n_{m+1}$,
\begin{align}
 \begin{aligned}
  \|\theta_j\| &\leq \|\theta_{n_m}\| + \sum_{k=n_m}^{j-1} a(k)\left[\|f(\theta_k, Y_{k})\|+\|M_{k+1}\|\right]\nonumber \\
               &\leq \tilde{C} + \tilde{K} \sum_{k=n_m}^{j-1} a(k) (1+ \|\theta_k\|)
 \end{aligned}
\end{align}
where $\tilde{K} =\max(K,K')$.
As $\sum_{k=n_m}^{j-1} a(k) \leq T$, discrete Gronwall inequality gives the result.
\end{proof}

\begin{lemma}
For sufficiently large $n_m$,  $\max_{n_m \leq j \leq n_{m+1}} \|B_j - B_{n_m}\| < \frac{\delta_B}{8K_T}$ a.s.
on the event $\mathcal{B}_{m-1}$.
\end{lemma}
\begin{proof}
Now, if we write
$B_{n_m} = a(0)v_{\theta_0}(Y_0) + \sum_{k=1}^{n_m-1}(a(k) - a(k-1))v_{\theta_k}(Y_k) - a(n_m-1) v_{\theta_{n_m}}(Y_{n_m})$
we obtain
\begin{align}
B_j - B_{n_m} = &\sum_{k=n_m}^{j-1}(a(k) - a(k-1))v_{\theta_k}(Y_k) +\nonumber \\
 &a(n_m-1) v_{\theta_{n_m}}(Y_{n_m})-a(j-1) v_{\theta_j}(Y_j).\nonumber
\end{align}
As $\|\theta_i\| \leq K'$ on $\mathcal{B}_{m-1}$ 
\begin{align}
\|B_j - B_{n_m}\| \leq C_R \sum_{k=n_m}^{j-1}(a(k-1) - a(k))(1+\|Y_k\|) +  \nonumber \\
C_R\left[ a(n_m-1)(1+ \|Y_{n_m}\|) +  a(j-1)(1+ \|Y_{j}\|)\right]\nonumber
\end{align}
using \textbf{(A7a)}.
Now using \textbf{(A1)}, \textbf{(A3)}\footnote{This is the only place where the requirement that step size is non-increasing in \textbf{(A3)} is used.}
we see that
\begin{align}
\|B_j - B_{n_m}\| \leq 2C_R''a(n_m-1),\nonumber
\end{align}
for some $C_R'' >0$.
Now choose $n_0^{(2)}$ such that
\begin{align}
\label{2}
2C_R''a(n_0^{(2)}-1) < \frac{\delta_B}{8K_T}.
\end{align}
The claim follows $\forall n_m \geq n^{(2)}_0$.
\end{proof}

\begin{lemma}
For sufficiently large $n_m$,  $\max_{n_m \leq j \leq n_{m+1}} \|C_j - C_{n_m}\| < \frac{\delta_B}{8K_T}$ a.s. 
on the event $\mathcal{B}_{m-1}$.
\end{lemma}
\begin{proof}
Using \textbf{(A7b)} we see that
\begin{align}
\|\zeta^{(3)}_{k+1}\| \leq C_R\|\theta_{k+1} - \theta_k\|(1+\|Y_{k+1}\|).\nonumber
\end{align}
Again using the stability of the iterates in the $T$ length interval on $\mathcal{B}_{m-1}$ and the assumptions \textbf{(A1)} and
\textbf{(A2)}
we see that \begin{align}
\|\zeta^{(3)}_{k+1}\| \leq C_R\tilde{K}\bar{C}a(k).\nonumber
\end{align}
Therefore \begin{align}
\|C_j - C_{n_m}\| \leq C_R\tilde{K}\bar{C} \sum_{k=n_m}^{j-1} a(k)^2. \nonumber
\end{align}
Now choose $n_0^{(3)}$ such that
\begin{align}
\label{3}
C_R \tilde{K}\bar{C} \sum_{k=n_0^{(3)}}^{j-1} a(k)^2 < \frac{\delta_B}{8K_T}.
\end{align} This is possible 
due to  \textbf{(A3)}. The claim follows for $n_m \geq n^{(3)}_0$.
\end{proof}
\begin{theorem}
\label{main_thm}
Under \textbf{(A1)} - \textbf{(A7)}, for $n_0$ sufficiently large,
\begin{align}
P(\bar{\theta}(t) \to H| \theta_{n_0} \in B) \geq 1-2de^{-\frac{\hat{K}\delta_B^2}{ds(n_0)}}- 2d e^{-\frac{\hat{C}\delta_B^2}{ds(n_0)}}.\nonumber
\end{align}
\end{theorem}
\begin{proof}
Set
\begin{align}
\label{large}
n_0 = \max (n_0^{(1)}, n_0^{(2)}, n_0^{(3)}).
\end{align}
From (\ref{rho}) we see that
for this (large) $n_0$
\small
\begin{align}
P(\rho_m \geq \delta_B |& \mathcal{B}_{m-1}) \leq P(\max_{n_m \leq j \leq n_{m+1}} \|A_j - A_{n_m}\|>\frac{\delta_B}{8K_T}|\mathcal{B}_{m-1})  \nonumber \\
                        & + P(\max_{n_m \leq j \leq n_{m+1}} \|D_j - D_{n_m}\|>\frac{\delta_B}{8K_T}|\mathcal{B}_{m-1}). \nonumber
\end{align}
\normalsize
Again using the stability of the iterates in the $T$ length interval on $\mathcal{B}_{m-1}$ and assumption
\textbf{(A7a)} we see that $\zeta^{(1)}_{k+1}$ is bounded a.s. on $\mathcal{B}_{m-1}$ by
the constant $C_0=2C_R(1+\bar{C})$ for $n_m \leq k \leq j-1$. Therefore each of the components in
this vector is also bounded by the same constant.
Therefore,
\small
\begin{align}
P(&\max_{n_m \leq j \leq n_{m+1}}  \|A_j - A_{n_m}\|>\delta_B/8K_T|\mathcal{B}_{m-1}) \nonumber \\
&\leq P(\max_{n_m \leq j \leq n_{m+1}} \|A_j-A_{n_m}\|_{\infty} > \frac{\delta_B}{8K_T\sqrt{d}}|\mathcal{B}_{m-1})\nonumber \\
&=P(\max_{n_m \leq j \leq n_{m+1}} \max_{1\leq i \leq d} |A^i_j - A^i_{n_m}| > \frac{\delta_B}{8K_T\sqrt{d}} |\mathcal{B}_{m-1})\nonumber \\
&=P(\max_{1\leq i \leq d} \max_{n_m \leq j \leq n_{m+1}}  |A^i_j - A^i_{n_m}| > \frac{\delta_B}{8K_T\sqrt{d}} |\mathcal{B}_{m-1})\nonumber \\
&\leq \sum_{i=1}^{d}P(\max_{n_m \leq j \leq n_{m+1}}  |A^i_j - A^i_{n_m}| > \frac{\delta_B}{8K_T\sqrt{d}} |\mathcal{B}_{m-1})\nonumber \\
&\leq \sum_{i=1}^{d}2\exp\{-\frac{\delta_B^2}{32K_T^2dC_0^2(\sum_{j=n_m}^{n_{m+1}}a(j)^2)}\}\nonumber \\
&\leq 2d \exp\{-\frac{\delta_B^2}{32K_T^2dC_0^2(\sum_{j=n_m}^{n_{m+1}}a(j)^2)}\}\nonumber \\
&= 2d \exp\{-\frac{\delta_B^2}{32K_T^2dC_0^2[s(n_m)-s(n_{m+1})]}\}\nonumber
\end{align}
\normalsize
In the third inequality above we use the conditional version of the martingale concentration inequality \cite[p.~39, chap.~4]{borkar1}.
We give a proof outline of it in Appendix (Chapter \ref{appendix}).
Now it can be shown as in Theorem 11 of \cite[Chapter 4]{borkar1} that for sufficiently large $n_0$,
\begin{align}
P(\rho_m < \delta_B~\forall m\geq 0| \theta_{n_0} \in B) \geq 1-2de^{-\frac{\hat{K}\delta_B^2}{ds(n_0)}} - 2d e^{-\frac{\hat{C}\delta_B^2}{ds(n_0)}}\nonumber
\end{align}
where $\hat{K} = 1/32K_T^2C_0^2$ and $\hat{C}$ is same as in Theorem 11 \cite[p. 40]{borkar1}.
\end{proof}
\subsection{Discussion on the assumptions}
\label{discuss}
\subsubsection{$Y_n$ unbounded}
Even if $Y_n$ is unbounded and iterate-dependent our analysis will go through in the following case by creating 
functional dependency between $\{Y_n\}$ and $\{\theta_n\}$.
\begin{enumerate}[label=\textbf{(A\arabic*)'}]
\item For large $n$, $\ \lVert Y_{n+1} \rVert \le K_0 (1 + \lVert \theta_n \rVert) \text{ for some $0 < K_0 < \infty$}$.
\end{enumerate}


Accordingly we may replace \textbf{(A2)} by
the point-wise boundedness of $f$ \textit{i.e.,}
\begin{enumerate}[label=\textbf{(A\arabic*)'}]
\setcounter{enumi}{1}
\item $\ \|f(\theta, y)\| \le K(1 + \lVert \theta \rVert + \lVert y \rVert)$.
\end{enumerate}

\subsubsection{$Y_n$ point-wise bounded}
Our analysis will also go through (with the addition of an error term) 
for the following relaxation of \textbf{(A1)}:
\begin{enumerate}[label=\textbf{(A\arabic*)''}]
\item $\limsup_n \|Y_n\| < \infty \mbox{~~a.s.}.$ 
\end{enumerate}
In this case the lock-in probability statement in Theorem \ref{main_thm} will be as follows:
For $\nu > 0$, $n_0(\nu)$ sufficiently large,
\begin{align}
P(\bar{\theta}(t) \to H| \theta_{n_0} \in B) \geq 1-2de^{-\frac{\hat{K}(\nu)\delta_B^2}{ds(n_0)}}- 2d e^{-\frac{\hat{C}(\nu)\delta_B^2}{ds(n_0)}} - 2\nu.\nonumber
\end{align}The proof will work by selecting a large compact set $C(\nu)$ s.t. $P(\limsup_n \|Y_n\| < C(\nu)) > 1- \nu$ and doing the 
same calculation as in Section \ref{main_res} on this set with probability at least $1-\nu$. 

\section{Proof of almost sure convergence}
\label{a.s.conv}
\subsection{Almost sure convergence under asymptotic tightness}
\begin{mydef}
A sequence of random variables $\{\theta_n\}$ is called asymptotically tight if for each $\epsilon>0$
there exists a compact set $K_{\epsilon}$ such that
\begin{align}
\label{tight}
\limsup_{n \to \infty} P(\theta_n \in K_{\epsilon}) \geq 1-\epsilon.
\end{align}
\end{mydef}
Clearly, (\ref{tight}) is a much weaker condition than (\ref{stab}).
In the following, we give a sufficient condition to guarantee the above:
\begin{lemma}
If there is a $\phi \geq 0$ so that
$\phi(\theta) \to \infty$ as $\|\theta\| \to \infty$
and
\begin{align}
\label{tight_condn}
\liminf_{n \to \infty} E[\phi(\theta_n)] < \infty,
\end{align}
then $\{\theta_n\}$ is asymptotically tight.
\end{lemma}
\begin{proof}
Proof by contradiction and similar to the proof of sufficient condition for full tightness as given in Theorem 3.2.8 of \cite[p.~104]{durrett}.
\end{proof}
Next, we show that if the stochastic approximation iterates are asymptotically
tight then we can prove almost sure convergence to $H$ under some reasonable assumptions.
\begin{theorem}
\label{a.s.}
Under \textbf{(A1)-(A7)}, if $\{\theta_n\}$ is asymptotically tight and  $\liminf_n P(\theta_n \in G) =1$ then $P(\theta_n \to H) =1$.
\end{theorem}
\begin{proof}
Choose an open $B$ with compact closure such that $H, K_{\epsilon} \cap G \subset B \subset \bar{B}\subset G$.
Therefore
\small
\begin{align}
&\limsup_{n_0 \to \infty}P(\theta_{n_0} \in B)\nonumber \\
                  &\geq \limsup_{n_0\to \infty}P(\theta_{n_0} \in G \cap K_{\epsilon})\nonumber \\
                  &= \limsup_{n_0\to \infty}\left[P(\theta_{n_0} \in K_{\epsilon}) + P(\theta_{n_0} \in G) - P(\theta_{n_0} \in G\cup K_{\epsilon})\right]\nonumber \\
                  &\geq \limsup_{n_0 \to \infty} P(\theta_{n_0} \in K_{\epsilon}) + \liminf_{n_0 \to \infty}P(\theta_{n_0} \in G) - \limsup_{n_0 \to \infty} P(\theta_{n_0} \in G\cup K_{\epsilon})\nonumber \\
                  &\geq 1 - \epsilon +1 - 1\nonumber.
\end{align}
\normalsize
Thus there exists a subsequence $n_0(k)$ s.t. $P(\theta_{n_0(k)} \in B)>0$. Now, 
\small
\begin{align}
&\lim_{k \to \infty} P(\theta_{n_0(k)} \in B, \theta_n \to H) = \lim_{k \to \infty} P(\theta_{n_0(k)} \in B) P(\theta_n \to H | \theta_{n_0(k)} \in B) \nonumber \\
&= \lim_{k \to \infty} P(\theta_{n_0(k)} \in B) \mbox{~~using Theorem \ref{main_thm}}\nonumber
\end{align}
\normalsize
Therefore, 
\begin{align}
P(\theta_n \to H) &\geq \limsup_{n_0 \to \infty} P(\theta_n \to H, \theta_{n_0} \in B) \nonumber \\
                  &= \limsup_{n_0 \to \infty} P(\theta_{n_0} \in B) \geq 1-\epsilon. \nonumber  
\end{align}
\normalsize 
Now let $\epsilon \to 0$.
\end{proof}
\begin{remark}
We compare Theorem \ref{a.s.} to the main convergence result (Kushner-Clark Lemma) from \cite[Section II~C]{benveniste}
where stability of the iterates was assumed. Note that in that case much weaker condition, namely $\theta_n \in A$ infinitely often
where $A$ is some compact subset of $G$ was sufficient to draw
the conclusion. Here we need a much stronger condition such as $\liminf_{n_0}P(\theta_{n_0} \in G)=1$.
\end{remark}
\begin{remark}
Theorem  \ref{a.s.} is valid for any `local' attractor $H$ whereas in \cite[Theorem 1]{sameer} $H$ was a `global' attractor.
\end{remark}

There are sufficient conditions to guarantee tightness (\cite[Chapter 6, Theorem 7.4]{kushner}) of the iterates in literature. 
In the following we describe another set of sufficient conditions which guarantee (\ref{tight_condn}):
\begin{lemma}
Suppose there exists a $\phi \geq 0$ and $\phi(\theta) \to \infty$ as $\|\theta\| \to \infty$ with the following properties:
Outside the unit ball
\begin{enumerate}[label=\textbf{(S\arabic*)}]
 \item $\phi$ is twice differentiable and all second order derivatives are bounded
by some constant $c$.
 \item for every $\theta$, $K \subset \mathbb{R}^k$ compact, $\langle \nabla \phi(\theta), f(\theta,y)\rangle \leq 0$ for all $y\in K$.
\end{enumerate}
 Then for the step size sequences of the form $a(n)=\frac{1}{n(\log n)^p}$ with $0< p \leq 1$, we have (\ref{tight_condn}).
\end{lemma}



\begin{proof}
Following similar steps as in \cite[Theorem 3]{sameer} and \textbf{(S2)} we get
\begin{align}
\label{sameer}
E[\phi(\theta_{n+1})|\mathcal{F}_n] \leq \phi(\theta_n) +ca(n)^2(1+\|\theta_n\|^2)\mbox{~a.s.}.
\end{align}
Now we know that, for $n \geq 1$
\begin{align}
 \begin{aligned}
  \|\theta_n\| &\leq \|\theta_{0}\| + \sum_{k=0}^{n-1} a(k)\left[\|f(\theta_k, Y_{k})\|+\|M_{k+1}\|\right]\nonumber \\
               &\leq \|\theta_{0}\| + \tilde{K} \sum_{k=0}^{n-1} a(k) + \tilde{K} \sum_{k=0}^{n-1} a(k)\|\theta_k\|.
 \end{aligned}
\end{align}
Therefore using a general version of discrete Gronwall inequality (See Appendix i.e. Chapter \ref{appendix})
and the fact that
$\|\theta_{0}\| + \tilde{K} \sum_{k=0}^{n-1} a(k)$ is an increasing function of $n$, we get that
\begin{align}
\|\theta_n\| \leq  \left[ \|\theta_{0}\| + \tilde{K} \sum_{k=0}^{n-1} a(k)\right]\exp(\tilde{K} \sum_{k=0}^{n-1} a(k)). \nonumber
\end{align}
Therefore
\begin{align}
\label{series}
\nonumber
\liminf_n &E[\phi(\theta_n)] < \phi(\theta_0) +ca(0)^2(1+\|\theta_0\|^2) + c\sum_{n=1}^{\infty} a(n)^2 + \\
&c\sum_{n=1}^{\infty} a(n)^2 [\|\theta_{0}\| + \tilde{K} \sum_{k=0}^{n-1} a(k)]^2\exp(2\tilde{K} \sum_{k=0}^{n-1} a(k)).
\end{align}
In the following, we show that for the mentioned step-size sequence the R.H.S converges.
Assume $0< p < 1$.
Then

\begin{align}
\sum_{i=2}^{n-1}a(i) &\leq \int_{1}^{n-1} \frac{1}{i(logi)^p}di \leq \frac{1}{1-p}(\log n)^{1-p}.\nonumber
\end{align}
Then,
\begin{align}
\sum_{n=2}^{\infty} \frac{(\log n)^{2(1-p)}}{n^2(\log n)^{2p}} \exp\left[\frac{2\tilde{K}}{1-p}(\log n)^{1-p}\right] = \sum_{n=2}^{\infty} \frac{(\log n)^{2-4p}}{n^{2+\frac{2\tilde{K}}{(p-1)(\log n)^p}}}.\nonumber
\end{align}
This is a convergent series for $0<p<1$ as there exists an $\epsilon >0$ such that for large $n$
\begin{align}
(\log n)^{2-4p} \leq n^{1- \frac{2\tilde{K}}{(1-p)(\log n)^p}-\epsilon}.\nonumber
\end{align}
Also, the following series converges
\begin{align}
\sum_{n=2}^{\infty} \frac{\exp\left[\frac{2\tilde{K}}{1-p}(\log n)^{1-p}\right]}{n^2(\log n)^{2p}}.\nonumber \\
\end{align}
Moreover, it is easy to check that the above arguments also hold for $p=1$.
\end{proof}
Thus we show that \textbf{(A5)} in Theorem 3 in \cite{sameer} is not required for the step size sequence of the form
$a(n)=\frac{1}{n(\log n)^p}$ with $0<p\leq 1$ which is clearly a divergent series but $\sum_{n} a(n)^2 < \infty$.

\begin{remark}
Theorem 3 of \cite{sameer} imposes assumptions on the strict Lyapunov function $V(\cdot)$
for the attractor $H$ to ensure tightness of the iterates.
For that reason $H$ is required to be a global attractor there. 
However, we observe that $\phi(\cdot)$ can be different from $V(\cdot)$ because
we only require properties like \textbf{(S2)} to ensure tightness of the iterates.
\end{remark}
\begin{remark}
Note that the series in R.H.S of (\ref{series}) won't converge if $a(n) = \frac{1}{n^k}$ with $1/2 < k \leq 1$.
In such a case \textbf{(A5)} from  \cite{sameer} will be required.
\end{remark}

\subsection{Proof of stability and a.s. convergence using our results}
Note that if the iterates belong to some arbitrary compact set (depending on the sample point) infinitely often,
it may not imply stability if the time interval between successively visiting it runs to infinity.
 We show that this does not happen if the compact set and the step-size have special properties.
Using the lock-in probability results from Section \ref{main_res}, we prove stability and therefore convergence
of the iterates
on the set $\{\theta_n \in B \mbox{~i.o.}\}$ when the step-size is $a(n) = \frac{1}{n^k}, \frac{1}{2} <k \leq 1$.

Consider the settings described in Section \ref{main_res}.
Let $A = \{\omega: \exists m \geq 0 \mbox{~~s.t.~~} \rho_m(\omega) \geq \delta\}$.
Then Theorem \ref{main_thm} shows that for sufficiently large $n_0$,
\begin{align}
P(A | \theta_{n_0}\in B) < 4de^{-\frac{C}{s(n_0)}} \nonumber \\
\implies P(A \cap \{\theta_{n_0}\in B\}) < 4de^{-\frac{C}{s(n_0)}} \nonumber \\
\implies \sum_{n_0 =1}^{\infty} P(A \cap \{\theta_{n_0}\in B\}) < \sum_{n_0 =1}^{\infty} 4de^{-\frac{C}{s(n_0)}}.\label{rhs}
\end{align}
Now, for $n \geq 2$
\begin{align}
s(n) = \sum_{i=n}^{\infty} \frac{1}{i^{2k}} &< \int_{i=n-1}^{\infty} \frac{1}{i^{2k}} di = \frac{1}{(2k-1)(n-1)^{2k-1}} \nonumber \\
                                         &\leq \frac{1}{(2k-1)(\frac{n}{2})^{2k-1}}\nonumber
\end{align}
Now, for large $n$, $e^{(2k-1) (\frac{n}{2})^{2k-1}} > n^2.$
Therefore R.H.S in (\ref{rhs}) is finite for the mentioned step-size. The same argument follows for the
step-size $\frac{1}{n (\log n)^{k}}, k\leq 1$ as for large $n$, $(\log n) ^{2k} \geq 1$.
Therefore,
\begin{align}
E[\sum_{n_0 =1}^{\infty} I_{A \cap \{\theta_{n_0}\in B\}}] < \infty \implies  I_{A} \sum_{n_0 =1}^{\infty}  I_{\{\theta_{n_0}\in B\}} < \infty \mbox{~a.s.}\nonumber
\end{align}
Therefore on the event $\{\theta_{n_0}\in B \mbox{~i.o}\}$, $I_A =0 \mbox{~a.s.}$ which is nothing but
$\sup_n \|\theta_n\| < \infty$ a.s. The result can be summarized as follows:
\begin{corollary}
Under the assumptions made in Section \ref{sec_def} and the following assumptions:
\begin{enumerate}[label=\textbf{(W\arabic*)}]
 \item $\forall N~~~\exists n \geq N$ s.t. $P(\theta_n \in B) > 0$ where $B$ is chosen as in Section \ref{sec_def}, 
 \item $\sum_{n=1}^{\infty} P(\theta_n \in B | \mathcal{F}_{n-1}) = \infty \mbox{~~a.s.}$,
\end{enumerate}
we have 
\begin{align}
\sup_n \|\theta_n\| < \infty \mbox{~~ a.s. and ~~}  \theta_n \to H \mbox{~~ a.s.} \nonumber
\end{align}
for the step-size sequence of the form $a(n)=\frac{1}{n^k}, 0.5 < k \leq 1$ and $\frac{1}{n(logn)^k}, k \leq 1$. 
\end{corollary}

\section{On the tracking ability of ``general'' adaptive algorithms using lock-in probability}
\label{track}
In this section we investigate the tracking ability of algorithms of the type:
\begin{align}
\label{fast_alg}
w_{n+1} &= w_n + b(n)\left[g(\theta_n, w_n, Z^{(2)}_n) + M^{(2)}_{n+1}\right],
\end{align}
that are driven by a ``slowly'' varying single timescale stochastic approximation process:
\begin{align}
\label{slow_process}
\theta_{n+1} = \theta_n + a(n)\left[h(\theta_n,Z^{(1)}_n) + M^{(1)}_{n+1}\right], 
\end{align}
when none of the iterates are known to be stable. Here, 
$\theta_n \in \mathbb{R}^d, w_n \in \mathbb{R}^k, Z^{(1)}_n \in \mathbb{R}^l, Z^{(2)}_n \in \mathbb{R}^m$.
Note that there is a unilateral coupling between (\ref{fast_alg}) and (\ref{slow_process}) in that 
(\ref{fast_alg}) depends on (\ref{slow_process}) but not the other way.
Suppose $w_n$ converges to a function $\lambda(\theta)$ in case $\theta_n$ is kept constant at $\theta$, then an 
interesting  question is that if $\theta_n$ changes slowly 
can $w_n$ track the changes in $\theta_n$ i.e. what can we say about the quantity $\|w_n-\lambda(\theta_n)\|$ in the limit.
As mentioned in \cite{konda} such algorithms may arise in the context of adaptive algorithms.
However, in that work tracking was proved under the restrictive assumption that the stochastic approximation driven by the slowly varying process 
is linear (see (1) in the same paper) and the underlying Markov process in the faster iterate 
is driven by only the slow iterate. 
Using the lock-in probability results of Section \ref{main_res} 
we prove convergence as well as tracking ability of much general algorithms such as 
(\ref{fast_alg})-(\ref{slow_process}) under the following assumptions (we also 
give a detailed comparison with  the assumptions of \cite{konda}):
\begin{enumerate}[label=\textbf{(B\arabic*)}]
 \item $h, Z^{(1)}_n$ and $M^{(1)}_{n+1}$ satisfy the same assumptions satisfied by  
similar quantities ($f, Y_n, M_n$ respectively) of Section \ref{sec_def}.
$g$ satisfies the following assumption:
$\sup_{z}\|g(\theta,w,z)\| \leq K_1 (1+\|\theta\|+ \|w\|+ \|z\|)$ for all $\theta,w,z$ where $K_1 >0$.
Additionally, $\hat{g}(\theta,w) = \int g(\theta, w,z)\Gamma^{(2)}_{\theta,w}(dz)$ is Lipschitz continuous, 
$\Gamma^{(2)}_{\theta,w}$ being the unique stationary distribution of $Z^{(2)}_n$ for a fixed $(\theta,w)$. 
\begin{remark} In 
(1) of \cite{konda}, the vector field in the faster iterate is linear in the faster iterate variable. Also, the slower
iterate is not a stochastic approximation iteration there.
\end{remark}
 \item $\{a(n)\}$ is as in \textbf{(A3)}. $\{b(n)\}$ satisfies the similar assumptions as $\{a(n)\}$. 
Additionally, $a(n) < b(n)$ for all $n$ and $\frac{a(n)}{b(n)} \to 0$. Also, $b(n) < 1$ for all $n$.
\begin{remark}
The latter is much 
weaker than Assumption 4 of \cite{konda}. 
\end{remark}
 \item The dynamics of $Z^{(2)}_n$ is specified by
\begin{align}
P(Z^{(2)}_{n+1} \in B |Z^{(2)}_m, \theta_m, w_m, m\leq n) = \int_{B} \Pi^{(2)}_{\theta_n, w_n}(Z^{(2)}_n; dz), \mbox{~a.s.} n\geq 0, \nonumber 
\end{align} 
for $B$ Borel in $\mathbb{R}^m$. 
Assumptions similar to \textbf{(A1)}, \textbf{(A4)}, \textbf{(A6)} and \textbf{(A7)} will be true in case of $Z^{(2)}_{n}$ also 
with the exception that now $\theta$ will be replaced by the tuple $(\theta,w)$. 
\begin{remark}
In \cite{konda}, 
the Markov process depends on only the slow parameter.
\end{remark}
 \item $\{M^{(i)}_n\}, i=1, 2$ are martingale difference sequences
w.r.t increasing $\sigma$-fields
\begin{align}
\mathcal{F}_n = \sigma(\theta_m, w_m, M^{(i)}_{m}, Z^{(i)}_m, m \leq n, i=1,2), n \geq 0,\nonumber 
\end{align}
where $M^{(2)}_n$ satisfies the following:
\begin{align}
\|M^{(2)}_{n+1}\| \leq K_2(1 + \|\theta_n\| + \|w_n\|), K_2>0.\nonumber 
\end{align} 
\begin{remark}
Our assumptions on martingale difference noise 
is stronger than the same in \cite{konda}(See Assumption 5). 
\end{remark}
\item The o.d.e
\begin{align}
\dot{w}(t) = \hat{g}(\theta,w(t)) \nonumber
\end{align}
has a global attractor $\lambda(\theta)$ with $\lambda: \mathbb{R}^d \to \mathbb{R}^k$ Lipschitz 
continuous.

The o.d.e 
\begin{align}
\label{slow}
\dot{\theta}(t)=\hat{h}(\theta(t)) 
\end{align}
has an asymptotically stable set $H^s$ with domain of attraction $G^s$ 
where $\hat{h}(\theta) = \linebreak \int h(\theta,y) \Gamma^{(1)}_\theta(dy)$ is Lipschitz continuous with 
$\Gamma^{(1)}_\theta$ same as $\Gamma_\theta$ in \textbf{(A4)}.

For every compact set $C_1 \subset \mathbb{R}^d$ the set 
$\{(\theta,\lambda(\theta)): \theta \in C_1\}$ is Lyapunov stable set of the coupled o.d.e. 
\begin{align}
\dot{w}(t) = \hat{g}(\theta(t),w(t)), \dot{\theta(t)} =0 \nonumber
\end{align}

\item The iterates $\{\theta_n,w_n\}$ are asymptotically tight (for which a sufficient condition is stated latter). 
\begin{remark}
In 
\cite{konda} one important step in the proof is the proof of the stability of the iterates.
\end{remark}
\end{enumerate}
Let there be an open set $B_1$ with compact closure such that $H^s \subset B_1 \subset \bar{B_1} \subset G^s$.
From the results of Section \ref{main_res}, we can find a $T^s$ such 
that any trajectory for the o.d.e (\ref{slow}) starting in $\bar{B_1}$ will be within some $\epsilon_1$ 
neighborhood of $H^s$ after time $T^s$. Let,  $S_1 = \left[\sup_{\theta \in \bar{B_1}}\|\theta\| +\tilde{K}\right] e^{\tilde{K}T^s}$ 
and $C_1=\{\theta: \|\theta\| \leq S_1\}$. 
Let there be an open set $B_2$ with compact closures such that 
$\lambda(C_1) \subset B_2 \subset \bar{B_2} \subset \mathbb{R}^k$. Choose $\delta_{B_1}$ 
in the same way $\delta_B$ is chosen in Section \ref{main_res}. Choose $\delta_{B_2}, 0<\epsilon''_1< \epsilon''$ such that 
$N_{\delta_{B_2} + \epsilon''_1}(\lambda(C_1)) \subset  N_{\epsilon''}(\lambda(C_1)) \subset B_2$. 
If the coupled o.d.e starts at a point such that its $\theta$ and $w$ co-ordinates are in $C_1$ and $\bar{B_2}$ 
respectively then as in Section \ref{main_res} one can find a $T^f >0$ (independent of the starting point) 
such that after that time 
the o.d.e will be in the $\epsilon''_1$ neighbourhood of $\{(\theta,\lambda(\theta)): \theta \in C_1\}$. 
Now, 
let $T^c = \max(T^f, T^s+1)$ and 
for $m \geq 1$ define,   
\begin{align}
n^c_0 = n^s_0 = n_0. \nonumber \\
t^c(n) = \sum_{i=0}^{n-1} b(i),  n^c_m = \min\{n: t^c(n) \geq t^c(n^c_{m-1}) + T^c\}. \nonumber \\
t^s(n) = \sum_{i=0}^{n-1} a(i),  n^s_m = \min\{n: t^s(n) \geq t^s(n^s_{m-1}) + T^s\}. \nonumber
\end{align}
Similarly, for $m \geq 0$ define 
\begin{align}
T^c_m = t^c(n^c_m), I^c_m = [T^c_m, T^c_{m+1}], \\
T^s_m = t^s(n^s_m), I^s_m = [T^s_m, T^s_{m+1}], \\
l_m = \max(k: t^s(n^s_k) \leq t^c(n^c_m)). \nonumber
\end{align}
Now define, 
\begin{align}
\rho^c_m:=\sup_{t\in I^c_m}\|\bar{\alpha}(t) - \alpha^{T^c_m}(t)\| \nonumber
\end{align}
where $\bar{\alpha}(\cdot)$ is the interpolated trajectory for the coupled iterate 
\begin{align}
\label{coupled}
\alpha_{n+1} = \alpha_n + b(n)\left[G(\alpha_n,Z^{(2)}_n) + \epsilon'_n +  M^{(4)}_{n+1}\right]
\end{align}
where $\alpha_n=(\theta_n, w_n), \epsilon_n = \frac{a(n)}{b(n)}h(\theta_n, Z^{(1)}_n)$ and $M^{(3)}_{n+1} = \frac{a(n)}{b(n)} M^{(1)}_{n+1}$ for $n\geq 0$. 
Let , $\alpha=(\theta,w) \in \mathbb{R}^{d+k}, G(\alpha,z)=(0, g(\alpha,z)), \epsilon'_n=(\epsilon_n, 0), 
M^{(4)}_{n+1}= (M^{(3)}_{n+1}, M^{(2)}_{n+1})$,
and $\alpha^{T^c_m}(\cdot)$ is the solution of the o.d.e 
\begin{align}
\dot{w}(t) = \hat{g}(\theta(t),w(t)), \dot{\theta}(t) = 0,\nonumber 
\end{align} on $I^c_m$ with the initial point $\alpha^{T^c_m}(T^c_m) = \bar{\alpha}(T^c_m)$.
Also, define 
\begin{align}
\rho^s_m:=\sup_{t\in I^s_m}\|\bar{\theta}(t) - \theta^{T^s_m}(t)\| \nonumber
\end{align}
where $\theta^{T^s_m}(\cdot)$ denotes the solution of the o.d.e (\ref{slow}) on $I^s_m$ with the initial point $\theta^{T^s_m}(T^s_m)=\bar{\theta}(T^s_m)$. 
Let us assume for the moment that $\theta_{n_0} \in B_1, w_{n_0} \in B_2$, and that $\rho^s_m < \delta_{B_1}$ and 
$\rho^c_m < \delta_{B_2}$ for all $m\geq 0$.  
 Then using similar arguments as in Section \ref{main_res}, one can show that 
$\sup_{t \geq T^c_0} (\bar{\theta}(t),\bar{w}(t))< \infty \mbox{~~a.s.}$. 
Further, $(\theta_n,w_n)$ infinitely often visits the 
compact set $C_1 \times \bar{B_2}$ which is in the domain of attraction  
$C_1 \times \mathbb{R}^d$ of the set $\{(\theta, \lambda(\theta)) :  \theta \in C_1\}$. Therefore,
\begin{align}
(\theta_n, w_n) \to \{(\theta, \lambda(\theta)) :  \theta \in \mathbb{R}^d\} \mbox{~~a.s.~~} \nonumber 
\end{align}This, in turn, implies that $\|w_n - \lambda(\theta_n)\| \to 0$ a.s. which implies that 
$(\theta_n, w_n) \to \bigcup_{\theta \in H^s} (\theta, \lambda(\theta))$. 
Let $\mathcal{B}^s_{m}$ denote the event that $\theta_{n_0} \in B_1, w_{n_0} \in B_2$ and $\rho^s_k < \delta_{B_1}$ for $k=0,1,\dots, 
m$. Also, let $\mathcal{B}'_{m,k}$ denote the event that  $\theta_{n_0} \in B_1, w_{n_0} \in B_2$, $\rho^c_j < \delta_{B_2}$ for $j=0,1,\dots, 
m$ and $\rho^s_j < \delta_{B_1}$ for $j=0,1,\dots,
k$.
Therefore, 
\begin{align}
&P((\theta_n,w_n) \to \bigcup_{\theta \in H^s} (\theta, \lambda(\theta))|\theta_{n_0} \in B_1, w_{n_0} \in B_2) \nonumber \\
&\geq  P\left[\rho^c_m < \delta_{B_2} \forall m \geq 0, \rho^s_m < \delta_{B_1} \forall m \geq 0  | \theta_{n_0}\in B_1, w_{n_0} \in B_2
 \right] \nonumber \\
& \geq P\left[\rho^s_m < \delta_{B_1} \forall m \geq 0 | \theta_{n_0}\in B_1, w_{n_0} \in B_2\right] P\left[\rho^c_m < \delta_{B_2} \forall m \geq 0 | \theta_{n_0}\in B_1, w_{n_0} \in B_2, \rho^s_m < \delta_{B_1} \forall m \geq 0\right] \nonumber \\
& \geq \left[1 - \sum_{m=0}^{\infty} P(\rho^s_m > \delta_{B_1}| \mathcal{B}^s_{m-1})\right] P\left[\rho^c_m < \delta_{B_2} \forall m \geq 0 | \theta_{n_0}\in B_1, w_{n_0} \in B_2, \rho^s_m < \delta_{B_1} \forall m \geq 0\right]\label{lockin}.
\end{align}
Now, using the simple fact that $P(A|BC) \leq \frac{P(A|B)}{P(C|B)}$,
\begin{align}
&P\left[\rho^c_m < \delta_{B_2} \forall m \geq 0 | \theta_{n_0}\in B_1, w_{n_0} \in B_2, \rho^s_m < \delta_{B_1} \forall m \geq 0\right] \nonumber \\
&\geq \left[1- \sum_{m=0}^{\infty}\frac{P(\rho^c_m > \delta_{B_2}|\mathcal{B}'_{m-1,l_m-1})}{P\left[\rho^s_k < \delta_{B_1} \forall k \geq l_m | 
\mathcal{B}'_{m-1,l_m-1}\right]}\right] \nonumber \\
&=\left[1- \sum_{m=0}^{\infty}\frac{P(\rho^c_m > \delta_{B_2}|\mathcal{B}'_{m-1,l_m-1})}{1- f(m)- g(m)}\right]\label{fastrho}
\end{align}
where $f(m) = P(\rho^s_{l_m} > \delta_{B_1}|\mathcal{B}'_{m-1,l_m-1})$ 
and $g(m) = \sum_{k=l_m+1}^{\infty} P\left[\rho^s_k > \delta_{B_1} | \mathcal{B}'_{m-1,k-1}\right].$
Clearly, $\mathcal{B}'_{m-1,l_m-1} \in \mathcal{F}_{n^c_m}$ and $\mathcal{B}'_{m-1,k-1} \in \mathcal{F}_{n^s_k}$ for all 
$k\geq l_m+1$. However, $\mathcal{B}'_{m-1,l_m-1} \notin \mathcal{F}_{l_m}$. 
Therefore, the tedious task is to calculate upper bound of $f(m)$. We describe the procedure in detail.
Now, due to the way $T^c$ is chosen 
\begin{align}
&f(m) \leq \frac{P(\rho^s_{l_m} > \delta_{B_1}|\mathcal{B}'_{m-2, l_m-1})}
{1- \frac{P(\rho^c_{m-1} > \delta_{B_2} | \mathcal{B}'_{m-2, l_{m-1}-1})}{
1- f(m-1) -
\sum_{k=l_{m-1}+1}^{l_m-1} h(k)}} \label{recurse} 
\end{align}
where $h(k) = P(\rho^s_{k} > \delta_{B_1} | \mathcal{B}'_{m-2,k-1})$.

Let $S_1(n_0) = \sum_{i=n_0}^{\infty} a(i)^2$ and $S_2(n_0) = \sum_{i=n_0}^{\infty} b(i)^2$.
From (\ref{recurse}) we can see that 
\begin{align}
f(m) \leq \frac{o(S_1(n_0))}{1- \frac{o(S_2(n_0))}{1-f(m-1) - o(S_1(n_0))}}. \nonumber  
\end{align}
 
One can recursively calculate the expression. At the bottom level one calculates the following expression:
\begin{align}
1-P(\rho^s_{l_{1}-1} > \delta_{B_1} |\mathcal{B}^s_{l_{1}-2})\nonumber 
\end{align}

Using the fact that $S_1(n_0) < S_2(n_0)$ we see from the above that for all $m \geq 0$, $f(m) \leq o(S_2(n_0))$. 
One can easily show using the technique of Section \ref{main_res} 
that for all $m \geq 0$, $g(m) \leq o(S_1(n_0))$. 
\begin{lemma}
Under \textbf{(B1)-(B6)}, for sufficiently large $n_0$,
\begin{align}
&P((\theta_n,w_n) \to \bigcup_{\theta \in H^s} (\theta, \lambda(\theta))|\theta_{n_0} \in B_1, w_{n_0} \in B_2) \nonumber \\ 
&\geq \left(1-o(S_1(n_0))\right)\left(1-\frac{o(S_2(n_0))}{1-o(S_1(n_0))-o(S_2(n_0))}\right)\nonumber
\end{align} 
\end{lemma}
\begin{remark}
For the case of Section \ref{sec_def} i.e. $\theta_n = \theta$ for all $n$, either 1) $a(n)=0$ or 2) $h(\theta_n, Z^{(1)}_{n}) + M^{(1)}_{n+1} =0$ for all $n$. 
Further, all the assumptions $(\mathbf{B1}) -  (\mathbf{B6})$
are satisfied and we can recover the results of Section \ref{main_res} by observing that either 1) $S_1(n_0) =0$ or 2) $M^{(1)}_n =0$ for all $n$ (follows 
from the fact that $\{M^{(1)}_n\}$ is a martingale difference sequence). 
\end{remark}

\begin{proof}[Proof Outline]
Handling the first term in the last inequality of  (\ref{lockin}) is exactly same as in Section \ref{main_res}. 
The numerator of the term inside the summation in (\ref{fastrho}) can also be handled 
in a similar manner except the fact that the additional error $\epsilon'_n$ in (\ref{coupled}) can be made negligible on $\mathcal{B}'_{m-1,l_m-1}$
using the stability of the iterates there  over 
$T^c$ length intervals (the latter can be proved as in Lemma \ref{T_stability}). $n_0$ will be the maximum of 
its versions arising to handle these two parts.  
\end{proof}

From this one can easily prove almost sure convergence under tightness
\begin{theorem}
	\label{adapt}
	Under \textbf{(B1)-(B6)}, if $\{\alpha_n\}$ is asymptotically tight and  $\liminf_n P(\theta_n \in G^s) =1$ then $P((\theta_n,w_n) \to \bigcup_{\theta \in H^s} (\theta, \lambda(\theta))) =1$ i.e.  
	 $\|w_n - \lambda(\theta_n)\| \to 0$ a.s.
\end{theorem}

The sufficient conditions for tightness can be derived in the exact
similar way as in Section \ref{a.s.conv}. 
\begin{lemma}
Suppose there exists a $V': \mathbb{R}^{d+k} \to [0,\infty)$ and $V'(\alpha) \to \infty$ as $\|\alpha\| \to \infty$ with the following properties:
Outside the unit ball
\begin{enumerate}[label=\textbf{(S'\arabic*)}]
 \item $V'$ is twice differentiable and all second order derivatives are bounded
by some constant $c$.
 \item for every $\alpha$, $K \subset \mathbb{R}^l$ compact, $\langle {(\nabla V'(\alpha))}_{1 \dots d}, h({(\alpha)}_{1 \dots d},z)\rangle \leq 0$ for all $z\in K$.
 \item for every $\alpha$, $K \subset \mathbb{R}^m$ compact, $\langle {(\nabla V'(\alpha))}_{d+1 \dots d+k}, g(\alpha,z)\rangle \leq 0$ for all $z\in K$.
\end{enumerate}
where the notation $(v)_{m \dots n}$ is the vector $(v_m, \dots, v_n)$ with $v =(v_1,v_2,\dots,v_{d+k}) \in \mathbb{R}^{d+k}$.

Then for the step size sequences of the form $b(n)=\frac{1}{n(\log n)^p}$ with $0< p \leq 1$, the iterate 
$\{\alpha_n\}$ is asymptotically tight.
\end{lemma}

\section{Sample Complexity}
\label{sample}
It is easy to check that using the results in the previous section one can get a similar
probability estimate for sample complexity as in \cite[Chapter 4, Corollary 14]{borkar1}. 
Note that here $T$ can be any positive real
number unlike in the lock-in probability calculation where we need to choose $T$ appropriately. Therefore
we can extend the sample complexity calculation for stochastic fixed point point iteration
in the setting of  Markov iterate-dependent noise as follows:

Consider the example as shown in \cite[p.~43]{borkar1}. Let $u(\theta) = \int f(\theta, y)\Gamma_\theta(dy)$ with
$u$ being a contraction, so that $\|u(\theta) - u(\theta')\| < \alpha \|\theta-\theta'\|$ for some $\alpha \in (0,1)$.
$\theta^*$ be the unique fixed point of $u(\cdot)$. Let $T>0$.  $B$ can be chosen to be  $\{\theta: \|\theta - \theta^*\| < r\}$
with $r \geq \frac{3\epsilon}{2}$. For the analysis next choose $r=\frac{3\epsilon}{2}$.
Therefore the sample complexity estimate can be stated as follows:
\begin{corollary}
Let a desired accuracy $\epsilon >0$ and confidence $0<\gamma<1$ be given. Let $\bar{\theta}$ be the value at iteration $n_0$
with $n_0$ satisfying:
\begin{enumerate}
 \label{b_n}
 \item $n_0$ sufficiently large as in (\ref{large}), $s(n_0) < \frac{\hat{C}\epsilon^2}{4}$ and
$a(n_0) < \frac{\hat{C}\epsilon^2}{4}$ (Theorem 11 of \cite[Chapter 4]{borkar1}).
 \item $s(n_0) < \frac{c\epsilon^2}{\ln(\frac{4d}{\gamma})}.$
\end{enumerate}
Then on the event $\{\bar{\theta} \in B\}$, one needs
\begin{align}
N_0 := \min \left[n: \sum_{i=n_0 + 1}^{n}a(i) \geq \frac{(T+1)}{(1-e^{-(1-\alpha)T})}\right]-n_0\nonumber
\end{align}
more iterates to get within $2\epsilon$ of $\theta^*$ with probability at least $1-\gamma$.
\end{corollary}
\begin{remark}
The results clearly show large vs. small step-size trade-off for non-asymptotic rate of convergence
well-known in the stochastic convex optimization literature \cite{nemro}.
For large step-size, the algorithm will make fast progress whereas the errors due
to noise/discretization  will be much higher simultaneously.
However,
our results show the quantitative estimate of this progress and the error. For large
step-size, $n_0$ satisfying the hypothesis in Corollary 6.1 will be 
higher whereas $N_0$ will be lower compared to small-step size and the opposite is true
for small step-size. Therefore the optimal step-size should be somewhere in between.
\end{remark}

However, it is not possible to calculate accurately the threshold $n_0$ as the constants such as 
$C, \hat{K}$ depend on $B$ which
indeed depends on $\theta^*$. If we consider some special cases where the range for $\theta^*$
is given although the actual $\theta^*$ is unknown, we can replace the terms involving constants in (\ref{large}) by a
single constant $M$. For those
cases the following analysis will be useful.

In the following we state an upper bound $N'_0$ of
$N_0 + n_0$ when $a(n) = \frac{1}{n^k}, \frac{1}{2} < k < 1$ under the following crucial assumption:
\begin{enumerate}[label=\textbf{(T\arabic*)}]
 \item $P(\theta_{n_0} \in B) =1$.
\end{enumerate}

Let $\alpha = 0.9$. Under the assumptions made, the estimates of $n_0$ and $N'_0$ are
\footnotesize
\begin{align}
n_0  &= \max ((\frac{M}{\epsilon})^{\frac{1}{k}},\ (\frac{M}{\epsilon(2k-1)})^{\frac{1}{2k-1}},(\frac{M}{\epsilon^2(2k-1)})^{\frac{1}{2k-1}},(\frac{M}{\epsilon})^{\frac{2}{k}},\ (\frac{M(\ln (\frac{1}{\gamma}))}{\epsilon^2(2k-1)})^{\frac{1}{2k-1}},\ (\frac{2Mk}{\epsilon(2k-1)})^{\frac{1}{2k-1}}),\nonumber \\
N'_0 &= \left(\left(n_0\right)^{\left(1-k\right)}+ 15.16(1-k)\right)^{\frac{1}{1-k}}.\nonumber
\end{align}
\normalsize
Then from $N'_0$ onwards the iterates will be within $2\epsilon$ of $\theta^*$ with probability at least $1-\gamma$.
Note that the minimum value of the quantity $\frac{2(T+1)}{(1-e^{-(1-\alpha)T})}$ for $\alpha =0.9$ is  $15.16$.

To understand what should be the optimal step-size i.e. the value of $k$ for which $N'_0$ will be minimum, we plot $N'_0$ as
a function of $k$ for two different values of $M$ each with two different values of $\epsilon$ (Fig. \ref{fig1a}, Fig. \ref{fig1b}, Fig. \ref{fig2a} and Fig. \ref{fig2b}).

\begin{figure}
\centering
  \begin{minipage}[b]{0.49\textwidth}
    \includegraphics[width=\textwidth]{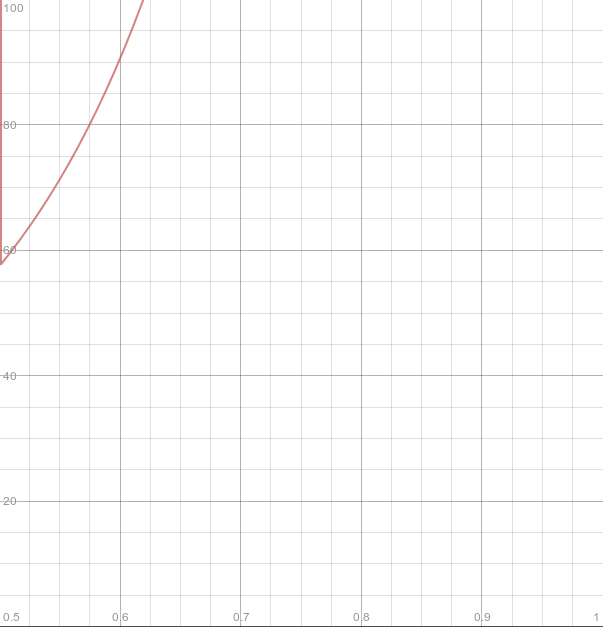}
    \caption{Sample complexity vs. step-size parameter; $y: N'_0, x: k, \gamma =0.1$, $M$ = 1E-07: $\epsilon =0.01$}
     \label{fig1a}
\end{minipage}
  \hfill
\begin{minipage}[b]{0.49\textwidth}
    \includegraphics[width=\textwidth]{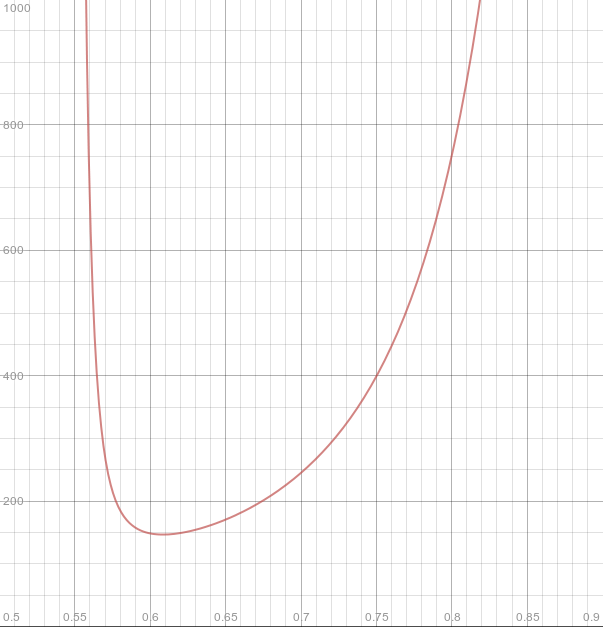}
    \caption{Sample complexity vs. step-size parameter; $y: N'_0, x: k, \gamma =0.1$, $M$ = 1E-07: $\epsilon =0.001$}
    \label{fig1b}
    \end{minipage}
\end{figure}

\begin{figure}
\centering
  \begin{minipage}[b]{0.49\textwidth}
    \includegraphics[width=\textwidth]{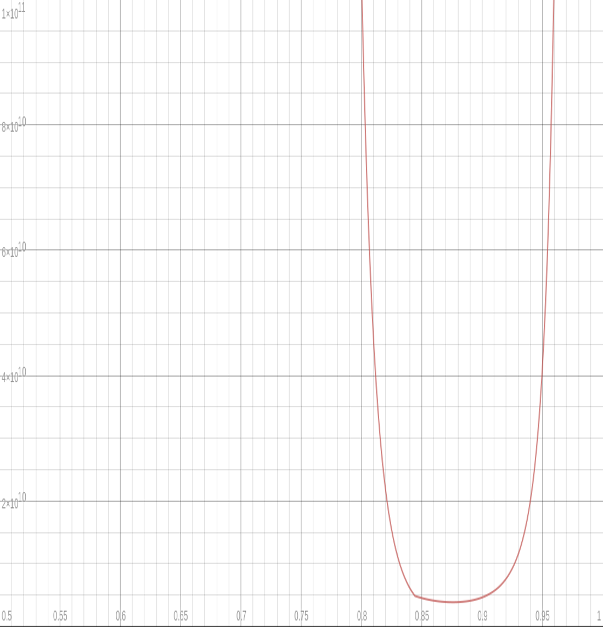}
    \caption{Sample complexity vs. step-size parameter; $y: N'_0, x: k, \gamma =0.1$, $M$ = 100: $\epsilon =0.01$}
     \label{fig2a}
\end{minipage}
  \hfill
\begin{minipage}[b]{0.49\textwidth}
    \includegraphics[width=\textwidth]{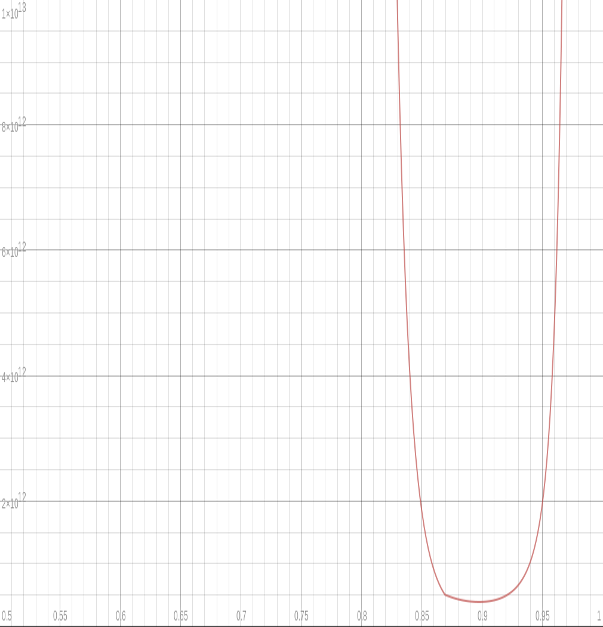}
    \caption{Sample complexity vs. step-size parameter; $y: N'_0, x: k, \gamma =0.1$, $M$ = 100: $\epsilon =0.001$}
    \label{fig2b}
    \end{minipage}
\end{figure}

\section{Conclusion}
In this chapter, we describe asymptotic and non-asymptotic convergence analysis of
stochastic approximation recursions with  Markov  iterate-dependent noise using the
lock-in probability framework.
Our results show that we are able to recover the same bound available for lock-in probability in the literature
for the case of i.i.d noise. Such results are used to calculate sample complexity estimate
of such stochastic approximation recursions which are then used for predicting the optimal step size.
Moreover, our results are extremely useful to prove almost sure convergence to specific attractors
in cases where asymptotic tightness
of the iterates can be proved easily.
An interesting future direction will be to extend this
analysis for two-timescale scenarios, both with and without Markov iterate-dependent noise.
We now provide a couple of appendices on a couple of results referred in the analysis.
\chapter{On the function approximation error for risk-sensitive reinforcement learning}
\label{chap:risk}





\section{Brief Introduction and Organization}
As described in Chapter \ref{chap:introduction}, risk-sensitive cost takes care of other moments except the average thus making it more realistic. In this chapter we provide several informative error bounds on function approximation error for the policy evaluation algorithm in this setting. The main idea is to use Bapat's inequlaity \cite{bapat} and to use Perron-Frobenius eigenvectors to get the new bounds.  With this new bounds we got better results compared to earlier spectral variation bound used in \cite{basu} in some cases. In our new bounds we use the irreducibility of the transition probability matrix whereas the earlier used spectral variation bound is true for any matrix. 

This chapter is organized as follows: Section \ref{rsec2} describes the preliminaries and background of the problem 
considered. Section \ref{rsec3} discusses the shortcomings of the bound proposed by \cite{basu}. Section \ref{rsec4} 
shows the theoretical conditions under which there is no error. This section also describes verifiable 
conditions when the transition kernel is doubly stochastic. Section \ref{rsec5} describes the new error bounds as well as how 
they compare with each other and with the state of the art bound. Section \ref{rsec6} presents conclusions and some future 
research directions.
\section{Preliminaries and Background}
\label{rsec2}

We begin by recalling the risk-sensitive framework. Consider an irreducible
aperiodic Markov chain $\{X_n\}$ on a finite state space $S = \{1, 2, \dots, s\}$, with transition matrix 
$P = [[p(j|i)]] i,j \in S$. 
While our real concern is a controlled Markov chain, we aim at a policy evaluation algorithm for a fixed
stationary policy. Thus we have suppressed the explicit control dependence. 
Let $c:S \times S \to \mathbb{R}$ denote a prescribed `running cost' function and $C$ be the $s \times s$ matrix
whose $(i,j)$-th 
entry is $e^{c(i,j)}$. The aim is to evaluate 
\begin{align}
\limsup_{n \to \infty} \frac{1}{n}\ln\left(E[e^{\sum_{m=0}^{n-1}c(X_m, X_{m+1})}]\right). \nonumber 
\end{align}That this limit exists follows from the
multiplicative ergodic theorem for Markov chains (see Theorem 1.2 of Balaji and Meyn (2000) \cite{bmrisk}, the
sufficient condition (4) therein is trivially verified for the finite state case here). Associated with
this is the multiplicative Poisson equation (see, e.g., Balaji and Meyn (2000) \cite[Theorem 1.2 (ii)]{bmrisk}):
We know from \cite{bmrisk} that there exists $\lambda >0$ and $V : S \to \mathbb{R^+}$ such that the  multiplicative Poisson 
equation holds as follows: 
\begin{align}
V(i) = \frac{\sum_j p(j|i)e^{c(i,j)}V(j)}{\lambda}. \nonumber 
\end{align}

For an explicit expression for $V(\cdot)$ see (5) in \cite{bmrisk}.

Thus $\lambda$ and  $V$ are respectively the Perron-Frobenius eigenvalue and eigenvector of the non-negative matrix 
$[[e^{c(i,j)}p(j|i)]]_{i,j \in S}$, whose existence is guaranteed by the Perron-Frobenius theorem (See Appendix i.e. 
Chapter \ref{appendix}). 
Furthermore,
under our irreducibility assumption, $V$ is specified uniquely up to a positive multiplicative scalar and $\lambda$ is
uniquely specified.
Also, the risk-sensitive 
cost defined as above
is $\ln\lambda$.

We know from \cite{bormn, borkarq,borkar_actor} that in case of both value iteration
and reinforcement learning algorithms based on value
iteration, the $i_0$-th component of the sequence of iterates will converge to $\lambda$.
The linear function approximation version in \cite{basu} provides the following parameter update for $n \geq 0$:
\begin{align}
\label{lspe}
r_{n+1} = r_n + a(n)\left(\frac{B_n^{-1}A_n}{\max(\phi^T(i_0)r_n, \epsilon)}-I \right)r_n, 
\end{align}
where 
\begin{align}
\nonumber
& \epsilon > 0, \mbox{~~ is fixed, ~~} \\ \nonumber
&V(i) \simeq \sum_{k=1}^M r^k\phi_k(i) = \phi^T(i) r, \\ \nonumber
&r = \left(r^1, \dots, r^M\right)^T \text{~is a vector of coefficients}, \\ \nonumber
&\phi^k(\cdot), 1 \leq k \leq M, \text{~are the basis functions or features chosen a priori}, \\ \nonumber
&\phi(i) = \left(\phi^1(i), \dots, \phi^M(i)\right)^T, \\ \nonumber
&\Phi = \mbox{an~}s \times M \text{~matrix whose~} (i,k)\text{~-th entry is~} \phi^k(i) \text{~for} \\ \nonumber  
&1 \leq i \leq s \text{~and~} 1 \leq k \leq M, \\ \nonumber
&\phi(i) = \text{~feature of state $i$}, \\ \nonumber
&A_n = \sum_{m=0}^n e^{c(X_m, X_{m+1})}\phi(X_m)\phi^T(X_{m+1}), \\ \nonumber
&B_n = \sum_{m=0}^n \phi(X_m)\phi^T(X_{m}), \\ \nonumber
&I = M \times M \mbox{~~identity matrix~~}.
\end{align}

We also know from \cite[Theorem 5.3]{basu} that under a crucial assumption (see ($\dagger$) in p~883 there)  on the feature matrix, 
the iterates 
$r_n$ 
satisfy the following: 
\begin{align}
\phi^T(i_0)r_n \to \mu,  \nonumber  
\end{align}
where $\mu >0$ is a Perron-Frobenius eigenvalue of the non-negative matrix $Q=\Pi \mathcal{M}$ with
$\Pi = \Phi(\Phi^T D \Phi)^{-1}\Phi^TD$ and $\mathcal{M} = C\circ P$ (unlike \cite{basu} we consider only a synchronous 
implementation for ease of understanding). Here
$D$ is a diagonal matrix with the $i$-th diagonal entry being $\pi_i$ where 
$\pi= \left(\pi_1, \pi_2, \dots, \pi_s\right)^T$ is the stationary distribution of $\{X_n\}$. 
Also, $e^{c(i,j)}p(j|i)$ is the $(i,j)$-th entry of $C \circ P$ where `$\circ$' denotes the component-wise product of two matrices with identical row and column dimensions. 
Assume that $\gamma_{ij}$ and $\delta_{ij}$ are the $(i,j)$-th entries of the matrix $C \circ P$ and $\Pi \mathcal{M}$ matrices respectively.

Therefore 
$\ln\mu$ serves as an approximation to the original risk-sensitive cost $\ln \lambda$. 
Our aim is to investigate the difference between these two, i.e., $\ln(\frac{\lambda}{\mu})$.

\begin{remark}
Throughout this chapter the results are stated in general for matrices $A$ and $B$ with largest eigenvalues of $A$ and $B$ as 
$\lambda >0$ 
and $\mu >0$ respectively. It should be clear from the context 
what the entries of $A$ and $B$ are. 
\end{remark}

\section{Related work and shortcomings}
\label{rsec3}

Let $\|A\|$ be the operator norm of a matrix defined by $\|A\| =\inf\{c >0: \|Av\| \leq c\|v\| ~~\forall v\}$ where 
$\|v\|=\sum_{i=1}^s|v_i|$. Let $A = C\circ P$ and $B= \Pi \mathcal{M}$. 
The following bound was given in \cite{basu}: 
\begin{align}
\label{spect}
\ln\left(\frac{\lambda}{\mu}\right) \leq \ln \left(1+ \frac{(\|A\| + \|B\|)^{1-\frac{1}{s}} \|A - B\|^{\frac{1}{s}}}{\mu}\right),
\end{align}
using the spectral variation bound from \cite[Theorem VIII.1.1]{bhatia}, namely that if 
$A$ and $B$ are two $s \times s$ matrices with 
eigenvalues $\alpha_1, \dots, \alpha_s$ and $\beta_1, \dots, \beta_s$ respectively, then 
\begin{align}
\label{spect_bound}
\max_j \min_i |\alpha_i - \beta_j| \leq (\|A\| + \|B\|)^{1-\frac{1}{s}} (\|A-B\|)^{\frac{1}{s}}. 
\end{align}
This follows from the observation that if $\alpha_1 > 0$ and $\beta_1 >0$ 
are the leading eigenvalues of $A$ and $B$ respectively  
and $\alpha_1 \leq \beta_1$, then $|\alpha_1 - \beta_1| < \max_j \min_i |\alpha_i - \beta_j|$. A similar thing
happens for the case $\alpha_1 > \beta_1$ except the fact that the roles of $\alpha_i$ and $\beta_i$ and hence the roles of 
$A$ and $B$ get reversed thus keeping the right hand side (R.H.S) of (\ref{spect_bound}) the same.

An important point to note is that when  $\alpha_1 \leq \beta_1$, the fact that $\beta_1$ is the leading eigenvalue of $B$ 
is not used. Same thing happens for the other case where $\alpha_1$ replaces $\beta_1$. 

Another important point above is that for large $s$ the bound given above cannot differentiate between the cases with two 
pairs of matrices $(A_1,B_1)$ and $(A_2,B_2)$ such that $\|A_1\| + \|B_1\| = \|A_2\| + \|B_2\|$ 
but $\|A_1 - B_1\|$ and $\|A_2 - B_2\|$ vary dramatically. 
This will be clear from the next toy example: Consider 
$A_1 = (x_{ij})_{s \times s}, B_1 = (y_{ij})_{s \times s}, A_2= (z_{ij})_{s \times s}, B_2 = (w_{ij})_{s \times s}$. Suppose 
$x_{ij} =p, y_{ij} =q, z_{ij} = p', w_{ij} = q'$ for $i,j \in {1,2,\dots, s}$ with $p + q = p' + q'$, $p'-q' > 0$ and 
$p-q > 0$. It is easy to see that $\|A_1\| + \|B_1\| = \|A_2\| + \|B_2\|$ and $r(A_1) = ps, r(B_1) = qs, r(A_2) = p's, r(B_2) = q's$.
Clearly, $p - q \neq p'-q'$ unless $pq = p'q'$. Here $r(A)$ denotes the Perron-Frobenius eigenvalue of matrix $A$.

In summary, when one is giving a bound between two quantities, the R.H.S 
should have terms involving the difference. However this 
does not occur while using spectral variation bound in the above example  as $(p-q)^{\frac{1}{s}}$ 
will converge to $1$ as $s \to \infty$. In Sections \ref{rsec5}, using the above example we show that the new 
error bounds that we obtain contain always the difference terms irrespective of the state space size $s$. 
\section{Conditions under which error is zero}
\label{rsec4}
We provide here a couple of conditions under which the error is zero and the corresponding results.
\subsection{Theoretical Conditions}

\subsubsection{Condition 1}
\begin{lemma}
Let $\mathbf{x}$ be the left Perron eigen vector of the non-negative matrix $C \circ P$ i.e. $\mathbf{x}^T C \circ P = \lambda \mathbf{x}^T$.  
If $\Phi$ is an $s \times 1$ matrix  and  $\phi_i = y_i$ where $y_i = \frac{x_i}{\pi_i}$,
then $\mu = \lambda$, i.e., there will be no error when function approximation is deployed. 
\end{lemma}
\begin{proof}
It is easy to check that $\delta_{ij} = \frac{\phi^{k(i)}(i) \sum_{l =1}^{s} \phi ^{k(i)}(l) \pi_l \gamma_{lj}}{\sum_{m=1}^{s}{\phi^{k(i)}(m)}^2 \pi_m}$, 
where $\gamma_{ij} = e^{c(i,j)} p(j|i)$.

We claim that with the choice of feature matrix as stated in the theorem, $\lambda$ is the eigenvalue of $B$ with eigenvector being $\mathbf{y}=(y_i)_{i \in \{1,2,\dots s\}}$.
\begin{align}
(\Pi \mathcal{M} y)_i = \sum_{k=1}^s \delta_{ik} y_k = \sum_{k=1}^s \frac{y_i \sum_{l=1}^s x_l \gamma_{lk}}{\sum_{m=1}^s \frac{{x_m}^2}{\pi_m}} y_k = 
\sum_{k=1}^s \frac{y_i \sum_{l=1}^s x_l \gamma_{lk}}{\sum_{m=1}^s \frac{{x_m}^2}{\pi_m}}y_k= \lambda y_i \frac{\sum_{k=1}^s \frac{{x_k}^2}{\pi_k}}{\sum_{m=1}^s \frac{{x_m}^2}{\pi_m}}  
\end{align}

\end{proof}
The claim follows.
\subsubsection{Condition 2}
Recall the assumption ($\dagger$) on the feature matrix $\Phi$ from \cite{basu} which says that the feature matrix $\Phi$ has all non-negative entries 
and any two columns are orthogonal to each other. In this work we strengthen the later part as follows: 

$(\star)$ Every row of the feature matrix $\Phi$ has exactly one positive entry, i.e., for all $i$ there exist 
$1\leq k(i) \leq M$ such that $\phi^{j}(i) > 0$ if $j=k(i)$, otherwise $\phi^{j}(i)=0$.






From \cite[Theorem 1]{bapat} it is easy to see that (this theorem is applicable due to Lemma 5.1 (ii) of \cite{basu} and $(\star)$)
the error can be zero even if 
$C\circ P \neq \Pi \mathcal{M}$, namely under the following conditions:
\begin{enumerate}
 \item there exists positive $\lambda_0, \beta_i, i=1,2, \dots, s$  such that 
 \begin{align}
  \delta_{ij} = \frac{\lambda_0 \gamma_{ij} \beta_i}{\beta_j}, i,j =1,2,\dots,s. \nonumber 
 \end{align}
\item $\Pi_{i,j=1}^s {\delta_{ij}}^{\gamma_{ij}x_i y_j} = \Pi_{i,j=1}^s {\gamma_{ij}}^{\gamma_{ij}x_i y_j}.$
\end{enumerate}

\begin{remark}
Note that if the matrix $\Phi$ has a row $i$ with all $0$s, then $\delta_{ij} =0$ for all $j=1,2,\dots,s$ whereas $\gamma_{ij} >0$ for 
at least one $j \in \{1,2,\dots,s\}$ which violates the conditions for zero error stated above.
\end{remark}

\subsection{Verifiable Condition with doubly stochastic transition kernel}

Note that if the transition kernel is a doubly stochastic matrix then 
it is very hard to find easily verifiable condition on the feature matrix such that $C \circ P = \Pi \mathcal{M}$. The reason is that 
this requires to find a feature matrix $\Phi$ which under $(\star)$ satisfies $\Phi(\Phi^T\Phi)^{-1}\Phi^T = I$. 
This will not be true under $(\star)$ as this requires  $k(i) \neq k(j)$ to hold if $i \neq j$. 
This problem can be alleviated 
by the temporal difference learning algorithm for this setting as under:
\small
\begin{equation}
\label{rtd}
\theta_{n+1}=\theta_n + a(n)\left[\frac{e^{c(X_n, X_{n+1})} \phi^T(X_{n+1})\theta_n}{\phi^T(i_0)\theta_n \lor \epsilon} -\phi^T(X_n) \theta_n\right]\phi(X_n). 
\end{equation}
\normalsize
The following theorem shows its convergence. 
\begin{theorem}
If $\Phi\Phi^T = D^{-1}$ and $\sup_n \|\theta_n\| < \infty \mbox{~a.s.}$ 
then $\phi^T(i_0)\theta_n \to \lambda$.
\end{theorem}
\begin{proof}
First, we analyze the $\epsilon=0$ case.
Note that the algorithm tracks  the o.d.e
\begin{align}
\dot{\theta}(t) = \left(\frac{A'}{\phi^T(i_0)\theta(t)} - B'\right)\theta(t), \nonumber
\end{align}
where $A' = \Phi^TDC\circ P \Phi$ and $B' = \Phi^TD\Phi$.

This follows because it is easy to see that the algorithm tracks the o.d.e
\begin{align}
\dot{\theta}(t) = h(\theta(t)), \nonumber
\end{align}
where $h(\theta)= \sum_i \sum_j \pi(i)p(j|i)\left[\frac{\phi^T(j)\theta}{\phi^T(i_0)\theta} - \phi^T(i)\theta\right]\phi(i)$.

The above statement follows from the 
convergence theorem for differential inclusion with Markov noise \cite{yaji_m} as 
the vector field in (\ref{rtd}) is merely continuous.

Now, the $k$-th entry of $A'\theta$ is 
\begin{align}
\langle \left(\sum_{i=1}^s \phi^k(i)  \sum_{j=1}^s e^{c(i,j)} p(j|i) \phi(j)\right), \theta \rangle.
\end{align}
Similarly, the $k$-th entry of $B'\theta$ can be shown to be the $k$-th entry of $\sum_i \sum_j \pi(i)p(j|i)\left[\phi^T(i)\theta\right]\phi(i)$.

Now, the claim follows directly from \cite[Theorem 5.3]{basu} (the synchronous implementation).
\end{proof}

\section{New error bounds}
\label{rsec5}

\subsection{Bound based on Bapat and Lindqvist's inequality}
Motivated by the discussion in Section \ref{rsec3} and the fact that 
risk-sensitive cost is $\ln\lambda$ rather than $\lambda$ we need to find an 
upper bound
for $\ln \frac{\lambda}{\mu}$. Let $r(A)$ denote the Perron-Frobenius eigenvalue of matrix $A =(a_{ij})_{s\times s}$. In the following we obtain three different bounds for the same quantity under the assumptions that 
a) $\lambda > \mu$, b) the matrix $P = p(j|i)$ has positive entries and impose conditions under
which one is better than the other.
Suppose $A$
admits left and right Perron eigenvectors $\textbf{x}, \textbf{y}$ respectively, with $\sum_i x_iy_i = 1$ (this is
satisfied, for example, if $A$ is irreducible).
The three upper bounds of $\ln \lambda - \ln \mu$ are (\ref{first}) -(\ref{third}).

\small
\begin{empheq}[box=\fbox]{align}
        &\frac{1}{\lambda}\sum_{i,j=1}^s                                      e^{c(i,j)}p(j|i)x_i y_j \left[c(i,j) + \ln p(j|i)- 
       \ln\phi^{k(i)}(i) -\ln\left(\sum_{l=1}^s \phi^{k(i)}(l)\pi_la_{lj}         \right)+\right. \nonumber \\
       & \left. \ln\left(\sum_{m=1}^s{\phi^{k(i)}(m)}^2\pi_m\right)\right],        \label{first} \\
       \nonumber \\
&\ln\left(\lambda\right) -\ln \left( \lambda - \sum_{i=1}^{s}x_i y_i \left(e^{c(i,i)}p(i|i)- \frac{\phi^{k(i)}(i)\sum_{l=1}^s \phi^{k(i)}(l)\pi_le^{c(l,i)}p(i|l)}{\sum_{m=1}^s{\phi^{k(i)}(m)}^2\pi_m}\right) - \right. \nonumber \\ 
&\left. \sum_{i \neq j} e^{c(i,j)}p(j|i)x_i y_j \left(c(i,j) + \ln p(j|i)  + \ln\left(\frac{\sum_{m=1}^{s} {\phi^{k(i)}(m)}^2 \pi_m}{\sum_{l=1}^s \phi^{k(i)}(l) \pi_l e^{c(l,j)}p(j|l)}\right) - \ln \phi^{k(i)}(i) \right)\right), \label{second} \\
\nonumber \\
&\ln \left(1+ \frac{1}{\mu}\left(\sum_{i=1}^{s}x_i y_i \left(e^{c(i,i)}p(i|i)- \frac{\phi^{k(i)}(i)\sum_{l=1}^s \phi^{k(i)}(l)\pi_le^{c(l,i)}p(i|l)}{\sum_{m=1}^s{\phi^{k(i)}(m)}^2\pi_m}\right) + \right.\right. \nonumber \\ 
&\left. \left. \sum_{i \neq j} e^{c(i,j)}p(j|i)x_i y_j \left(c(i,j) + \ln p(j|i)  + \ln\left(\frac{\sum_{m=1}^{s} {\phi^{k(i)}(m)}^2 \pi_m}{\sum_{l=1}^s \phi^{k(i)}(l) \pi_l e^{c(l,j)}p(j|l)}\right)- \ln \phi^{k(i)}(i) \right)\right)\right) \label{third}.
\end{empheq}
\normalsize

The bounds (\ref{first}) -(\ref{third}) of $\frac{\ln \lambda}{\ln \mu}$ follow from (\ref{ineq1})-(\ref{ineq3}).
\small
\begin{empheq}[box=\fbox]{align}
&\frac{r(A)}{r(B)} \leq  \Pi_{i,j=1} ^ s \left(\frac{a_{ij}}{b_{ij}}\right)^{\frac{a_{ij}x_i y_j}{r(A)}}, \label{ineq1} \\
&\frac{r(A)}{r(B)} \leq  \frac{r(A)}{r(A) + \sum_{i=1}^{s}x_i y_i (b_{ii}-a_{ii}) + \sum_{i \neq j} a_{ij}x_i y_j \ln\left(\frac{b_{ij}}{a_{ij}}\right)}, \label{ineq2} \\
&\frac{r(A)}{r(B)} \leq  1 + \frac{1}{r(B)}\left[\sum_{i=1}^{s}x_i y_i (a_{ii}-b_{ii}) + \sum_{i \neq j} a_{ij}x_i y_j \ln\left(\frac{a_{ij}}{b_{ij}}\right)\right]. \label{ineq3} \end{empheq}
\normalsize

\begin{remark}
 Note that in general it is hard to compare the bound given in  (\ref{spect}) 
 with the same in (\ref{first}) -(\ref{third}). We will only show that for the toy example of Section \ref{rsec3} 
 the bounds given in (\ref{ineq1})-(\ref{ineq3}) are much better than the spectral variation bound when the state 
 space is large. Therefore $A$ and $B$ 
 will refer to matrices $A_1$ and $B_1$ respectively. It is easy to calculate
 $\|A\|, r(A), \mathbf{x}$ with this choice of $A$ and $B$. 
 Note that the actual error is $\ln (1 + \frac{\epsilon}{q})$ where $p = q+ \epsilon$ where $\epsilon << q$.
If we use (\ref{spect}) then the error is bounded by $\ln (3 + \frac{\epsilon}{q})$.
However, if we use (\ref{ineq1}) the error is bounded by $\ln (1 + \frac{\epsilon}{q})$ i.e., the actual error. 
If we use (\ref{ineq3}) the error 
is bounded by $\ln\left(1+\left(1+\frac{\epsilon}{q}\right) \ln \left(1 + \frac{\epsilon}{q}\right)\right)$ which 
reduces to $\ln\left(1 + \frac{\epsilon}{q}\right)$ (using $x+ x^2 \sim x \mbox{~~if~~} x<<1$). If we 
use (\ref{ineq2}) the error is bounded by $\ln\left(1 + \frac{\epsilon}{q}\right)$ (using Binomial approximation theorem). 

If $A$ is such that all its diagonal elements are $p$ and the off-diagonal elements are $q$ then
for large state space the actual error is zero. If we use (\ref{ineq1}) then the bound is also zero whereas 
the right hand side of (\ref{spect}) is $\ln 3$. 
 
If $A$ is such that the entry in the first row and first column is $p$ and the rest are all $q$, then also similar 
thing happens except the fact that now the right hand side of (\ref{spect}) is $\ln \left(1 + 2e^{-\frac{4q}{3}}\right)$.

Note that here $a_{ij} > b_{ij} \forall i,j \in \{1,2,\dots, s\}$ in the above example. Our bound will be much more 
useful in cases where there will be $i,j$ such that $b_{ij} > a_{ij}$. From the definition of $\delta_{ij}, \gamma_{ij}$ 
it is clear that for all $j$ there  exists $i$ such that $\delta_{ij} > \gamma_{ij}$. In such a case, for 
every $j$ there will be at least one non-positive term inside the summation 
over $i$ which will make the bound small. The bound given 
in (\ref{spect}) does not capture such cases for large $s$.

Also, from the proof of \cite[Theorem VIII.1.1]{bhatia} we know that the bound (\ref{spect}) is a strict inequality if the 
matrix $P$ has all the entries positive whereas 
from the main theorem in \cite{bapat} we see that there are cases when equality condition holds in (\ref{ineq1}).
\end{remark}

Here (\ref{second}) holds under (\ref{as1}) which follows from the fact that the following condition is necessary and 
sufficient for 
(\ref{ineq2}) to be true:
\begin{align}
\label{condn}
r(A) >  \sum_{i=1}^{s}x_i y_i (a_{ii}-b_{ii}) + \sum_{i \neq j} a_{ij}x_i y_j \ln\left(\frac{a_{ij}}{b_{ij}}\right)
\end{align}
and $\min_i \sum_j a_{ij} \leq r(A)$. Later in the proof of Lemma \ref{neq} we will see that,
in our setting, under $(\star)$, (\ref{condn}) gets
satisfied if the assumptions in Lemma \ref{neq} are true.

\small

\begin{empheq}[box=\fbox]{align}
&\min_i \sum_j e^{c(i,j)}p(j|i) > \sum_{i}x_i y_i\left(e^{c(i,i)}p(i|i)- \frac{\phi^{k(i)}(i)\sum_{l=1}^s \phi^{k(i)}(l)\pi_le^{c(l,i)}p(i|l)}{\sum_{m=1}^s{\phi^{k(i)}(m)}^2\pi_m}\right) +   \nonumber \\ 
& \sum_{i \neq j} x_i y_j e^{c(i,j)}p(j|i) \left(c(i,j) + \ln p(j|i)  +  \ln\left(\frac{\sum_{m=1}^{s} {\phi^{k(i)}(m)}^2 \pi_m}{\sum_{l=1}^s \phi^{k(i)}(l) \pi_l e^{c(l,j)}p(j|l)}\right)- \ln \phi^{k(i)}(i) \right) \label{as1}. 
\end{empheq}

\normalsize

(\ref{ineq1})-(\ref{ineq3}) immediately follow from the classic results of \cite[Theorem 1]{bapat} and \cite[Theorem 2]{henry}. In \cite[Theorem 3]{henry}, it 
is shown that under one condition on matrix entries, (\ref{first}) is better than (\ref{second}) whereas under some 
other condition, it is opposite. In the following we investigate how (\ref{third}) compares to the other two.

\begin{lemma}
The bound (\ref{third}) is always better than (\ref{second}).
\end{lemma}
\begin{proof}
Let $L = \sum_{i=1}^s x_i y_i (a_{ii} - b_{ii}) + \sum_{i,j=1, i \neq j}^s a_{ij} x_i y_j \ln\left(\frac{a_{ij}}{b_{ij}}\right)$. 
Now, from \cite[Theorem 2]{henry} we know that
$L \geq r(A) - r(B)$ which implies that 
\begin{align}
L(L - r(A) + r(B)) \geq 0. \nonumber 
\end{align}
which again implies that 
\begin{align}
\frac{r(B) + L}{r(B)} \leq \frac{r(A)}{r(A) - L}. \nonumber
\end{align}
This means that  (\ref{third}) is better than (\ref{second}). 
\end{proof}

\subsubsection{Some conditions}    
In this section we describe some conditions. They are sufficient conditions under which (\ref{first})-(\ref{second})
compare with each other. They will be referred in the next two lemmas.

\small
\begin{empheq}[box=\fbox]{align}
& \forall i, e^{c(i,i)} p(i|i)\left(\sum_{m=1}^s {\phi^{k(i)}(m)}^2 \pi_m  - {\phi^{k(i)}(i)}^2 \pi_i\right) = \phi^{k(i)}(i) \sum_{l=1, l\neq i}^{s} \phi^{k(i)}(l) \pi_l a_{li} \label{as2} \\
& \forall i \neq j, e^{c(i,j)} p(j|i)\left(\sum_{m=1}^s {\phi^{k(i)}(m)}^2 \pi_m  - {\phi^{k(i)}(i)}^2 \pi_i\right) = \phi^{k(i)}(i) \sum_{l=1, l\neq i}^{s} \phi^{k(i)}(l) \pi_l a_{lj} \label{as3} \\
& \exists i \mbox{~~s.t~~} e^{c(i,i)} p(i|i)  > \max_{1 \leq l \leq s, l\neq i} e^{c(l,i)} p(i|l) \mbox{~~or,} \label{as5} \\
& \exists i \mbox{~~s.t~~} e^{c(i,i)} p(i|i)  < \min_{1 \leq l \leq s, l\neq i} e^{c(l,i)} p(i|l) \label{as6}
\end{empheq}
\normalsize

\begin{lemma}
\label{eq}
Assume that for all $i$, $b_{ii} = a_{ii}$ \cite[Theorem 3 (i)]{henry}. Then (\ref{first}) is better than (\ref{second}) 
\end{lemma}
\begin{proof}
Under the condition mentioned in  \cite[Theorem 3 (i)]{henry},
\begin{align}
 r(A) \Pi_{i \neq j} \left(\frac{b_{ij}}{a_{ij}}\right)^{\frac{a_{ij}x_i y_j}{r(A)}} \geq r(A) - L. \nonumber
\end{align}
Therefore (\ref{first}) is better than (\ref{second}). 
\end{proof}
\begin{remark}
One such example where the condition of  Lemma \ref{eq} gets satisfied is:
$A= (a_{ij})_{s \times  s}$ with $a_{ij} = q$ if $i=j$ and $a_{ij}=p$ otherwise, 
  and $B= (b_{ij})_{s \times  s}$ with $b_{ij}=q$ for all $1 \leq i,j \leq s$ with $p-q \leq q$. 
It is easy to check that 
(\ref{condn}) gets satisfied for this example.
\end{remark}  
\begin{remark}
In our setting the condition mentioned in Lemma \ref{eq} gets satisfied if (\ref{as2}) is true. If the 
feature matrix is a single column matrix with all entries equal then  a sufficient condition for (\ref{as3}) 
is that for every $j$, $e^{c(i,j)}p(j|i)$ is same for all $i$ (for example, the transition probability satisfies
$p(j|i) = e^{-c(i,j)}$ with the cost function $c(.,.)$ being non-negative).
\end{remark}

\begin{lemma}
\label{neq}
Assume that  for all $i \neq j$, $b_{ij} = a_{ij}$ and there is at least one $i$ 
such that $b_{ii} \neq a_{ii}$ \cite[Theorem 3 (ii)]{henry}. Then (\ref{second}) is better than first (\ref{first}).
\end{lemma}

\begin{proof}
Under the condition mentioned in  \cite[Theorem 3 (ii)]{henry},
\begin{align}
 r(A) \Pi_{i=1}^s \left(\frac{b_{ii}}{a_{ii}}\right)^{\frac{a_{ii}x_i y_i}{r(A)}} \leq r(A) - L. \nonumber
\end{align}
Therefore (\ref{condn}) gets satisfied trivially if for all $i$, $b_{ii} \neq 0$ (which is true in our setting under 
$(\star)$ and $b)$). Therefore (\ref{second}) is better than (\ref{first}). 
\end{proof}
\begin{remark}
In our setting the condition mentioned in Lemma \ref{neq} gets satisfied if (\ref{as3}) is true and there exist at least one 
$i$ for which either (\ref{as5}) or (\ref{as6}) is true (assuming that 
feature matrix is a single column matrix with all entries equal). If the 
feature matrix is a single column matrix with all entries equal then  a necessary and sufficient condition for (\ref{as5}) 
is that for every $j$, $e^{c(i,j)}p(j|i)$ is same for all $i \neq j$. 
\end{remark}
\begin{remark}
Similar bounds can be derived in the same way if $\lambda < \mu$. 
\end{remark}

\subsection{Another bound when $A$ is invertible}

Let, $\alpha(A) = \max_i (x_A)^{-1}_i$ where $x_A$ is Perron-Frobenius eigenvector of $A$
which has positive components if $A$ is irreducible.
\begin{theorem}
Under the assumptions
\begin{enumerate}[label=\textbf{(A\arabic*)}]
 \item $A$ is invertible 
 \item $A$ positive semidefinite 
\end{enumerate}
\begin{align}
\label{invert}
 \ln\left(\frac{\lambda}{\mu}\right) \leq \ln\left(\frac{det(A)}{\mu(\mu-\alpha(A)\|A- B\|)}\right).
\end{align}
\end{theorem}
\begin{proof}
If $\lambda$ is an eigenvalue of matrix $A$ with eigenvector $x_A$, then $\frac{det(A)}{\lambda}$ is an eigenvalue of the 
adjoint 
$A^*$ with the same eigenvector.
\begin{align}
&\left(\frac{det(A)}{\lambda} - \mu\right)\langle x_{A}, x_{B} \rangle \nonumber \\
&=\langle A^* x_{A}, x_{B}\rangle - \langle x_{A},B x_B\rangle \nonumber \\
&=\langle x_{A}, Ax_{B}\rangle - \langle x_{A},B x_B\rangle \nonumber \\
&= \langle  x_{A}, (A-B)x_B\rangle. \nonumber 
\end{align}
Moreover, 
\begin{align}
\langle x_{A}, x_{B} \rangle \geq {\alpha(A)}^{-1}\|x_{B}\| = {\alpha(A)}^{-1}. \nonumber
\end{align}
Then the proof follows from the observation that 
\begin{align}
|\langle  x_{A}, (A-B)x_B\rangle | \leq \|A-B\|. \nonumber
\end{align}
Here all the eigenvectors are normalized so that their norm is 1.

From the above one can easily see that 
\begin{align}
|\frac{det(A)}{\lambda} - \mu| \leq \alpha(A)\|A - B\| .\nonumber 
\end{align}
The result trivially follows from the above.
\end{proof}
\begin{remark}
If $A=C\circ P$, 
\begin{enumerate}
 \item Using Oppenheim's Inequality \cite[p~144]{bapat_b} \textbf{(A1)} is satisfied if $P$ is positive definite.  
 \item \textbf{(A2)} is satisfied if $C$ and $P$ are both positive semidefinite.
\end{enumerate} 
\end{remark}

\begin{remark}
Let us take $A = (a_{ij})_{s\times s}$ with $a_{ij} = p$ if $i=j$ and $a_{ij} =q$ otherwise and $b_{ij} = q$ for all $i,j$ with $p >q$. Also, 
let $\epsilon :=p-q$.
Then $det(A) = (p + (s-1)q)(p-q)^{s-1}$. $\alpha(A) = s$. The right hand side of (\ref{invert}) becomes $\ln\left(\frac{(p-q + s q) (p-q)^{s-1}}{qs(qs-(p-q) s)}\right)$.

\begin{align}
\nonumber
f(s)
&=\ln\left(\frac{(p + (s-1)q)(p-q)^{s-1}}{qs(qs-(p-q)s)}\right)\nonumber\\
&=\ln\left(\frac{(p + (s-1)q)(p-q)^{s-1}}{qs^2(q-(p-q))}\right)\nonumber\\
&=\ln\left(\frac{(p + (s-1)q)(p-q)^{s-1}}{qs^2(2q-p)}\right)\nonumber\\
&=\ln\left(\frac{(p-q + sq)(p-q)^{s-1}}{qs^2(2q-p)}\right)\nonumber\\
&=\ln\left(\frac{(\epsilon + sq)\epsilon^{s-1}}{qs^2(q-\epsilon)}\right)
\quad \epsilon=p-q > 0 \nonumber\\
&=\ln\left(\frac{(\epsilon + sq)\epsilon^{s-1}}{qs^2(q-\epsilon)}\right)\nonumber\\
&=\ln\left(\frac{(\epsilon + sq)\epsilon^s}{q\epsilon s^2(q-\epsilon)}\right)\nonumber\\
&=\ln((\epsilon + sq)\epsilon^s)-\ln(q\epsilon s^2(q-\epsilon))\nonumber\\
&=\ln(\epsilon + sq)+s\ln(\epsilon)-\ln(q\epsilon(q-\epsilon))+\ln(s^2)\nonumber\\
&=\ln(s(\epsilon/s + q))+s\ln(\epsilon)-\ln(q\epsilon(q-\epsilon))+2\ln(s)\nonumber\\
&=\ln(s)+\ln(\epsilon/s + q)+s\ln(\epsilon)-\ln(q\epsilon(q-\epsilon))+2\ln(s)\nonumber\\
&=s\ln(\epsilon)+3\ln(s)+\ln(\epsilon/s + q)-\ln(q\epsilon(q-\epsilon))\nonumber\\
&=s\ln(\epsilon)+3\ln(s)+\ln(q)+\ln(\epsilon/(sq) + 1)-\ln(q\epsilon(q-\epsilon))\nonumber
\end{align}

Since
$2q> p > q$,
$q> p-q > 0
$
so $q > \epsilon > 0$.

Since
$\ln(s)/s \to 0$
and
$\ln(1+c/s)
\approx c/s$,

\begin{align}
\nonumber
\frac{f(s)}{s}
&=\ln(\epsilon)+3\frac{\ln(s)}{s}+\frac{\ln(q)+\ln(1+\epsilon/(sq)))-\ln(q\epsilon(q-\epsilon))}{s}\nonumber\\
&\approx \ln(\epsilon)+3\frac{\ln(s)}{s}+\frac{\ln(q)+\epsilon/(sq)-\ln(q\epsilon(q-\epsilon))}{s}\nonumber\\
&\approx \ln(\epsilon)+3\frac{\ln(s)}{s}+\frac{-\ln(\epsilon(q-\epsilon))}{s}
+\frac{\epsilon}{s^2q}\nonumber\\
&\to \ln(\epsilon)\nonumber
\end{align}

Therefore
if $\epsilon=1$ then
$f(s)
\approx 3\frac{\ln(s)}{s}+\frac{-\ln(q-1))}{s}
+\frac{1}{s^2q}
\to 0
$.
The actual error for large $s$ becomes zero.  The right hand side of (\ref{spect}) becomes
$\ln 3$.
\end{remark}

\begin{remark}
Note that $B$ need not be irreducible under the assumption $(\dagger)$ in \cite{basu}. Therefore, 
$x_B$ need not have all the components positive. 
\end{remark}

\section{Conclusion}
\label{rsec6}
In this chapter we gave several new bounds on the function approximation error for policy evaluation 
algorithm in the context of risk-sensitive reinforcement learning. An important future 
direction will be to design and analyze suitable learning algorithms to find 
the optimal policy with the accompanying error bounds. It will be interesting to see whether one can  
use our bounds for policy evaluation problem to provide error bounds for the full control problem.

\chapter{Conclusions and Future work}
\label{chap:con}

In this thesis we analyze several kinds of stochastic approximation algorithms with Markov noise under general conditions that 
was previously not done. This allows one to apply such results for convergence analysis of  several reinforcement learning
 algorithms under general conditions. One such interesting application described in detail is the 
 convergence analysis of off-policy learning algorithms 
 in an on-line learning
environment. Finally, we provide several informative error bounds for function approximation 
for a policy evaluation based algorithm in the case of risk-sensitive reinforcement learning.

Several interesting future directions (to our understanding that are not natural extensions of earlier work) are as follows:
\begin{itemize}
 \item In our analysis of two time-scale stochastic approximation we have assumed pointwise boundedness of the iterates. 
 Certain sufficient conditions for this assumption which is verifiable (also called the Borkar-Meyn theorem \cite{meyn}) 
 has already been proposed by Borkar and Meyn in the case of single time-scale stochastic approximation.
 It may 
 not be possible to `naturally' extend the sufficient 
 condition (which is verifiable) for single timescale stochastic approximation 
 to two time-scale recursions as well as recursion with Markov noise under classical assumptions such as in \cite{metivier}. 
 One needs to check whether  the 
 Lyapunov function based method \cite{andrieu1,kushner} is useful here. Additionally, the dependency of Lipschitz constant 
 on the state space of the Markov process needs to 
 be incorporated in case of two time-scale recursions also.
 \item Another interesting future direction is to extend the lock-in probability results to two time-scale recursions.
 \item Note that the results of Chapter \ref{chap:intro} are true if we assume that the range of the underlying controlled 
       Markov process is a compact metric space. Similarly, in Chapter \ref{chap:lock} we assume the range of the Markov 
        process to be a bounded subset of the Eucledian space. Therefore an interesting future direction is to extend these 
        results for general state space such as Polish space.
 \item Finally, in case of risk-sensitive reinforcement learning, we need to find whether our bound is better than the 
 spectral variation bound in general. Also, one needs to find whether the assumption of knowing stationary distribution apriori 
 can be relaxed in case of convergence analysis of temporal-difference learning in this setting.
\end{itemize}

\chapter{Appendix}
\label{appendix}
\section{Azuma's inequality}
Let $(\Omega, \mathcal{F},P)$ be a probability space and $\{\mathcal{F}_n\}$ a family of increasing sub $\sigma$- fields 
of $\mathcal{F}$. Let $X_n$ be a martingale with respect to the filtration $\mathcal{F}_n$.  Suppose that 
\begin{align}
|X_n - X_{n-1}| \leq k_n < \infty \nonumber 
\end{align}
for some deterministic constants $\{k_n\}$. Then for $\lambda>0$, 
\begin{align}
P(\sup_{m \leq n} |X_m| > \lambda) \leq 2 e^{-\frac{\lambda^2}{\sum_{m\leq n}k_m^2}} \nonumber
\end{align}

\subsection{Proof of conditional and maximal version of Azuma's inequality}
Let $P_B$ denote probability measure defined by $ P_B(A)=\frac{P(A\cap B)}{P(B)}$ where $B \in \mathcal{F}_1$.
If we can show that with this new probability measure $\{S_n\}$ is a martingale, then we can follow the steps in
\cite[(3.30), p. 227]{mcdarmiad} to conclude the proof.

Let us denote by $ E_B$ the expectation with respect to $ P_B$. Clearly,
$E_B(X)= \frac{\int_BX dP}{P(B)}$.
Let $G \in \mathcal{F}_n$.
Now,
\begin{align}
\int_G E_B[S_{n+1}|&\mathcal{F}_n]dP_B = E_B[E_B[I_G S_{n+1}|\mathcal{F}_n]] =E_B[I_G S_{n+1}] = \frac{E[I_{G \cap B} S_{n+1}]}{P(B)} \nonumber \\
                   &=\frac{E[I_{G \cap B}E[S_{n+1}|\mathcal{F}_n]]}{P(B)}=\frac{E[I_{G \cap B} S_n]}{P(B)} = \int_G S_n dP_B. \nonumber
\end{align}
\section{Gronwall's inequality}
\subsection{Continuous version}
For continuous $u(\cdot), v(\cdot) \geq 0$ and scalars $C,K,T \geq 0$,
\begin{align}
u(t) \leq C + K \int_{0}^t u(s)v(s)ds ~~ \forall ~~t \in [0,T], \nonumber  
\end{align}
implies
\begin{align}
u(t) \leq C e^{K \int_{0}^T v(s) ds}, t \in [0,T] \nonumber 
\end{align}
\subsection{Discrete version}
Let $\{x_n, n\geq 0\}$ (resp. $\{a_n, n\geq 0\}$) be non-negative (resp. positive) sequences and 
$C,L \geq 0$ scalars such that for all $n$, 
\begin{align}
x_{n+1} \leq C + L\left(\sum_{m=0}^n a_m x_m \right). \nonumber 
\end{align}
Then for $T_n = \sum_{m=0}^n a_m,$
\begin{align}
x_{n+1} \leq C e^{LT_n}. \nonumber 
\end{align}

\subsection{General discrete Gronwall inequality}
Let $\{\theta_n,n\geq 0\}$ (respectively $\{a_n, n\geq 0\})$ be non-negative (respectively positive)
sequences, $L \geq 0$ and $f(n)$ be a increasing function of $n$ such that for all $n$
\begin{align}
\theta_{n+1} \leq f(n) + L (\sum_{m=0}^{n}a_m \theta_m)\nonumber.
\end{align}
Then for $T_n = \sum_{m=0}^{n}a_m$,
\begin{align}
\theta_{n+1} \leq f(n) e^{LT_n} \nonumber
\end{align}
\begin{proof}
Similar to the proof of Lemma 8 in Appendix B of \cite{borkar1}.
\end{proof}
\section{Perron-Frobenius theorem}
\subsection{For matrices with all the entries positive}
Let $A = (a_{ij})$ be an $n \times n$ positive matrix: $a_{ij} > 0$ for $1 \leq i , j \leq n$. Then the following statements hold:
\begin{enumerate}
\item There is a positive real number $r$, called the Perron root or the Perron-Frobenius eigenvalue, such that $r$ is an eigenvalue of $A$ and 
    any other eigenvalue $\lambda$ (possibly, complex) is strictly smaller than $r$ in absolute value, $|\lambda| < r$. 
    Thus, $r$ is also called the spectral radius $\rho(A)$ of $A$.
    
\item The Perron–Frobenius eigenvalue is simple: $r$ is a simple root of the characteristic polynomial of $A$. 
    Consequently, the eigenspace associated to $r$ is one-dimensional. (The same is true for the left eigenspace, 
    i.e., the eigenspace for $A^T$.)
    
 \item   There exists an eigenvector $v = (v_1,…,v_n)$ of $A$ with eigenvalue $r$ (called the Perron eigenvector) 
     such that all components of $v$ are positive: 
    $A v = r v, v_i > 0$ for $1 \leq i \leq n$. (Respectively, there exists a positive left eigenvector $w : w^T A = r w^T, w_i > 0$.)
    There are no other positive (moreover non-negative) eigenvectors except positive multiples of $v$ (respectively, left 
    eigenvectors except w), i.e., all other eigenvectors must have at least one negative or non-real component. 
 \item The Perron-Frobenius eigenvalue satisfies the inequalities
       \begin{align}
        \min_i \sum_j a_{ij} \leq r \leq \max_i \sum_j a_{ij} \nonumber
       \end{align}
\end{enumerate}
\subsection{For irreducible matrices}
Let $A=(a_{ij})$ be a non-negative irreducible (for every pair of indices $i$ and $j$, 
there exists a natural number $m$ such that the $(i,j)$-th entry of the matrix $A^m$ is not equal to 0) square matrix. Then 
the last three statements holds as in the earlier case, however, the first statement is replaced by the following:

There is a positive real number $r$, called the Perron root or the Perron–Frobenius eigenvalue, such that $r$ is an eigenvalue of $A$ and 
    any other eigenvalue $\lambda$ (possibly, complex) is less or equal than $r$ in absolute value, $|\lambda| \leq r$.

\bibtitle{References}
\bibliographystyle{plain}
\bibliography{mybib}

\end{document}